\documentclass[11pt]{article}
\usepackage{etex}

\usepackage{amsmath}
\usepackage[margin=1in]{geometry}
\usepackage[utf8]{inputenc}
\usepackage[T1]{fontenc}
\usepackage{newunicodechar}
\newunicodechar{（}{(}
\newunicodechar{）}{)}
\usepackage[dvipsnames]{xcolor}
\definecolor{BeaBlue}{RGB}{71,127,124}
\definecolor{BeaRed}{RGB}{210, 77, 4}

\usepackage [colorlinks=true, linkcolor=BeaRed, citecolor=BeaBlue, linktoc=page,pagebackref]{hyperref}
\usepackage{url}
\usepackage{booktabs}
\usepackage{amsfonts}
\usepackage{nicefrac}
\usepackage{microtype}
\usepackage{xcolor}
\usepackage[numbers]{natbib}
\usepackage{url}
\usepackage{amsmath, amsfonts, amsthm, amssymb}
\usepackage{thmtools}
\usepackage{thm-restate}
\usepackage{nicefrac}
\usepackage{bm}
\usepackage{enumitem}
\usepackage[T1]{fontenc}
\usepackage[english]{babel}
\usepackage{mathtools}
\usepackage{makecell}
\usepackage{multirow}
\usepackage{ucs}
\usepackage{url}
\usepackage{dsfont}
\usepackage{graphicx}
\usepackage{caption}

\usepackage[toc,page]{appendix}
\usepackage{stmaryrd}
\usepackage{authblk}
\usepackage{pgfplots}
\pgfplotsset{compat=1.8}
\usepackage{cleveref}
\usepackage{changepage}
\usepackage{breakcites}
\usepackage{patchcmd}
\usepackage{comment}
\usepackage{etoolbox}

 \usepackage[vvarbb]{newtx}
 \usepackage{newtxmath}

\newtheorem{theorem}{Theorem}[section]
\newtheorem{lemma}[theorem]{Lemma}
\newtheorem{assumption}[theorem]{Assumption}

\newtheorem{example}[theorem]{Example}
\newtheorem{proposition}[theorem]{Proposition}
\newtheorem{corollary}[theorem]{Corollary}

\newtheorem{remark}[theorem]{Remark}

\newtheorem{definition}[theorem]{Definition}

\let\oldparagraph=\paragraph
\renewcommand\paragraph[1]{\oldparagraph{#1.}}

\newcommand{\G}{\mathcal{G}}

\newcommand{\R}{\mathbb{R}}
\newcommand{\E}{\mathbb{E}}

\newcommand{\PP}{\mathbb{P}}
\newcommand{\Sph}{\mathcal{S}}

\newcommand{\eps}{\epsilon}

\newcommand{\vertiii}[1]{{\left\vert\kern-0.25ex\left\vert\kern-0.25ex\left\vert #1
    \right\vert\kern-0.25ex\right\vert\kern-0.25ex\right\vert}}

\newcommand{\one}{\mathop{\mathds{1}}}

\newcommand{\Unif}{\mathrm{Unif}}

\newcommand{\norm}[1]{\left\| #1 \right\|}

\newcommand{\abs}[1]{\left| #1 \right|}

\newcommand{\sph}{\mathcal{S}^{d-1}}

\newcommand{\op}{\mathrm{op}}

\usepackage{graphicx}
\usepackage{amssymb}
\usepackage{mathtools}
\usepackage{comment}
\usepackage{enumitem}

\DeclareMathOperator{\rank}{rank}

\newcommand{\ignore}[1]{}

\newcommand{\distpm}{\operatorname{dist}_{\pm}}
\newcommand{\Mat}{\mathsf{Mat}}

\newcommand{\ot}{\otimes}
\newcommand{\ip}[2]{\left\langle #1,#2\right\rangle}

\newcommand{\grad}{\nabla}

\newcommand{\Corr}{\mathrm{Corr}}
\newcommand{\od}{\odot}
\newcommand{\Gc}{{\mathcal G_\theta}}

\title{The Multiscale Single-Index Model: A Stylized Model for Hierarchical Feature Learning}
\author[1]{Joan Bruna}
\affil[1]{Courant Institute School of Mathematics, Computing and Data Science\\ New York University}
\begin{document}

\maketitle

\begin{abstract}
We consider the Multiscale Single-Index Model (MSIM), first introduced in \cite{oymak2021learning}, as a stylized model for hierarchical learning with \emph{scale separation}. Each layer extracts a shared single-index feature at one physical scale and passes it to the next, thus defining a tractable setting in which to study how deep architectures learn multiscale representations. Under non-degeneracy and delocalization assumptions on the link function and planted features respectively, for fixed depth $K$ and local scale $d$, the first Wiener chaos of the target behaves as a perturbed spiked tensor, where the perturbation of order $d^{-1/2}$ comes from the non-linearity --- revealing the MSIM as a natural non-linear analogue of the Tensor PCA model \cite{montanari2014statistical}.

While this perturbative picture is sufficient to enable efficient spectral recovery based on Tensor unfolding (as already observed in \cite{oymak2021learning}), it is not precise enough for the analysis of backpropagation gradient-based methods.
In this work, we address this limitation by performing a fine-grained analysis of the Wiener chaos using Edgeworth expansions.
In the first chaos, this gives a finite-rank hierarchy at scales $d^{-q/2}$. In higher chaoses, balanced flattenings exhibit staircase singular-value plateaus of size $d^{-\rho/2}$ and multiplicity $d^{\rho}$ under a natural higher-chaos non-cancellation condition. Using this higher-chaos structure, and under an additional slow Hermite-energy tail condition, we first establish shallow-network approximation lower bounds, quantifying the benefit of depth in this model. Next, and most importantly, we prove that online SGD on the correlation objective, where all layers evolve in the same timescale, achieves $1 - o_d(1)$ recovery with $n = \widetilde{O}( d^{K-1})$ samples, again recovering the same sample complexity as in the linear counterpart.

\end{abstract}

\tableofcontents

\section{Introduction}

Feature learning, informally defined as the ability to extract informative low-dimensional representations out of high-dimensional data, is the defining characteristic of Neural Networks trained with gradient descent ---and the basis for the modern AI paradigm of pretraining then fine-tuning.
Over the recent years, a feature learning theory has emerged, focusing on a stylized model for high-dimensional learning based on multi-index models \cite{arous2021onlinestochasticgradientdescent, dudeja2018learning, abbe2022merged, abbe2023sgd, dandi2023learning, barbier2019optimal, cui2024asymptotics,latourelle2026statistical}; see \cite{bruna2025surveyalgorithmsmultiindexmodels} for a recent survey, demonstrating the ability of shallow neural networks to perform efficient feature learning, and even identifying sharp phase transitions \cite{montanari2026phasetransitionsfeaturelearning}.

A natural shortcoming for such theory is that shallow NNs are limited idealizations of modern NNs, where depth is a critical component. Towards addressing this gap, several works \cite{allen-zhu2019what, nichani2023provable, ren2026provablelearningrandomhierarchy, dandi2025computationaladvantagedepthlearning, wang2023learninghierarchicalpolynomialsthreelayer} have studied feature learning in deeper architectures, raising novel questions such as elucidating the hierarchy between features at different layers, or how to adapt gradient methods to facilitate learning across all layers. In this context, our \emph{desiderata} for a hierarchical feature learning model are (i) the necessity of depth, ie the fact that previous shallow models cannot efficiently learn, and (ii) the sufficiency of gradient-based `vanilla' backpropagation, without artificial timescale separations across different layers.

In this work, motivated by applications in data domains such as images or text, we consider a stylized model for deep feature learning based on the principle of \emph{scale separation}, where each layer of the network can be associated with a different physical scale. In the context of computer vision, the model becomes a Convolutional Neural Network with disjoint receptive fields, as originally introduced in \cite{oymak2021learning}: the first hidden-layer extracts local features at each patch, and subsequent layers capture local information at the next scale, progressively coarse-graining the domain \cite{1467551,bruna2013invariant, mallat2016understanding}.
Each feature learning stage is instantiated using a single-index model. Specifically, we consider a setting with Gaussian inputs $X \in (\mathbb{R}^{d})^{\otimes K}$, capturing a multiscale structure with $K$ scales and `receptive fields' of extensive size $d$, and a target obtained by a $K$-layer network, whereby each layer $k \in [K]$ coarse-grains the feature map $Z^{(k-1)} \in (\R^{d})^{\otimes K-k+1}$ as $Z^{(k)} = \phi( Z^{(k-1)} \theta_{k}) \in (\R^{d})^{\otimes (K-k)}$; see Figure \ref{fig:schema}. The resulting \emph{multiscale single-index} model (MSIM) $X \mapsto f_{\theta}(X):=Z^{(K)} \in \R$ is thus parametrized by hidden directions or features $\theta_1, \ldots, \theta_K \in \sph$, and a non-linear activation function $\phi: \R \to \R$.

By viewing the input as a $K$-th order tensor, we can consider the Wiener chaos expansion of the target $f_\theta$, whose first level is $\G_\theta:=\E[ f_\theta(X) X]$. The MSIM can then be viewed as a non-linear analogue of Tensor PCA \cite{montanari2014statistical} (which becomes a special case when $\phi(t)=t$)\footnote{under a different data acquisition model: in Tensor PCA one directly observes a spiked tensor, whereas in the MSIM model we observe pairs $\{X_i, f(X_i)\}_i$.}, where the planted directions $\theta_1, \ldots, \theta_K$ may be estimated from a spike in the associated chaos decomposition tensor $\G_\theta$ under appropriate conditions.
This naturally leads to the study of two different algorithmic frameworks: (i) a spectral method, motivated by its efficiency at similar problems, eg single- and multi-index models \cite{mondelli2018fundamental, damian2024computational, joshi2024complexity}, and the tensor PCA \cite{richard2014statistical,donoho2023sharp};  and (ii) stochastic gradient-descent (SGD) on a `native' objective function, inspired by the practical connection with deep learning ---and the main focus of this work.

Our analysis makes two key structural assumptions, already present in the original \cite{oymak2021learning}: first, we consider the setting where $\phi$ is \emph{non-degenerate}, in the sense that $\E_{G \sim N(0,1)}[\phi'(G)] :=\kappa \neq 0$ \footnote{This corresponds to information exponent $1$ in the language of \cite{arous2021onlinestochasticgradientdescent}.}. Next, we assume that the planted directions are \emph{incoherent}, ie delocalized in the basis on which the non-linear activations act.
Under these conditions, we easily verify that $\G_\theta = \lambda \theta_1 \otimes \dots \otimes \theta_K + W$, with $\lambda := \kappa^K$ and $\|W \|_F = O(d^{-1/2})$, from which the efficiency of spectral estimators directly follows, as already observed in \cite{oymak2021learning}. Indeed, in this setting the machinery of tensor unfolding developed for Tensor PCA (e.g. \cite{richard2014statistical, donoho2023sharp}) directly applies, leading to strong recovery guarantees of the form $\tilde{\theta}_j \cdot \theta_j = 1 - o_d(1)$ as soon as the number of samples $n \gg O(d^{\lceil K/2 \rceil})$.

While the previous spiked structure at scale $O(d^{-1/2})$ is sufficient for the spectral analysis, it is not sufficiently precise for the SGD analysis. Instead, we consider a high-order Edgeworth expansion that reveals a low rank hierarchy in $W$ at different orders  $O(d^{-q/2})$, $q \leq K$, and which also extends to higher-order chaos.
We then leverage this fine-grained structure to address both the depth necessity and SGD sufficiency.

First, under a natural high-order chaos non-degeneracy and slow Hermite tail condition on the activation function $\phi$, we establish approximation lower bounds of the MSIM model using shallow NNs, which confirms the fact that this hierarchical model truly extends the abilities of shallow learning.
Next, we turn into the analysis of online spherical SGD over the correlation objective. We establish a `propagation-of-incoherence' property, whereby the incoherence of the student network is preserved during the dynamics. Combined with the Edgeworth expansion, this ultimately gives us enough control on the nonlinear SGD dynamics in the early phase of training (\emph{a.k.a.} the \emph{mediocrity} phase). Our main result shows that spherical SGD on the correlation objective $\mathcal{L}(\tilde{\theta}_1, \dots, \tilde{\theta}_K) := \E[f_{\theta}(X) f_{\tilde{\theta}}(X)] $, when initialized in a favorable basin of constant-in-$d$ probability, achieves weak recovery in $\widetilde O(d^{K-1})$ samples, and continues to $1-o_d(1)$ overlap after a lower-order continuation phase. In that respect, SGD  enjoys the same sample complexity as in the Tensor PCA model \cite{10.1214/19-AOP1415}.

\paragraph{Related Works}

Understanding the benefits of depth in Neural Networks is a longstanding question in the field, with initial progress centered on approximation questions \cite{telgarsky2016benefits,eldan2016power}. The efforts then moved towards incorporating
optimization aspects, with early notable works such as \cite{pmlr-v178-safran22a,malach2021connection}, and later developed in \cite{wang2023learninghierarchicalpolynomialsthreelayer, nichani2023provable, allen-zhu2019what}. One aspect shared by all these works is the fact that the input high-dimensional distribution is rotation-invariant, ie with no particular physical structure. Early works that also exploited tensor representations of data to capture scale separation are \cite{cohen2016expressive}, which focused on expressivity.

Besides the aforementioned \cite{oymak2021learning}, the closest works to ours are \cite{dandi2025computationaladvantagedepthlearning} and \cite{ren2026provablelearningrandomhierarchy}, which we discuss briefly.
Dandi et al. \cite{dandi2025computationaladvantagedepthlearning} study Single/Multi-Index Gaussian Hierarchical Targets $f^\star(x) = g^\star(a^{\star\top} P_k(W^\star x)/\sqrt{d^{\epsilon_1}})$ with $W^\star \in \mathbb{R}^{d^{\epsilon_1}\times d}$, where the hierarchy is one of \emph{dimensions} $d \to d^{\epsilon_1} \to 1$. They train a three-layer MLP \emph{layer-wise}
and propagate Hermite decompositions across layers via the asymptotic Gaussianity of the inner statistic $h^\star$, which holds because $\mathrm{IE}(P_k) = 2$. Our setup differs in three essentials: (1) the hierarchy is one of \emph{physical scales} (within- vs.\ across-patch), with both layers acting in the same dimension $d$ on single-index teachers $\beta^\star,\theta^\star \in \sph$; (2) both layers are trained \emph{jointly} by online SGD on the same timescale, so the second-layer student sees a moving, non-Gaussian feature map throughout; and (3) at $\mathrm{IE}(\phi) = 1$ the inner statistic is not asymptotically Gaussian, so we replace the CLT by a quantitative Stein/Lindeberg comparison driven by the delocalization.
Ren et al. \cite{ren2026provablelearningrandomhierarchy} prove that $L$-layer convolutional networks trained layer-wise learn $L$-level Random Hierarchy Models with $O(m^{(1+o(1))L})$ samples, via a problem-independent ``shallow-to-deep chaining'' reduction: layer-wise training succeeds whenever the target correlates with lower-level latents, the signal to lower layers is clean, and the lower-level features are weakly identifiable. Their setting is complementary to ours: The data are discrete tokens drawn from a probabilistic context-free grammar rather than continuous Gaussian patches, and the targets are categorical class labels rather than scalar regressions, so the relevant geometry is a tree of synonym classes rather than the product sphere.

Our use of Gaussian integration by parts is related to the score-function and
Stein-tensor literature for learning neural networks.  Tensor methods for shallow
networks were developed by Janzamin, Sedghi, and Anandkumar~\cite{janzamin2015beating},
while Ge, Lee, and Ma~\cite{ge2017learning} showed that, under Gaussian inputs,
score functions reduce to Hermite polynomials and the population objective
implicitly decomposes low-rank tensors.
These works motivate spectral and tensor-decomposition methods for
learning planted neural directions, but they do not analyze the fine-grained finite-rank perturbation structure of the Wiener chaos considered here.

The model is also connected to the literature on Gaussian single-index and
multi-index learning.  Classical single-index recovery in Gaussian space is
studied, for example, by Dudeja and Hsu~\cite{dudeja2018learning}.  Recent work
on feature learning for Gaussian multi-index targets emphasizes Hermite structure,
staircase or information-exponent phenomena, and the emergence of hierarchical
features under gradient dynamics~\cite{abbe2022merged,abbe2023sgd, dandi2023twolayer, bietti2023learning, damian2025generative}.
Our MSIM architecture differs in that the hierarchical structure is planted directly through repeated tensor contractions and nonlinearities, allowing us to identify the associated Stein tensor explicitly. Our use of incoherence is closely related to recent work on universality for high-dimensional SGD by Gheissari and Jagannath \cite{gheissari2025universality}, who focused on shallow models and arbitrary product data distributions.
Another notable recent related work is \cite{tabanelli2026deep}, where the authors study a latent Hermite-polynomial hierarchy with layer-wise spectral recovery (a la \cite{nichani2023provable}). While the MSIM model has instead  a physical multiscale tensor/patch structure with shared single-index directions, and we consider simultaneous SGD,
both works share several elements: compositional Gaussian targets, staged feature recovery, Hermite/Wiener-chaos structure, Gaussian universality, and spectral methods.

At a technical level, our population expansion is closest in spirit to
Edgeworth and cumulant expansions for neural networks.  Finite-width deviations
from Gaussian-process behavior have been modeled using multivariate Edgeworth
expansions~\cite{lu2023bayesian}, and recent work develops Hermite/Edgeworth
approximations adapted to neural-network preactivations~\cite{nica2024para}.
Here, high-order weighted Edgeworth expansions are used for a different purpose:
to show that the non-Gaussian corrections are not generic
noise, but a deterministic finite-rank hierarchy of tensor spikes.

Finally, the statistical recovery problem is naturally compared with tensor PCA.
The rank-one spiked tensor model was introduced and analyzed by Richard and
Montanari~\cite{richard2014statistical}, and sharp spectral thresholds for
unfolding, partial tracing, and recursive spectral methods were recently obtained
by Donoho and Feldman~\cite{donoho2023sharp}.  In contrast with the standard
single-spike tensor PCA model, our population tensor is itself a structured
multi-spike object.
Thus the empirical Stein tensor combines sampling noise with a deterministic
multiscale low-rank bias.

\paragraph{Statement of Tool Use} We used GPT 5.5 Pro to proofread our results and to carry out several routine calculations (specifically, Bernstein concentration, martingale and retraction control for the SGD results, and high-order derivative expansions). The tool was also used to generate Python scripts for our numerical simulations.
Following the use of this tool, the author independently re-derived, reviewed, and rigorously verified all mathematical assertions, proofs, computations, and code. The author assumes full responsibility and accountability for the accuracy, integrity, and originality of the final manuscript.

\paragraph{Acknowledgments}
The author would like to thank Yunwei Ren, Florent Krzkala, Denny Wu and Gordon Dai for reviewing our draft and providing helpful feedback. This work was completed during the author's sabbatical visit at the Flatiron Institute, which also provided the computing resources used in the project.

\paragraph{Notation}
All expectations \(\E_X\) are with respect to \(X\), with the directions held fixed,
unless explicitly stated otherwise.  For unit vectors \(u,v\), define the
sign-invariant distance $\distpm(u,v):=\min\{\norm{u-v}_2,\norm{u+v}_2\}$.
For a vector \(u\in\R^d\) and an integer \(m\ge1\), write
\(u^{\odot m}:=(u(1)^m,\ldots,u(d)^m).\)
For a tensor \(T\), \(\rank_{\rm CP}(T)\) denotes its CP rank.  We also write
$\mathcal G_\theta:=\E_X[\nabla f_\theta(X)]$ for the population first Stein tensor.

\section{Setup and Main Results}
\label{sec:mainresults}

\subsection{The Multi-scale Single-Index Model}
Let \(K\ge1\) be fixed. We model the input as a \(K\)-th order Gaussian tensor
of dimension \(d\), i.e. \(X\in(\R^d)^{\ot K}\) has independent \(N(0,1)\)
entries. Given unit vectors \(\theta_1,\ldots,\theta_K\in S^{d-1}\), define
\(Z^{(0)}=X\), and recursively
\begin{align}
  Y^{(r)}_{i_{r+1},\ldots,i_K}
  &=\sum_{j=1}^d\theta_r(j)
    Z^{(r-1)}_{j,i_{r+1},\ldots,i_K},
    \qquad 1\le r\le K,\label{eq:model-y}\\
  Z^{(r)}&=\phi(Y^{(r)})\quad\text{entrywise}.\nonumber
\end{align}
At the final layer \(Z^{(K)}\in\R\), and we write
$ f_\theta(X):=Z^{(K)}$.

\begin{figure}
    \centering
    \includegraphics[width=0.75\linewidth]{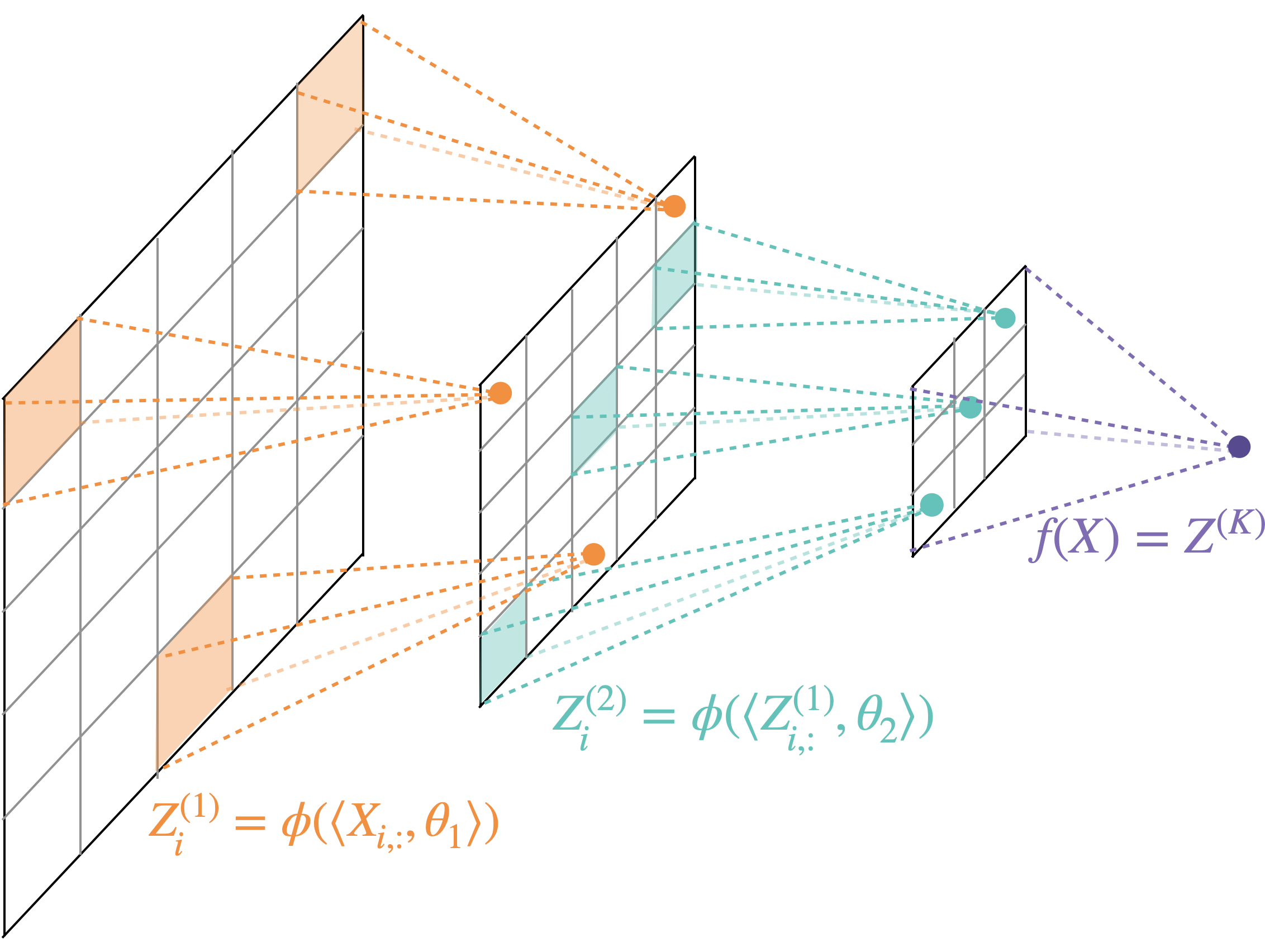}
    \caption{Schema of the Multiscale SIM model. Each patch within the same
    layer \(k\) shares the same local feature, parametrized by a SIM direction
    \(\theta_k\).}
    \label{fig:schema}
\end{figure}

As illustrated in Figure~\ref{fig:schema}, this is a stylized model of a neural
network with physical structure: each layer computes \emph{local} and
\emph{translation-invariant} features at a fixed spatial \emph{scale}.

We now introduce our main assumptions.
\begin{assumption}[Non-degeneracy and normalization]
\label{ass:nondegen_informal}
Let \(G\sim N(0,1)\). We assume
\begin{equation}
  \E\phi(G)=0,
  \qquad
  \E\phi(G)^2=1,
  \qquad
  \kappa:=\E\phi'(G)\ne0.
  \tag{A1}\label{ass:normalization}
\end{equation}
\end{assumption}

\begin{assumption}[Smoothness for the Edgeworth hierarchy]
\label{ass:edgeworth-regularity}
For the finite-rank hierarchy below it is enough to assume that, for a large
constant \(L_K\) depending only on \(K\), \(\phi\in C^{L_K}(\R)\), \(\phi\) is
globally Lipschitz, and all derivatives up to order \(L_K\) have at most
polynomial growth.  A simpler sufficient condition is
\(\phi\in C^{L_K}(\R)\) with
\(\sup_{1\le j\le L_K}\norm{\phi^{(j)}}_\infty<\infty\).
\end{assumption}

In the language of Ben Arous et al.~\cite{arous2021onlinestochasticgradientdescent},
non-degeneracy corresponds to the case where the information exponent of \(\phi\)
is \(1\). In the shallow setting \(K=1\), this assumption greatly simplifies the
geometry of the model, ensuring a descent direction from random initialization.
In our setting \(K>1\), the same first-Hermite non-degeneracy is instrumental for
gradient methods to escape mediocrity and simplifies the spectral analysis.

\begin{assumption}[Incoherence]
\label{ass:incoherence_informal}
The planted directions are \emph{delocalized}, in the sense that
\(\norm{\theta_k}_\infty \leq C_\infty \sqrt{\log(d)/d} \) and $\| \theta_k\|_p \leq C_p d^{1/p - 1/2}$ for each $p>2$ and for all \(k\in[K]\).
\end{assumption}

For \(\theta_k\sim_{\rm iid}{\rm Unif}(S^{d-1})\), these bounds hold with high probability, and the displayed powers are the typical scales.
We will adopt this uniform prior by default in our population analysis. The incoherence assumption formalizes the fact that no individual patch exerts a dominant influence on the output, and it enables a layerwise Gaussian-universality approximation, in the spirit of mean-field and
AMP-type approximations. Finally, we note that both incoherence and non-degeneracy already appear in the original formulation of the MSIM model in \cite{oymak2021learning} (respectively called \emph{kernel difuseness} and \emph{gain} and introduced in definitions 4.6 and 4.3).

\subsection{Main Results}

\paragraph{Chaos Expansion and fine-grained low-rank structure}
Our analysis starts by the harmonic analysis of $f_\theta$, which in our setting is given by the Chaos expansion $f_\theta = \sum_{r \geq 0} \langle \mathcal{G}_{r, \theta} , H_r \rangle$, where $H_r$ is the $r$-th order Hermite tensor of dimension $d^K$, and $\mathcal{G}_{r, \theta} = \E[ f_\theta H_r]$ is the $r$-th Chaos tensor.

Thanks to our structural assumptions on both $\phi$ and $\theta$, the first chaos $\G_{1,\theta} := \G_\theta$ is already informative: a simple first-order expansion along the backpropagation path directly yields $\G_\theta = \lambda \theta_1 \otimes \dots \otimes \theta_K + W$, with $\|W \|_F = O_\delta(d^{-1/2})$ with probability greater than $1-\delta$ under the uniform prior. The first chaos is thus a spiked tensor, whose spike contains the planted parameters of the MSIM model. As already observed in \cite{oymak2021learning}, and in direct analogy with the Tensor PCA model \cite{richard2014statistical, donoho2023sharp}, this property can be immediately leveraged with a suitable spectral method, such as the tensor unfolding (as we confirm in Section \ref{sec:spectral_estimation}, to yield strong recovery at sample complexity $n \gg O(d^{\lceil K/2 \rceil})$).
In this picture, the non-linear nature of the MSIM model is entirely captured in the perturbation $W$.

While spectral methods are indeed robust to any perturbation with small norm, this is not the case for gradient-methods initialised in the so-called \emph{mediocrity zone} \cite{arous2020online}. We thus need a fine-grained description of the specific perturbation $W$. For that purpose, we consider an Edgeworth expansion, which exploits incoherence and smoothness in $\theta$ and $\phi$ respectively to obtain a residual at scale $d^{-K/2}$, instead of the previous $d^{-1/2}$.
This expansion reveals a finite-rank staircase of the first chaos:
\begin{theorem}[Edgeworth finite-rank hierarchy (informal)]
\label{thm:edgeworth-hierarchy-body-informal}
Let \(K\ge2\) be fixed and assume Assumptions~\ref{ass:nondegen_informal} and
\ref{ass:edgeworth-regularity}.  Let
\(\theta_1,\ldots,\theta_K\stackrel{\rm iid}{\sim}\Unif(S^{d-1})\), independently
of \(X\).  For every fixed \(\delta>0\), there exist constants
\(C_{\phi,K,\delta}<\infty\) such that, with probability at
least \(1-\delta\) over the directions,
\begin{equation}
  \mathcal G_\theta
  =
  \sum_{q=0}^{K-1}\mathcal G_\theta^{[q]}
  +\mathcal R_K,
  \qquad
  \norm{\mathcal R_K}_F
  \le
  \,\frac{C_{\phi,K,\delta}}{d^{K/2}}.
  \label{eq:edgeworth-hierarchy-informal}
\end{equation}
Each level \(\mathcal G_\theta^{[q]}\) is a finite linear combination of $O_{K,q}(1)$ rank-one
tensor atoms of total formal order \(q\), and satisfies
\begin{equation}
  \norm{\mathcal G_\theta^{[q]}}_F
  \le
  C_{\phi,K,\delta}\,d^{-q/2},
  \qquad 0\le q\le K-1.
  \label{eq:level-size-informal}
\end{equation}
\end{theorem}

This theorem thus extends the first-order approximation $\mathcal{G}_\theta = \lambda \theta_1 \otimes \dots \otimes \theta_K + W$, $\|W \|_F = O(d^{-1/2})$
to a fine-grained description. While the first-order approximation is sufficient to establish efficient recovery by spectral tensor unfolding methods (Section \ref{sec:spectral_estimation} and also \cite{oymak2021learning}), the fine-grained Edgeworth expansion is crucial for the analysis of SGD of Section \ref{seq:nonlinear_correl}.

\paragraph{Depth separation against Shallow NNs}
Next, we extend the staircase analysis of the first chaos to high-order chaos terms, revealing an interesting structure: as we show in Theorem \ref{thm:balanced-flattening-staircase},
the \(r\)-th chaos \(\mathcal G_{r,\theta}\), viewed through a balanced or near-balanced flattening, has staircase singular-value plateaus: \(d^\ell\) singular values at scale \(d^{-\ell/2}\), for
$\ell=0,\ldots,(K-1)\lfloor r/2\rfloor ~,$ as soon as a high-order generalization of the non-degeneracy assumption $\kappa \neq 0$ holds.
Combined with an energy decay $\Lambda_r:= \| \G_{r,\theta} \|_F^2$ of the form $\sum_{r \geq J} \Lambda_r \gtrsim J^{-p}$, which is controlled in turn by the hermite expansion of $\phi$, this allows us to show that the MSIM model is not efficiently approximated by shallow NNs:
\begin{theorem}[Shallow depth separation (informal)]
\label{cor:shallow-depth-separation-epsilon-informal}
Let $\epsilon>0$ be sufficiently small.
Assume the slow-tail condition and quantitative top-plateau non-cancellation uniformly over the finite block of chaos orders \(R\in\{J_\epsilon,\ldots,L_\epsilon\}\) used to capture an \(\epsilon\)-tail.
Then there exists a constant \(C_{\rm sep}>0\), independent of \(\epsilon\) and \(d\), such that, with high probability over the spherical planted directions and for all sufficiently
large \(d\),
\begin{equation}
    M\le d^{C_{\rm sep}\epsilon^{-1/p}}
    \quad\Longrightarrow\quad
    \inf_{V\in\mathcal N_M(\varrho)}
    \|f_\theta-V\|_{L^2(\gamma_D)}^2
    \ge \epsilon.
    \label{eq:epsilon-depth-separation-general-p-informal}
\end{equation}
\end{theorem}
We thus have in the MSIM a \emph{bona-fide} extension of the class of Multi-index models, with approximation lower bounds against shallow neural networks under suitable high-order non-degeneracy assumptions.

\paragraph{SGD recovery}
We finally study the ability of gradient-based methods to efficiently learn the MSIM class, again under incoherence and non-degeneracy assumptions. We focus on the spherical online SGD from \cite{arous2020online}, where all parameters are trained simultaneously at the same timescale, thus avoiding ad-hoc layer-wise training strategies. For reasons that will become clear later, the `correct' objective turns out to be the correlation $\mathcal{L}(\tilde{\theta}) := \E[f_\theta(X) f_{\tilde{\theta}}(X)]$.
The natural order statistics of the underlying population gradient flow dynamics are the overlaps $m_j(t) := \tilde{\theta}_j(t) \cdot \theta_j$. From the first-order approximation $\G_\theta = \lambda \theta_1 \otimes \dots \otimes \theta_K + W$, $\|W \|_F = O(d^{-1/2})$, the overlaps verify $\dot{m}_j \gtrsim \prod_{i \neq j} m_i + W \times_{i \neq j} \tilde{\theta}_i \times \theta_i$, which under the crude estimate $\|W \| = O(d^{-1/2})$, only enable growth in the overlaps for $K \leq 2$.
However, thanks to the fine-grained structure in $\G_\theta$ uncovered by the Edgeworth expansion, we can overcome this limitation, and obtain instead a control of the form
$\dot{m}_j \geq C \prod_{i \neq j} m_i + O(d^{-K/2})$ in the search phase where $m_j \ll 1$ for arbitrary depth $K$.

\begin{theorem}[Online SGD for nonlinear correlation (informal)]
\label{thm:nl-online-sgd-weak-recovery-informal}
  Fix
$\delta_{\rm sgd}\in(0,1)$.  There exist constants
\[
  \eta_0,a,B,\varepsilon,S,d_0>0,
\]
depending only on $\phi,K$ and on the fixed confidence parameters, such that the
following holds for all $d\ge d_0$.  Run the nonlinear online spherical SGD
recursion \eqref{eq:nl-online-sgd-update} with raw step size
$\eta_d=\widetilde{\Theta}(d^{-K/2})$.  Assume the initialization satisfies the favorable basin conditions.
Then, with probability at least \(1-\delta_{\rm sgd}\), the weak phase reaches
$$
    \min_j m_j \ge \varepsilon
$$
within \(\widetilde{O}_{\phi,K}(d^{K-1})\) samples. After an additional
\(O_{\phi,K}(d^{K/2}\log d)\) continuation phase, the iterates satisfy
$$
    \min_j m_j \ge 1-o_d(1).
$$
\end{theorem}
Since \(d^{K/2}\log d\le \widetilde O(d^{K-1})\) for fixed \(K\ge2\), the total sample complexity is \(\widetilde O_{\phi,K}(d^{K-1})\). In other words, conditioned on the favorable initialization and correct sign
pattern, online SGD for the nonlinear correlation objective achieves recovery with sample complexity
$$
  N_{\rm rec}=\widetilde{O}_{\phi,K}(d^{K-1}).
$$
The formal recovery theorem (stated in Theorems \ref{thm:nl-online-sgd-weak-recovery} and \ref{thm:nl-online-sgd-strong-recovery}) combines the weak-recovery result with a continuation theorem under the macroscopic mixed-incoherence condition. The weak phase establishes both the overlap lower bound and the mixed-incoherence bounds needed to initialize the continuation theorem.
The proof
uses two key ideas: first, we introduce the linear proxy $\widetilde{\mathcal{L}}(\tilde{\theta}) = \langle f_\theta, g_{\tilde{\theta}} \rangle$, where $g_\theta(X) = \langle X, \tilde{\theta}_1 \otimes \dots \otimes \tilde{\theta}_K \rangle$, and leverage the Edgeworth hierarchy of Theorem \ref{thm:edgeworth-hierarchy-body} to obtain the desired overlap growth estimate $\dot{m}_j \gtrsim \prod_{i \neq j} m_i$. Then we study the gradient linearization errors $\nabla_{\tilde{\theta}}( \lambda \widetilde{\mathcal{L}} - \mathcal{L} )$ and show that they can also be dominated by $\prod_{j \neq i} m_i$ as long as the current directions $\tilde{\theta}$ remain incoherent. The proof is completed by then establishing a \emph{propagation-of-incoherence} under the SGD dynamics, which ultimately enables the escape from mediocrity and strong recovery at sample complexity $n = O(d^{K-1})$. Here, and throughout the nonlinear section, by ``strong recovery" we mean $\min_j m_j=1-o_d(1)$; the quantitative continuation Theorem \ref{thm:nl-online-sgd-strong-recovery} reaches a $d^{-1/2}$-scale floor.

Interestingly, the efficiency of SGD on the non-linear correlation objective is \emph{not} preserved in the Mean-Squared error objective. Indeed, the difference between these two dynamics is given by the energy fluctuations $\nabla_{\tilde{\theta}}\| f_{\tilde{\theta}}\|^2$. In the search phase, where the correlation gradient is of order $d^{-(K-1)/2}$, the uninformative energy fluctuations of order $d^{-1/2}$ dominate as soon as $K>2$.

\section{First Wiener Chaos and Spectral Recovery}
\label{sec:spectral}
In this section we consider the Chaos expansion of $f_\theta$, given by the $L^2_\gamma$ representation $f_\theta = \sum_{r \geq 0} \langle \mathcal{G}_{r,\theta}, H_r \rangle$, where $H_r$ is the $r$-th Hermite tensor and
$\mathcal{G}_{r,\theta}:= \E[f_\theta(X) H_r(X)] \in ((\R^{d})^{\otimes K})^{\otimes r}$ is the $r$-chaos Tensor of $f_\theta$.

Thanks to our structural assumptions, the first chaos $\G_{1,\theta}:= \G_\theta$ satisfies $\| \G_{\theta} \| > 0$ almost surely under the uniform prior. As in the single-scale SIM/MIM counterparts, this term provides the most precious source of information that drives the sample complexity of both spectral and gradient-based methods. In Theorem \ref{thm:edgeworth-hierarchy-body} we establish a key structural result needed for the rest of the paper, namely a hierarchical low-rank representation of $\Gc$ into a `staircase' of decreasing order in $d^{-1/2}$ and increasing dimensions. As an immediate consequence, we show in Section \ref{sec:spectral_estimation} that a balanced tensor unfolding strategy succeeds at the threshold $n = \widetilde{O}( d^{\lceil K/2 \rceil})$.

\subsection{A Fine-grained Structure Theorem in the First Chaos}
 It is instructive to consider first the linear setting
\(\phi(t)=t\).  In that case
$$
  f_\theta(X)=\langle X,\theta_1\otimes\cdots\otimes\theta_K\rangle,
$$
and therefore the first Wiener chaos directly gives
$$
  \mathcal G_\theta:= \E_X[f_\theta(X)X]
  =\E_X[\nabla f_\theta(X)]
  =\theta_1\otimes\cdots\otimes\theta_K.
$$
The goal of this section is to compute this chaos expansion
for nonlinear \(\phi\). As we shall now see, under the non-degeneracy and incoherence assumptions, the leading rank-one spike survives, but it is accompanied
by a finite-rank hierarchy of smaller spikes at different orders in $d^{-1/2}$, controlled by the harmonic expansion of $\phi$ and incoherence estimates of $\theta$.

Fix planted directions \(\theta_1,\ldots,\theta_K\).
For a
multi-index \(\mathbf i=(i_1,\ldots,i_K)\), define the derivative factors along
its path by
$$
  D_r(\mathbf i)
  :=
  \phi'\bigl(Y^{(r)}_{i_{r+1},\ldots,i_K}\bigr),
  \qquad 1\le r\le K.
$$
The chain rule gives
\begin{equation}
  \frac{\partial f_\theta(X)}{\partial X_{i_1,\ldots,i_K}}
  =
  \prod_{r=1}^K\theta_r(i_r)D_r(\mathbf i).
  \label{eq:chain-rule}
\end{equation}
Consequently,
\begin{equation}
  \E_X\left[\frac{\partial f_\theta(X)}{\partial X_{i_1,\ldots,i_K}}\right]
  =
  \left(\prod_{r=1}^K\theta_r(i_r)\right)\Lambda_{\mathbf i},
  \qquad
  \Lambda_{\mathbf i}:=\E_X\prod_{r=1}^KD_r(\mathbf i).
  \label{eq:lambda-definition}
\end{equation}
Equivalently,
\begin{equation}
  \mathcal G_\theta
  =
  (\theta_1\otimes\cdots\otimes\theta_K)\odot\Lambda.
  \label{eq:exact-factorization}
\end{equation}
The tensor \(\Lambda\) thus captures the nonlinear correction to the first-chaos
signal.

Let us first informally argue how the non-degeneracy and incoherence assumptions lead to a crude control of the form $\Lambda_{\mathbf{i}}=\lambda + O(K d^{-1/2})$.
By writing
$$Y^{(r)}_{i_{r+1}, \ldots, i_K} = \theta_r(i_r) Z^{(r-1)} + R^{(r)}_{\mathbf{i}}~,$$
observe that the residuals $(R^{(r)}_{\mathbf{i}})_r$ are independent thanks to the tree structure induced by the tensor. Moreover, since $\| \theta_r\|_3^3 \sim d^{-1/2}$, a Berry-Esseen CLT argument directly yields
$\phi'(Y^{(r)}_{i_{r+1}, \ldots, i_K}) = \phi'( G^{(r)}) + O(d^{-1/2})$, where $G^{(r)}$ is a standard Gaussian independent innovation, leading to $\Lambda_{\mathbf{i}} = \E[\phi'(G)]^K + O(K d^{-1/2}) = \lambda + O(d^{-1/2}) $.

Under the incoherence event, the first chaos of the MSIM model is thus of the form $\G_\theta = \lambda~ \theta_1 \otimes \dots \otimes \theta_K + W$, with $\|W \|_F = O(d^{-1/2})$. The parameters of interest can therefore be identified from the leading spike in this tensor. As we will confirm in Section \ref{sec:spectral_estimation}, this leading-order control is sufficient to guarantee spectral recovery using tensor unfolding. However, for our purposes, we need a finer estimate, that we now present. The main technical tool is to consider a high-order analogue of the CLT, given by Edgeworth expansions.

For \(m\ge3\), define the signed power sums
\begin{equation}
  P_{r,m}:=\sum_{j=1}^d\theta_r(j)^m,
  \qquad 1\le r\le K,
  \label{eq:power-sums}
\end{equation}
and their formal orders
\begin{equation}
  \omega_m:=2\left\lfloor\frac{m-1}{2}\right\rfloor,
  \qquad m\ge3.
  \label{eq:power-sum-order}
\end{equation}
Thus \(P_{r,3}\) and \(P_{r,4}\) have order \(2\), while \(P_{r,5}\) and
\(P_{r,6}\) have order \(4\), matching their typical scales under the spherical
prior, and so on. For a multi-index \(\mathbf{i}=(i_1,\ldots,i_K)\in\mathbb N_0^K\), write
\(|\mathbf{i}|=\sum_r i_r\) and
$$
  \Theta^{(\mathbf i)}
  :=
  \bigotimes_{r=1}^K\theta_r^{\odot(i_r+1)}.
$$

The following theorem refines the crude estimate
\(\Lambda = \lambda \one + O(d^{-1/2})\) by revealing the full finite-rank Edgeworth
hierarchy of \(\mathcal G_\theta\).
\begin{theorem}[Edgeworth finite-rank hierarchy]
\label{thm:edgeworth-hierarchy-body}
Let \(K\ge2\) be fixed and assume Assumptions~\ref{ass:nondegen_informal} and
\ref{ass:edgeworth-regularity}.  Let
\(\theta_1,\ldots,\theta_K\stackrel{\rm iid}{\sim}\Unif(S^{d-1})\), independently
of \(X\).  For every fixed \(\delta>0\), there exist constants
\(C_{\phi,K,\delta}<\infty\) such that, with probability at
least \(1-\delta\) over the directions,
\begin{equation}
  \mathcal G_\theta
  =
  \sum_{q=0}^{K-1}\mathcal G_\theta^{[q]}
  +\mathcal R_K,
  \qquad
  \norm{\mathcal R_K}_F
  \le
  \,\frac{C_{\phi,K,\delta}}{d^{K/2}}.
  \label{eq:edgeworth-hierarchy-gradient-proof}
\end{equation}
Each level \(\mathcal G_\theta^{[q]}\) is a finite linear combination of $O_{K,q}(1)$ rank-one
tensor atoms of total formal order \(q\).  More precisely, it can be written as
\begin{equation}
  \mathcal G_\theta^{[q]}
  =
  \sum_{n\in\mathcal N_{K,q}}
  a_{q,n}(\theta)
  \Theta^{(n)},
  \label{eq:collected-hierarchy}
\end{equation}
where \(a_{q,n}(\theta)\) is a scalar polynomial in the signed power sums
\(P_{r,m}\) of formal order \(q-|n|\), and
\begin{equation}
  \mathcal N_{K,q}
  :=
  \left\{
  n\in\mathbb N_0^K:
  n_1\ \text{is even},\ |n|\le q,\ q-|n|\ \text{is even}
  \right\}.
  \label{eq:sharp-index-set}
\end{equation}
Consequently,
\begin{equation}
  \rank_{\rm CP}(\mathcal G_\theta^{[q]})
  \le
  R_{K,q}^{\sharp}
  :=
  \sum_{j=0}^{\lfloor q/2\rfloor}
  (j+1)\binom{K+q-2j-2}{q-2j}.
  \label{eq:sharp-rank-bound}
\end{equation}
Finally, we have the estimate
\begin{equation}
  \norm{\mathcal G_\theta^{[q]}}_F
  \le
  C_{\phi,K,\delta}\,d^{-q/2},
  \qquad 0\le q\le K-1.
  \label{eq:level-size}
\end{equation}
\end{theorem}
In the following, we will denote by $\mathcal{E}_\theta(\delta)$ the \emph{Edgeworth event} described in Theorem \ref{thm:edgeworth-hierarchy-body}, occurring with probability at least $1-\delta$ under the uniform prior. The first two levels have a particularly simple form.  Let
\begin{equation}
  \nu:=\E\phi''(G),
  \qquad
  \gamma:=\E[\phi(G)\phi'(G)] \qquad \lambda := \kappa^K~.
  \label{eq:nugamma}
\end{equation}
Then
\begin{equation}
  \mathcal G_\theta^{[0]}
  =
  \lambda \,\theta_1\otimes\cdots\otimes\theta_K,
  \label{eq:first-level-leading}
\end{equation}
and
\begin{equation}
  \mathcal G_\theta^{[1]}
  =
  \kappa^{K-2}\nu\gamma
  \sum_{r=2}^K
  \theta_1\otimes\cdots\otimes\theta_{r-1}
  \otimes\theta_r^{\odot2}
  \otimes\theta_{r+1}\otimes\cdots\otimes\theta_K.
  \label{eq:first-order-secondary}
\end{equation}
In particular, this formalizes the first-order expansion
\begin{equation}
  \mathcal G_\theta
  =
  {\lambda}\theta_1\otimes\cdots\otimes\theta_K
  +W_\theta,
  \qquad
  \norm{W_\theta}_F
  \le
  C_{\phi,K,\delta}\,d^{-1/2},
  \label{eq:frobenius-sphere-updated}
\end{equation}
with probability at least \(1-\delta\).

\subsection{Proof of Theorem \ref{thm:edgeworth-hierarchy-body}}

We prove the hierarchy in three steps.  First, we describe a high-probability
spherical event on which all coordinate powers and signed power sums have their
expected sizes.  Second, we prove a recursive Edgeworth expansion for the
``off-spine'' expectations that appear in the derivative path.  Third, we combine
this expansion with a Taylor expansion of $\phi$ along the selected path and collect the
resulting rank-one tensor atoms.

Fix a confidence level \(\delta\in(0,1)\), as in the theorem statement.
Throughout this section, constants denoted by \(C\) may change from line to
line, but depend only on \(\phi,K\), the fixed confidence parameter \(\delta\),
and the finite smoothness order used in
Assumption~\ref{ass:edgeworth-regularity}.  They never depend on \(d\).

\paragraph{Spherical power-sum event}

For \(m\ge1\), recall that
$$
  P_{r,m}=\sum_{j=1}^d \theta_r(j)^m,
$$
and that the formal order of \(P_{r,m}\) is
$$
  \omega_1=\omega_2=0,
  \qquad
  \omega_m=2\left\lfloor \frac{m-1}{2}\right\rfloor,\quad m\ge3.
$$
Thus \(\omega_3=\omega_4=2\), \(\omega_5=\omega_6=4\), and so on.

\begin{lemma}[Spherical power sums and coordinate powers]
\label{lem:appendix-spherical-event}
Fix integers \(K\ge2\) and \(M\ge1\).  Let
\(\theta_1,\ldots,\theta_K\stackrel{\rm iid}{\sim}\Unif(S^{d-1})\).  For every
\(\delta\in(0,1)\), with probability at least \(1-\delta\), the following bounds hold
simultaneously for all \(1\le r\le K\), all \(1\le m\le M\), and all
\(1\le b\le M\):
\begin{align}
  \norm{\theta_r}_\infty
  &\le C\sqrt{\frac{\log d}{d}}, \label{eq:app-infty}\\
  \norm{\theta_r^{\odot b}}_2
  &\le C d^{-(b-1)/2}, \label{eq:app-power-vector}\\
  |P_{r,1}|
  &\le C, \label{eq:app-P1}\\
  |P_{r,m}|
  &\le C d^{-\omega_m/2},\qquad m\ge3. \label{eq:app-Pm}
\end{align}
Consequently, every scalar-tensor atom of formal order \(q\), namely
$$
  \left(\prod_{r=1}^K\prod_{m=1}^M P_{r,m}^{\ell_{r,m}}\right)
  \bigotimes_{r=1}^K \theta_r^{\odot b_r},
  \qquad
  q=\sum_{r=1}^K(b_r-1)+\sum_{r=1}^K\sum_{m=1}^M\ell_{r,m}\omega_m,
$$
satisfies
\begin{equation}
  \left\|
  \left(\prod_{r,m}P_{r,m}^{\ell_{r,m}}\right)
  \bigotimes_{r=1}^K\theta_r^{\odot b_r}
  \right\|_F
  \le C d^{-q/2}.
  \label{eq:app-atom-size}
\end{equation}
\end{lemma}

\begin{proof}
Represent \(\theta_r=g_r/\norm{g_r}_2\), where \(g_r\sim N(0,I_d)\).  A standard
chi-square concentration bound gives \(\norm{g_r}_2^2\asymp d\) for all
\(r\le K\) with probability at least \(1-\delta/4\).  On the same event,
Gaussian maxima give \(\max_j|g_r(j)|\le C\sqrt{\log d}\), proving
\eqref{eq:app-infty}.

For each fixed \(p\ge2\) in the finite set needed below, concentration of the polynomial
\(\sum_j |g_r(j)|^p\), together with the lower bound on \(\norm{g_r}_2\), gives
$$
  \sum_{j=1}^d |\theta_r(j)|^p
  \le C_{p,K,\delta}d^{1-p/2}.
$$
Taking \(p=2b\) proves \eqref{eq:app-power-vector}.  Taking \(p=m\) proves
\eqref{eq:app-Pm} for even \(m\ge4\).  For odd \(m\ge3\), the signed numerator
\(\sum_j g_r(j)^m\) is a centered fixed-degree Gaussian polynomial with variance
of order \(d\).  Hypercontractivity and Markov's inequality, or equivalently a
standard tail bound for fixed-degree Gaussian chaoses, imply
$$
  \left|\sum_{j=1}^d g_r(j)^m\right|
  \le C_{m,K,\delta}d^{1/2}
$$
with probability at least \(1-\delta/(4KM)\).  Dividing by
\(\norm{g_r}_2^m\asymp d^{m/2}\) yields
\(|P_{r,m}|\le C d^{-(m-1)/2}\), which is
\eqref{eq:app-Pm} for odd \(m\).  Finally, \(P_{r,1}=\langle \theta_r,{\bf 1}\rangle\)
has the distribution of the first coordinate of a uniform vector on the sphere
multiplied by \(\sqrt d\), so \eqref{eq:app-P1} follows from the usual spherical
cap tail bound.  A union bound over the finitely many \(r,m,b\) completes the
proof, and \eqref{eq:app-atom-size} follows by multiplying the displayed bounds.
\end{proof}

\paragraph{A smooth weighted Edgeworth expansion}

We next state the smooth scalar Edgeworth decomposition used recursively. The key feature is that the expansion is expressed through \emph{signed} weight power
sums \(\sum_j w_j^m\) instead of the absolute Berry-Esseen quantity \(\sum_j |w_j|^3\). We first establish the expansion for activations with bounded derivatives (Lemma \ref{lem:appendix-weighted-edgeworth}), and then extend the result to polynomially bounded derivatives in Lemma \ref{lem:weighted-edgeworth-poly-growth} using a standard truncation argument.

\begin{lemma}[Smooth weighted Edgeworth expansion]
\label{lem:appendix-weighted-edgeworth}
Fix an integer \(J\ge1\).  Let \(V_1,\ldots,V_d\) be iid
subgaussian random variables whose cumulants $\chi_m(V)$ up to order \(2J+4\) are bounded.
Let \(h\in C^{2J+4}(\R)\) have bounded derivatives up to order \(2J+4\).  Let
\(w\in\R^d\) satisfy \(\norm w_\infty\le C\sqrt{\log d/d}\),
\(\norm w_2\le1\), \(|\sum_j w_j|\le C\), and
$$
  \sum_j |w_j|^p\le C_p d^{1-p/2},
  \qquad 3\le p\le 2J+3.
$$
Put
\(S_w=\sum_j w_jV_j\).  Then
\begin{equation}
  \E h(S_w)
  =
  \sum_{\ell\in\mathfrak E_J}
  c_{h,\ell}
  \prod_{m=1}^{2J+2}\Delta_m(w,V)^{\ell_m}
  +O_{h,J}\bigl(d^{-(J+1)/2}\bigr),
  \label{eq:app-weighted-edgeworth}
\end{equation}
where \(\mathfrak E_J\) is a finite set depending only on \(J\), the coefficients
\(c_{h,\ell}\) are deterministic Gaussian expectations of derivatives of \(h\),
and
\begin{align*}
  \Delta_1(w,V)&:=\chi_1(V)\sum_j w_j,\\
  \Delta_2(w,V)&:=\chi_2(V)\sum_j w_j^2-1,\\
  \Delta_m(w,V)&:=\chi_m(V)\sum_j w_j^m,
  \qquad 3\le m\le 2J+2.
\end{align*}
The finite sum is truncated
to cumulant monomials of formal order at most \(J\), and the remainder is uniform
over all weights satisfying the displayed bounds.
\end{lemma}

\begin{proof}
This is a standard smooth Edgeworth expansion for triangular arrays.  We include
the argument to identify the signed quantities that enter the expansion.  Let
\(\psi(t)=\E e^{itV}\) be the characteristic function of $V$. The characteristic function of \(S_w\) is
$$
  \prod_{j=1}^d\psi(tw_j)
  =\exp\left\{\sum_{j=1}^d\log\psi(tw_j)\right\}.
$$
Expanding \(\log\psi(u)\) through order \(2J+2\) gives
$$
  \sum_{j=1}^d\log\psi(tw_j)
  =
  -\frac{t^2}{2}
  +
  \sum_{m=1}^{2J+2}\frac{(it)^m}{m!}\Delta_m(w,V)
  +\operatorname{Err}_{J}(t,w),
$$
where the term \(-t^2/2\) is the reference Gaussian exponent and where
\(\operatorname{Err}_{J}\) is bounded by
$$
  C_J |t|^{2J+3}\sum_j |w_j|^{2J+3}\E(1+|V|^{2J+3}).
$$
Thanks to the subgaussianity assumption and the moment bounds on \(w\), we have
\(\sum_j |w_j|^{2J+3}\le C d^{-(2J+1)/2}\).  Fourier inversion for
smooth test functions, after multiplying by a standard cutoff and using bounded
derivatives of \(h\), allows the characteristic-function expansion to be
integrated term by term.  Expanding the exponential of
\(\sum_m (it)^m\Delta_m/m!\) and keeping only monomials of formal order at most
\(J\) gives \eqref{eq:app-weighted-edgeworth}; all discarded terms and the
Fourier-tail contribution are bounded by
\(C_{h,J}d^{-(J+1)/2}\).  Equivalently, the coefficients can be
written as Gaussian expectations of derivatives through the formal identity
$$
  \E h(G) (it)^m\quad \longleftrightarrow\quad \E h^{(m)}(G),
$$
which proves the stated form of the constants.
\end{proof}

\begin{lemma}[Smooth weighted Edgeworth expansion: polynomial-growth tests]
\label{lem:weighted-edgeworth-poly-growth}
Fix \(J\ge 1\), and put \(L=2J+4\).  Let
\(\mathcal H\) be a finite family of functions \(h\in C^{L}(\mathbb R)\).
Assume that for some \(Q\ge 0\),
$$
  \max_{h\in\mathcal H}\max_{0\le r\le L}
  \sup_{x\in\mathbb R}
  \frac{|h^{(r)}(x)|}{(1+|x|)^Q}
  <\infty .
$$
Assume also that the bounded-derivative version of
Lemma~\ref{lem:appendix-weighted-edgeworth} holds in the following quantitative
form: for every \(g\in C_b^{L}(\mathbb R)\),
\begin{equation}
  \left|
  \mathbb E g(S_w)
  -
  \sum_{\ell\in\mathfrak E_J}
  c_{g,\ell}
  \prod_{m=1}^{2J+2}\Delta_m(w,V)^{\ell_m}
  \right|
  \le
  C_J \|g\|_{C_b^{L}} d^{-(J+1)/2},
  \label{eq:edgeworth_d1}
\end{equation}
uniformly over all admissible weights \(w\), where
$$
  \|g\|_{C_b^{L}}:=\max_{0\le r\le L}\|g^{(r)}\|_\infty .
$$
Then, uniformly over \(h\in\mathcal H\) and uniformly over all admissible
weights \(w\),
\begin{equation}
\label{eq:poly-growth-edgeworth-remainder}
  \mathbb E h(S_w)
  =
  \sum_{\ell\in\mathfrak E_J}
  c_{h,\ell}
  \prod_{m=1}^{2J+2}\Delta_m(w,V)^{\ell_m}
  +
  O_{\mathcal H,J}\!\left(
    (\log d)^{Q/2} d^{-(J+1)/2}
  \right).
\end{equation}
More generally, if the constant in \eqref{eq:edgeworth_d1} is bounded by a fixed polynomial in
\(\|g\|_{C_b^L}\), the same argument gives the same conclusion with
\((\log d)^{Q/2}\) replaced by another fixed power of \(\log d\).
\end{lemma}

\begin{proof}
Let \(\eta\in C_c^\infty(\mathbb R)\) satisfy
$$
  0\le \eta\le 1,\qquad
  \eta(x)=1 \text{ for } |x|\le 1,\qquad
  \eta(x)=0 \text{ for } |x|\ge 2.
$$
For \(R\ge 1\), define
$$
  h_R(x):=h(x)\eta(x/R).
$$
Since \(\mathcal H\) is finite and the derivatives of the functions in
\(\mathcal H\) have polynomial growth, there is a constant \(C_{\mathcal H}\)
such that
$$
  |h^{(r)}(x)|\le C_{\mathcal H}(1+|x|)^Q,
  \qquad
  h\in\mathcal H,\quad 0\le r\le L .
$$
By Leibniz' rule,
$$
  h_R^{(r)}(x)
  =
  \sum_{a=0}^r
  \binom{r}{a}
  h^{(a)}(x) R^{-(r-a)}\eta^{(r-a)}(x/R).
$$
The derivatives of \(\eta\) are supported on \(|x|\le 2R\), and \(R\ge1\).
Therefore, uniformly in \(h\in\mathcal H\),
\begin{equation}
  \|h_R\|_{C_b^L}
  \le
  C_{\mathcal H,L,\eta}(1+R)^Q .
\end{equation}

We apply the quantitative bounded-derivative Edgeworth expansion \eqref{eq:edgeworth_d1} to
\(h_R\).  This gives
\begin{equation}
\label{eq:edgeworth_d3}
  \mathbb E h_R(S_w)
  =
  \sum_{\ell\in\mathfrak E_J}
  c_{h_R,\ell}
  \prod_{m=1}^{2J+2}\Delta_m(w,V)^{\ell_m}
  +
  O_{\mathcal H,J}\!\left(
    (1+R)^Q d^{-(J+1)/2}
  \right),
\end{equation}
uniformly in \(h\in\mathcal H\) and \(w\).

It remains to compare \(h_R\) with \(h\).  First, the variables \(S_w\) have
uniform subgaussian tails.  Indeed, \(V-\mathbb EV\) is subgaussian, and
\(\|w\|_2\le1\), while
$$
  |\mathbb E S_w|
  =
  |\mathbb EV|\,\Big|\sum_j w_j\Big|
  \le C .
$$
Hence there are constants \(C,c>0\), independent of \(d,w\), such that
$$
  \mathbb P(|S_w|>t)\le C e^{-c t^2},
  \qquad t\ge0.
$$
Consequently,
\begin{equation}
  \mathbb E\left[(1+|S_w|)^Q\mathbf 1_{\{|S_w|>R\}}\right]
  \le
  C_{\mathcal H,Q} (1+R)^Q e^{-cR^2}.
  \label{eq:edgeworth_d5}
\end{equation}
Since \(h_R=h\) on \([-R,R]\), \eqref{eq:edgeworth_d5} implies
\begin{equation}
  \sup_{h\in\mathcal H}
  \left|
  \mathbb E h(S_w)-\mathbb E h_R(S_w)
  \right|
  \le
  C_{\mathcal H,Q}(1+R)^Q e^{-cR^2}.
  \label{eq:edgeworth_d6}
\end{equation}

Next we compare the Edgeworth coefficients.  Each coefficient \(c_{h,\ell}\) is a
finite linear combination, depending only on \(J\) and \(\ell\), of Gaussian
expectations of derivatives of \(h\) up to order \(2J+2\).  Therefore,
using again Leibniz' rule and the fact that \(h_R=h\) on \([-R,R]\),
\begin{equation}
  |c_{h_R,\ell}-c_{h,\ell}|
  \le
  C_{\mathcal H,J} (1+R)^Q e^{-cR^2}.
  \label{eq:edgeworth_d7}
\end{equation}
The cumulant factors are uniformly bounded over admissible weights.  Indeed,
$$
  |\Delta_1(w,V)|\le C,\qquad
  |\Delta_2(w,V)|\le C,
$$
and for \(m\ge3\),
$$
  |\Delta_m(w,V)|
  \le
  C_m\sum_j |w_j|^m
  \le
  C_m d^{1-m/2}
  \le C_m .
$$
Since \(\mathfrak E_J\) is finite, \eqref{eq:edgeworth_d7} gives
\begin{equation}
  \sum_{\ell\in\mathfrak E_J}
  |c_{h_R,\ell}-c_{h,\ell}|
  \prod_{m=1}^{2J+2}
  |\Delta_m(w,V)|^{\ell_m}
  \le
  C_{\mathcal H,J}(1+R)^Q e^{-cR^2}.
  \label{eq:edgeworth_d8}
\end{equation}

Combining \eqref{eq:edgeworth_d3}, \eqref{eq:edgeworth_d6}, and \eqref{eq:edgeworth_d8}, we obtain
$$
  \mathbb E h(S_w)
  =
  \sum_{\ell\in\mathfrak E_J}
  c_{h,\ell}
  \prod_{m=1}^{2J+2}\Delta_m(w,V)^{\ell_m}
  +
  O_{\mathcal H,J}\!\left(
    (1+R)^Q d^{-(J+1)/2}
    +(1+R)^Q e^{-cR^2}
  \right).
$$
Finally choose
$$
  R=R_d=A\sqrt{\log d},
$$
with \(A\) large enough that
$$
  (1+R_d)^Q e^{-cR_d^2}
  \le
  C_{\mathcal H,J} d^{-(J+1)/2}.
$$
Then
$$
  (1+R_d)^Q d^{-(J+1)/2}
  \le
  C_{\mathcal H,J}
  (\log d)^{Q/2}d^{-(J+1)/2},
$$
which proves the claim. The finitely many small values of \(d\) can be absorbed into the constant.
\end{proof}

\begin{remark}[Slackened Edgeworth truncation convention]
\label{rem:edgeworth-slack-convention}
We shall use Lemma~\ref{lem:weighted-edgeworth-poly-growth} with one order of
slack.  More precisely, suppose that an argument requires an expansion
retaining all terms of formal order strictly smaller than \(J\).  We
apply Lemma~\ref{lem:weighted-edgeworth-poly-growth} at order \(J\), decompose
the explicit Edgeworth polynomial into terms of formal order \(<J\) and
terms of formal order exactly \(J\), and absorb the latter into the
remainder.

On the spherical power-sum event of Lemma~\ref{lem:appendix-spherical-event},
every explicit scalar-tensor atom of formal order \(J\) is bounded by
\(C d^{-J/2}\).  On the other hand, the analytic remainder in
\eqref{eq:poly-growth-edgeworth-remainder} is
$$
    O\!\left((\log d)^A d^{-(J+1)/2}\right)
    =
    o(d^{-J/2})
$$
for fixed \(J,A\).  Hence, after increasing \(d_0\) and the constant,
we may write the slackened expansion as
$$
    \mathbb E h(S_w)
    =
    \sum_{\operatorname{ord}(\tau)<J}
    c_\tau M_\tau(w,V)
    +
    \operatorname{Err}_J(h,w),
    \qquad
    |\operatorname{Err}_J(h,w)|
    \le
    C d^{-J/2}.
$$
More generally, if this scalar expansion is multiplied by a deterministic
tensor atom of formal order \(s\), then the induced Frobenius error is
bounded by
$$
    C d^{-(s+J)/2}.
$$
All constants may depend on the fixed parameters \(\phi,K,J,\delta\)
and on the finite test-function family, but not on \(d\).
\end{remark}

\paragraph{Recursive expansion of off-spine coefficient vectors}

For a path coordinate \({\bf i}=(i_1,\ldots,i_K)\), write
$$
  \beta_r:=\theta_r(i_r),
  \qquad
  \xi:=X_{i_1,\ldots,i_K}.
$$
Along this path define
$$
  U_0:=\xi,
  \qquad
  U_r:=Z^{(r)}_{i_{r+1},\ldots,i_K},
$$
so that
\begin{equation}
  U_r=\phi(R_r+\beta_rU_{r-1}),
  \qquad
  D_r=\phi'(R_r+\beta_rU_{r-1}),
  \label{eq:app-Ur-Dr}
\end{equation}
where
\begin{equation}
  R_r=R_{r,i_r}:=
  \sum_{j\ne i_r}\theta_r(j)Z^{(r-1)}_{j,i_{r+1},\ldots,i_K}.
  \label{eq:app-Rri}
\end{equation}
The random variables \(\xi,R_1,\ldots,R_K\) are mutually independent.  Indeed,
\(R_r\) depends only on leaves whose \(r\)-th coordinate is different from
\(i_r\) and whose later coordinates equal \(i_{r+1},\ldots,i_K\); for different
values of \(r\) these leaf sets are disjoint, and they are also disjoint from the
single spine leaf \(X_{i_1,\ldots,i_K}\).

The following lemma considers a recursive coefficient expansion, showing that every
off-spine expectation \(a\mapsto \E h(R_{r,a})\) is a finite polynomial in global
signed power sums and the deleted coordinate powers \(\theta_r(a)^u\), up to the
required order.
\begin{lemma}[Off-spine coefficient expansion]
\label{lem:appendix-coeff-expansion}
Fix \(J\le K\), and let \(\mathcal H\) be any finite family of test functions whose derivatives up to the required order have at most polynomial growth. On the event of
Lemma~\ref{lem:appendix-spherical-event}, for every \(h\in\mathcal H\), every
\(1\le r\le K\), and every \(a\in[d]\),
\begin{equation}
  \E_X h(R_{r,a})
  =
  \sum_{\sigma\in\mathcal I(h,r,J)}
  c_\sigma
  \left(\prod_{u=1}^K\prod_{m=1}^{2J+2}P_{u,m}^{\ell_{\sigma,u,m}}\right)
  \theta_r(a)^{v_\sigma}
  +\rho_{h,r,J}(a).
  \label{eq:app-coeff-expansion}
\end{equation}
Moreover, for every fixed \(b\) in the finite range needed below,
\begin{equation}
  \norm{\theta_r^{\odot b}\odot \rho_{h,r,J}}_2
  \le C d^{-(J+b-1)/2}.
  \label{eq:app-coeff-remainder-tensorized}
\end{equation}
The index set \(\mathcal I(h,r,J)\) is finite and depends only on \(h,r,J,K\),
not on \(d\).  The constants \(c_\sigma\) are deterministic and depend only on
\(\phi,K,J\) and Gaussian expectations of derivatives of functions in
\(\mathcal H\).  Every displayed monomial has formal order strictly smaller than
\(J\), after counting the local coordinate factor \(\theta_r(a)^{v_\sigma}\) as
order \(v_\sigma\).
\end{lemma}

\begin{proof}
The proof proceeds by induction on \(r\).  For \(r=1\), the summands in \(R_{1,a}\)
are independent standard Gaussians, and therefore
$$
  R_{1,a}\sim N(0,1-\theta_1(a)^2).
$$
Expanding the map \(\sigma^2\mapsto \E h(\sigma G)\) around \(\sigma^2=1\) gives
$$
  \E h(R_{1,a})
  =
  \sum_{q=0}^{\lfloor(J-1)/2\rfloor} c_{h,q}\theta_1(a)^{2q}
  +O\bigl(|\theta_1(a)|^{J}\bigr),
$$
which has the required tensorized form because
\(\norm{\theta_1^{\odot(b+J)}}_2\le C d^{-(J+b-1)/2}\).

Assume the claim holds up to layer \(r-1\) for all test functions needed below.
A generic input to layer \(r\) has the form
$$
  V=\phi(Y_{r-1}),
$$
where \(Y_{r-1}\) is a generic coordinate of the previous preactivation.  The
cumulants \(\chi_m(V)\), for \(m\le 2J+4\), are polynomials in moments
\(\E \phi(Y_{r-1})^a\) of bounded degree.  Applying the induction hypothesis to
the finite family of functions \(x\mapsto \phi(x)^a\) shows that each cumulant has
a finite expansion in signed power sums of the previous directions, with the
tensorized remainder bound \eqref{eq:app-coeff-remainder-tensorized}.

Now write
$$
  R_{r,a}=\sum_{j\ne a}\theta_r(j)V_j,
$$
where \(V_j\) are iid copies of \(V\).
We then apply Lemma~\ref{lem:weighted-edgeworth-poly-growth} together with the slackened truncation convention of Remark~\ref{rem:edgeworth-slack-convention} to the deleted weights \(w_j=\theta_r(j)\mathbf 1_{\{j\ne a\}}\). Their signed power sums are exactly
\begin{equation}
  \sum_j w_j^m=P_{r,m}-\theta_r(a)^m.
  \label{eq:app-deleted-power-sum}
\end{equation}
Substituting the inductive cumulant expansions and
\eqref{eq:app-deleted-power-sum} into \eqref{eq:app-weighted-edgeworth}, and then
expanding the finite polynomial, produces precisely the finite sum in
\eqref{eq:app-coeff-expansion}.  The factors \(P_{u,1}\) that appear from mean
shifts have formal order zero and are bounded by \(C\) on the spherical event.
The Edgeworth remainder is \(O(d^{-J/2})\), and all discarded local coordinate
monomials have coordinate degree at least \(J\).  Multiplying by
\(\theta_r^{\odot b}\) and using Lemma~\ref{lem:appendix-spherical-event}
therefore gives \eqref{eq:app-coeff-remainder-tensorized}.  This proves the
induction step.
\end{proof}

\paragraph{Path Taylor expansion and proof of the hierarchy}
From the chain rule, we have
\begin{equation}
  \mathcal G_{\theta,{\bf i}}
  =
  \left(\prod_{r=1}^K\theta_r(i_r)\right)\Lambda_{\bf i},
  \qquad
  \Lambda_{\bf i}:=\E_X\prod_{r=1}^K D_r.
  \label{eq:app-G-Lambda}
\end{equation}

\begin{lemma}[Separated path Taylor expansion]
\label{lem:appendix-path-taylor}
For every multi-index \({\bf i}\),
\begin{equation}
  \Lambda_{\bf i}
  =
  \sum_{|\alpha|\le K-1}
  A_\alpha(i_1,\ldots,i_K)
  \prod_{r=1}^K\theta_r(i_r)^{\alpha_r}
  +\operatorname{Rem}^{\rm path}_{\bf i},
  \label{eq:app-path-taylor}
\end{equation}
where each coefficient is a finite separated sum
\begin{equation}
  A_\alpha(i_1,\ldots,i_K)
  =
  \sum_{a=1}^{N_\alpha} c_{\alpha,a}
  \prod_{r=1}^K \E_X h_{\alpha,a,r}(R_{r,i_r}),
  \label{eq:app-separated-coefficients}
\end{equation}
with \(N_\alpha=O_K(1)\) and with \(h_{\alpha,a,r}\) belonging to a finite smooth
family depending only on \(K\) and \(\phi\).  Moreover, after multiplying by the
base factor \(\prod_r\theta_r(i_r)\), the path remainder has Frobenius norm at
most
\begin{equation}
  \left\|
  \left(\prod_{r=1}^K\theta_r(i_r)\right)_{\bf i}
  \odot \operatorname{Rem}^{\rm path}
  \right\|_F
  \le C d^{-K/2}.
  \label{eq:app-path-remainder-F}
\end{equation}
\end{lemma}

\begin{proof}
The variables \(U_r\) and \(D_r\) are generated from the independent innovations
\(\xi,R_1,\ldots,R_K\) by the smooth recursion \eqref{eq:app-Ur-Dr}.  For fixed
innovations, the map
$$
  \beta=(\beta_1,\ldots,\beta_K)
  \mapsto
  \prod_{r=1}^K D_r
$$
is \(C^K\).  Taylor expansion around \(\beta=0\) to total degree \(K-1\) gives
$$
  \prod_{r=1}^K D_r
  =
  \sum_{|\alpha|\le K-1} \beta^\alpha H_\alpha(\xi,R_1,\ldots,R_K)
  +\widetilde R_K(\beta,\xi,R_1,\ldots,R_K),
$$
where the derivatives defining \(H_\alpha\) are finite sums of products of
functions of the single variables \(\xi,R_1,\ldots,R_K\).  This follows by
induction from the recursion: differentiating \(U_r=\phi(R_r+\beta_rU_{r-1})\)
or \(D_r=\phi'(R_r+\beta_rU_{r-1})\) only produces products of derivatives of
\(\phi\) evaluated at one innovation and lower-order derivatives of previous
\(U_s\)'s.  Since the innovations are independent, taking expectation factors
each product, yielding \eqref{eq:app-separated-coefficients}.

The Taylor remainder is bounded by a finite sum of terms of the form
$$
  C\prod_{r=1}^K |\beta_r|^{\alpha_r} M_{\bf i},
  \qquad |\alpha|=K,
$$
where \(M_{\bf i}\) has bounded second moment uniformly in \(d\).  This uses the
boundedness, or polynomial growth together with subgaussianity, of the required
derivatives of \(\phi\).  After multiplication by the base factor, a term with
\(|\alpha|=K\) has tensor norm
$$
  \prod_{r=1}^K \norm{\theta_r^{\odot(\alpha_r+1)}}_2
  \le C d^{-K/2}
$$
by Lemma~\ref{lem:appendix-spherical-event}.  Summing over the finitely many
\(\alpha\)'s proves \eqref{eq:app-path-remainder-F}.
\end{proof}

We can now prove Theorem~\ref{thm:edgeworth-hierarchy-body}.

\begin{proof}[Proof of Theorem~\ref{thm:edgeworth-hierarchy-body}]
Work on the high-probability event of Lemma~\ref{lem:appendix-spherical-event}
with \(M=2K+2\), and also on the event where the recursive Edgeworth expansions
of Lemma~\ref{lem:appendix-coeff-expansion} hold for the finite family of test
functions generated by Lemma~\ref{lem:appendix-path-taylor}.  This combined event
has probability at least \(1-\delta\), after adjusting constants.

We now insert the path expansion \eqref{eq:app-path-taylor} into
\eqref{eq:app-G-Lambda}.  Fix a path degree \(\alpha\) with \(|\alpha|\le K-1\).
The corresponding tensor factor from the base spike and the path monomial is
$$
  \bigotimes_{r=1}^K \theta_r^{\odot(\alpha_r+1)},
$$
which has formal tensor order \(|\alpha|\).

Set
$$
    J_\alpha := K-|\alpha|.
$$
We invoke the slackened Edgeworth convention of
Remark~\ref{rem:edgeworth-slack-convention}: for each coefficient
factor in the separated product, we retain only coefficient terms of
formal order strictly smaller than \(J_\alpha\).  The explicit
coefficient terms of formal order exactly \(J_\alpha\), together with
the analytic polynomial-growth Edgeworth remainder from
Lemma~\ref{lem:weighted-edgeworth-poly-growth}, are placed in the remainder.
After multiplication by the path tensor of formal order \(|\alpha|\),
this contributes
$$
    O\!\left(d^{-(|\alpha|+J_\alpha)/2}\right)
    =
    O(d^{-K/2}).
$$
Thus the retained atoms have total formal order \(q<K\), and hence
belong to the levels \(G_\theta^{[0]},\ldots,G_\theta^{[K-1]}\), while
all order-\(K\) and smaller analytic remainders are absorbed into
\(\mathcal R_K\).

Applying Lemma~\ref{lem:appendix-coeff-expansion} with \(J=J_\alpha\) gives, for
each factor \(\E h_{\alpha,a,r}(R_{r,i_r})\), a finite expansion in scalar power
sums and local coordinate powers \(\theta_r(i_r)^v\).  Multiplying the finitely
many factors in the separated product and expanding produces terms of the form
$$
  C_\tau
  \left(\prod_{u=1}^K\prod_{m=1}^{2K+2}P_{u,m}^{\ell_{\tau,u,m}}\right)
  \prod_{r=1}^K\theta_r(i_r)^{v_{\tau,r}},
$$
with coefficient formal order strictly smaller than \(J_\alpha\).  Multiplying
by the base factor and by the path monomial gives the rank-one tensor atom
$$
  C_\tau
  \left(\prod_{u=1}^K\prod_{m=1}^{2K+2}P_{u,m}^{\ell_{\tau,u,m}}\right)
  \bigotimes_{r=1}^K
  \theta_r^{\odot(1+
  \alpha_r+v_{\tau,r})}.
$$
Its formal order is
$$
  |\alpha|+
  \sum_{r=1}^K v_{\tau,r}+
  \sum_{u=1}^K\sum_{m=1}^{2K+2}\ell_{\tau,u,m}\omega_m,
$$
which is strictly less than \(K\).  Grouping all atoms with the same formal order
\(q\in\{0,\ldots,K-1\}\) defines \(\mathcal G_\theta^{[q]}\) and yields
\eqref{eq:collected-hierarchy}.

It remains to bound the accumulated errors.  The path remainder is controlled by
Lemma~\ref{lem:appendix-path-taylor}.  For the coefficient errors, note that a
coefficient error in the \(r\)-th factor is multiplied in Frobenius norm by
$$
  \prod_{s\ne r}\norm{\theta_s^{\odot(\alpha_s+1)}}_2
  \cdot
  \norm{\theta_r^{\odot(\alpha_r+1)}
        \odot\rho_{h_{\alpha,a,r},r,J_\alpha}}_2
  \le C d^{-(|\alpha|+J_\alpha)/2}.
$$
Since \(J_\alpha=K-|\alpha|\), each such error is
\(O(d^{-K/2})\).  There are only \(O_K(1)\) path terms, separated
products, and coefficient factors, so the total remainder satisfies
$$
  \norm{\mathcal R_K}_F
  \le C_{\phi,K,\delta}d^{-K/2}.
$$
The level-size estimate \eqref{eq:level-size} follows immediately from
Lemma~\ref{lem:appendix-spherical-event}, because every atom in level \(q\) has
formal order \(q\).

We next prove the collected representation and the dictionary-rank bound.
For any atom, set
$$
  n_r:=b_{\tau,r}-1,
  \qquad n=(n_1,\ldots,n_K).
$$
The tensor-power part has formal order \(|n|\).  The remaining order
\(q-|n|\) is supplied by scalar Edgeworth power sums.  Since all nonzero scalar
orders \(\omega_m\), \(m\ge3\), are even, while \(P_{r,1}\) and \(P_{r,2}\) have
order zero, we must have \(q-|n|\) even.  There is also a parity constraint in
mode one.  The first path influence enters only through the symmetric product
\(\beta_1\xi\), with \(\xi\sim N(0,1)\).  Therefore
\(\Lambda_{\bf i}\) is an even function of \(\beta_1\), so only even powers of
\(\theta_1(i_1)\) beyond the base factor survive.  Hence \(n_1\) is even.
Thus all tensor profiles belong to
$$
  \mathcal N_{K,q}
  =
  \{n\in\mathbb N_0^K:n_1\text{ even},\ |n|\le q,
  \ q-|n|\text{ even}\}.
$$
Collecting scalar coefficients with the same profile yields
\eqref{eq:collected-hierarchy}.  Since each profile gives one rank-one tensor
atom, the CP rank is at most \(|\mathcal N_{K,q}|\).

Finally, compute this cardinality.  Write \(|n|=p=q-2k\).  If \(n_1=2a\), then
the remaining \(K-1\) coordinates sum to \(p-2a\), giving
\(\binom{K+p-2a-2}{p-2a}\) possibilities.  Therefore
$$
  |\mathcal N_{K,q}|
  =
  \sum_{k=0}^{\lfloor q/2\rfloor}
  \sum_{a=0}^{\lfloor(q-2k)/2\rfloor}
  \binom{K+q-2k-2a-2}{q-2k-2a}.
$$
Setting \(j=k+a\) gives
$$
  |\mathcal N_{K,q}|
  =
  \sum_{j=0}^{\lfloor q/2\rfloor}
  (j+1)\binom{K+q-2j-2}{q-2j},
$$
which is \eqref{eq:sharp-rank-bound}.  This completes the proof of the theorem.
\end{proof}

\subsection{Identification of the first two levels}

For completeness, we also justify the explicit formulas
\eqref{eq:first-level-leading} and \eqref{eq:first-order-secondary}.  At order
zero, all path influences are set to zero and all off-spine innovations are
replaced by standard Gaussians.  Hence
$$
  \Lambda_{\bf i}^{[0]}=\bigl(\E\phi'(G)\bigr)^K={\lambda},
$$
which gives
$$
  \mathcal G_\theta^{[0]}=
  {\lambda}~\theta_1\otimes\cdots\otimes\theta_K.
$$

At first order, write
$$
  A_r:=\phi'(R_r),
  \qquad
  B_r:=\phi''(R_r).
$$
A Taylor expansion gives
$$
  D_r=A_r+\beta_r B_r U_{r-1}+O(\beta_r^2U_{r-1}^2).
$$
At zeroth order in the earlier path influences,
$$
  U_0=\xi,
  \qquad
  U_{r-1}=\phi(R_{r-1})\quad (r\ge2).
$$
Thus the first-order part of \(\Lambda_{\bf i}\) is
$$
  \sum_{r=1}^K \beta_r
  \E\left[B_rU_{r-1}\prod_{q\ne r}A_q\right].
$$
The \(r=1\) term vanishes because \(U_0=\xi\) and \(\E\xi=0\).  For
\(r\ge2\), independence of the innovations gives, at Gaussian leading order,
$$
  \E\left[B_r\phi(R_{r-1})\prod_{q\ne r}A_q\right]
  =
  \E\phi''(G)\,\E[\phi(G)\phi'(G)]\,\kappa^{K-2}
  =\nu\gamma\kappa^{K-2}.
$$
All Edgeworth and deletion corrections to this coefficient have formal order at
least two, and hence do not contribute to level \(q=1\).  Multiplying by the base
factor \(\prod_q\theta_q(i_q)\) gives exactly
$$
  \mathcal G_\theta^{[1]}
  =
  \kappa^{K-2}\nu\gamma
  \sum_{r=2}^K
  \theta_1\otimes\cdots\otimes\theta_{r-1}
  \otimes \theta_r^{\odot2}
  \otimes \theta_{r+1}\otimes\cdots\otimes\theta_K.
$$
The remainder after subtracting the first two levels is the sum of all levels
\(q\ge2\) plus \(\mathcal R_K\), and therefore has Frobenius norm
\(O(d^{-1})\) by \eqref{eq:level-size} and
\eqref{eq:edgeworth-hierarchy-gradient-proof}.

\subsection{Tensor Unfolding Estimator}
\label{sec:spectral_estimation}
We conclude this section by studying the ability of spectral methods to efficiently estimate the
MSIM model. As it turns out, it is sufficient to consider the expansion at first-order
in Theorem \ref{thm:edgeworth-hierarchy-body} to obtain strong recovery guarantees that match the sample complexity of the Tensor PCA analogue.
We note that \cite{oymak2021learning} provided a spectral recovery algorithm very similar to ours based on a low-rank Tensor Decomposition, which is not directly amenable to computationally efficient implementation. Here we complete their analysis by focusing on the computationally-efficient tensor unfolding strategy.

Theorem \ref{thm:edgeworth-hierarchy-body} establishes a strong correspondence between the MSIM model and Tensor PCA \cite{montanari2014statistical}. As a result, we can now deploy a fairly standard tensor unfolding strategy \cite{NEURIPS2023_b14d76c7} that leads to a sample complexity of order $n\gg O(d^{\lceil K/2 \rceil}\log d)$ for strong recovery of the planted directions.

Let \((X_\ell,y_\ell)_{\ell=1}^n\) be iid samples from the model, \(y_\ell=f_\theta(X_\ell),\quad X_\ell\in(\R^d)^{\ot K}.\)
Define the empirical first Stein tensor
\begin{equation}
  T_n:=\frac1n\sum_{\ell=1}^n y_\ell X_\ell.
  \label{eq:empirical-stein}
\end{equation}
Since \(f_\theta\) is Lipschitz as a function of the Gaussian vector \(X\), Gaussian integration by parts gives
\begin{equation}
  \E_X[T_n\mid\theta]
  =\E_X[Xf_\theta(X)]
  =\E_X[\nabla f_\theta(X)].
  \label{eq:stein-identity}
\end{equation}
Indeed, each contraction has operator norm one and \(\phi\) is globally Lipschitz, so the weak Gaussian Stein identity applies by standard smooth approximation.
Let
\begin{equation}
  a:=\lfloor K/2\rfloor,
  \quad
  b:=K-a=\lceil K/2\rceil,
  \quad
  D_L:=d^a,
  \quad
  D_R:=d^b,
  \quad
  D_*:=\max(D_L,D_R)=d^b.
  \label{eq:balanced-dimensions}
\end{equation}
Let $\Mat_a:(\R^d)^{\ot K}\to\R^{D_L\times D_R}$ be the matricization that groups the first \(a\) tensor modes as rows and the remaining \(b\) modes as columns.

\paragraph{Recursive unfolding}
We consider the following estimator:
\begin{enumerate}[label=\textbf{Step \arabic*.}, leftmargin=1.5cm]
\item Define $\widehat M:=\Mat_a(T_n)\in\R^{D_L\times D_R}.$
\item Let \(\widehat u,\widehat v\) be top left and right singular vectors of \(\widehat M\).
\item Apply the following recursive rank-one vector factorization to \(\widehat u\) and \(\widehat v\). Given a unit vector \(q\in\R^{d^m}\), if \(m=1\), return \(q\). If \(m\ge2\), set \(m_1=\lfloor m/2\rfloor\), \(m_2=m-m_1\), reshape \(q\) as a \(d^{m_1}\times d^{m_2}\) matrix, take its top left and right singular vectors \(q_L,q_R\), and recurse on \(q_L\) and \(q_R\).
\end{enumerate}

We consider the following mild regularity assumption on $\phi$:
\begin{equation}
  \phi\in C^3(\R),
  \qquad
  \norm{\phi'}_\infty+
  \norm{\phi''}_\infty+
  \norm{\phi'''}_\infty<\infty.
  \tag{A2}\label{ass:regularity}
\end{equation}

\begin{theorem}[Moment-SVD estimation guarantee]\label{thm:main-estimation}
Assume \eqref{ass:normalization} and \eqref{ass:regularity}.
Fix directions \(\theta_1,\ldots,\theta_K\in S^{d-1}\) for which the population decomposition \eqref{eq:frobenius-sphere-updated} holds with \(\norm{W_\theta}_F\le b_d\). Define
\begin{equation}
  r_{n,d}:=b_d+C_{\phi,K}\left[
  \sqrt{\frac{D_*\log((D_L+D_R)/\delta)}{n}}
  +
  \frac{\sqrt{D_*}\,\log^{3/2}(n/\delta)\log((D_L+D_R)/\delta)}{n}
  \right].
  \label{eq:rnd}
\end{equation}
Conditional on these directions, with probability at least \(1-\delta\) over the samples, if  $r_{n,d}\le c_K\abs\lambda$,
then the recursive SVD estimator satisfies
\begin{equation}
  \max_{1\le k\le K}\distpm(\widehat\theta_k,\theta_k)
  \le
  C_K\frac{r_{n,d}}{\abs\lambda}.
  \label{eq:main-estimation-bound}
\end{equation}
The constants \(c_K,C_K\) depend only on \(K\).
\end{theorem}
\begin{proof}
We use a standard matrix Bernstein concentration argument combined with a recursive application of Wedin's sine theorem, as in \cite{montanari2014statistical}.

On an event where the population bias satisfies \(\norm{W_\theta}_F\le b_d\), equations \eqref{eq:frobenius-sphere-updated} and \eqref{eq:stein-identity} imply
\begin{equation}
  M_*:=\Mat_a(\E_X[T_n\mid\theta])
  =\lambda uv^\top+\Mat_a(W_\theta),
  \qquad
  \norm{\Mat_a(W_\theta)}_\op\le b_d.
  \label{eq:population-matrix}
\end{equation}

The matrices in the empirical average are independent across samples, but their entries are not independent within a sample. The following lemma is the needed matrix Bernstein bound, with truncation to handle the unbounded Lipschitz activation.

\begin{lemma}[Matricized empirical Stein concentration]\label{lem:matrix-conc-app}
Assume \eqref{ass:regularity}. Conditional on any fixed directions \(\theta_1,\ldots,\theta_K\in S^{d-1}\), there is a constant \(C_{\phi,K}<\infty\) such that
$$
  \norm{f_\theta(X)}_{\psi_2}\le C_{\phi,K}.
$$
Consequently, for every \(\delta\in(0,1)\), with probability at least \(1-\delta\) over the samples,
\begin{equation}
  \norm{\Mat_a(T_n-\E_XT_n)}_\op
  \le
  C_{\phi,K}\left[
  \sqrt{\frac{D_*\log((D_L+D_R)/\delta)}{n}}
  +
  \frac{\sqrt{D_*}\,\log^{3/2}(n/\delta)\log((D_L+D_R)/\delta)}{n}
  \right].
  \label{eq:matrix-conc-full-app}
\end{equation}
In particular, if
$$
  n\ge C_{\phi,K}D_*\log^4((D_L+D_R)n/\delta),
$$
then the first term dominates and
\begin{equation}
  \norm{\Mat_a(T_n-\E_XT_n)}_\op
  \le
  C_{\phi,K}
  \sqrt{\frac{D_*\log((D_L+D_R)/\delta)}{n}}.
  \label{eq:matrix-conc-simple-app}
\end{equation}
\end{lemma}

We now consider two elementary perturbation facts.

\begin{lemma}[Rank-one Wedin bound]\label{lem:wedin-app}
Let
$$
  \widehat M=\lambda uv^\top+E,
  \qquad
  \norm u_2=\norm v_2=1,
  \qquad
  \lambda\ne0.
$$
If \(\norm E_\op\le\abs\lambda/4\), then any top left and right singular vectors \(\widehat u,\widehat v\) of \(\widehat M\) satisfy
\begin{equation}
  \distpm(\widehat u,u)
  \le C\frac{\norm E_\op}{\abs\lambda},
  \qquad
  \distpm(\widehat v,v)
  \le C\frac{\norm E_\op}{\abs\lambda}.
  \label{eq:wedin-app}
\end{equation}
\end{lemma}

\begin{lemma}[Recursive Kronecker factor extraction]\label{lem:recursive-app}
Fix \(m\ge1\) and unit vectors \(q_1,\ldots,q_m\in S^{d-1}\). Let
$$
  q=q_1\ot\cdots\ot q_m\in\R^{d^m}.
$$
Suppose \(\widehat q\in\R^{d^m}\) is unit norm and
$$
  \distpm(\widehat q,q)\le\eps\le c_m
$$
for a sufficiently small constant \(c_m>0\). Apply the recursive SVD factorization from Section 3.1 to \(\widehat q\). Then the returned vectors \(\widehat q_1,\ldots,\widehat q_m\) satisfy
\begin{equation}
  \max_{1\le r\le m}\distpm(\widehat q_r,q_r)
  \le C_m\eps,
  \label{eq:recursive-factor-bound-app}
\end{equation}
where \(C_m\) depends only on \(m\).
\end{lemma}

By \eqref{eq:population-matrix},
$$
  \widehat M=\Mat_a(T_n)=\lambda uv^\top+E,
$$
where
$$
  E:=\Mat_a(W_\theta)+\Mat_a(T_n-\E_XT_n).
$$
Thus
$$
  \norm E_\op
  \le
  \norm{W_\theta}_F+
  \norm{\Mat_a(T_n-\E_XT_n)}_\op
  \le r_{n,d}
$$
with probability at least \(1-\delta\) over the samples, by Lemma \ref{lem:matrix-conc}. If \(r_{n,d}\le c\abs\lambda\), Lemma \ref{lem:wedin} gives
$$
  \distpm(\widehat u,u)+\distpm(\widehat v,v)
  \le C\frac{r_{n,d}}{\abs\lambda}.
$$
Applying Lemma \ref{lem:recursive} to \(\widehat u\) and to \(\widehat v\) yields, for every factor in the left and right blocks,
$$
  \distpm(\widehat\theta_k,\theta_k)
  \le C_K\frac{r_{n,d}}{\abs\lambda}.
$$
This proves \eqref{eq:main-estimation-bound}.
\end{proof}

Combining this pointwise (in the planted direction) control with the uniform prior leads to

\begin{corollary}[Uniform planted prior]\label{cor:sample-size_body}
Let
 $ \theta_1,\ldots,\theta_K\stackrel{\mathrm{iid}}{\sim}\Unif(S^{d-1}).$
Assume fixed \(K\) and fixed activation \(\phi\) satisfying \eqref{ass:normalization}--\eqref{ass:regularity}.
Then, with probability $1-o_d(1)$ over both the random directions and the samples,
\begin{equation}
  \max_{1\le k\le K}\distpm(\widehat\theta_k,\theta_k)
  \le
  C_{\phi,K,\kappa}
  \left(
  \sqrt{\frac{\log d}{d}}
  +
  \sqrt{\frac{d^{\lceil K/2\rceil}\log d}{n}}
  \right).
  \label{eq:sharp-statistical-form}
\end{equation}
Consequently, if  $n\gg d^{\lceil K/2\rceil} \log d$,
then  $\max_{1\le k\le K}\distpm(\widehat\theta_k,\theta_k)
  =o_{d}\left(1\right)$.
\end{corollary}
This rate coincides, up to logarithmic factors, with the balanced-unfolding computational threshold in Tensor PCA \cite{dudeja2021statistical}. Whether this threshold reflects a genuine computational barrier for MSIM remains an interesting open question.

\begin{proof}[Proof of Corollary \ref{cor:sample-size_body}]
Let us consider the confidence levels $\delta_\theta = 1/\log d$ over directions and $\delta_{X} = d^{-c}$ over samples.
Write \(\theta=g/\norm g_2\), where \(g\sim N(0,I_d)\). Standard Gaussian concentration gives, after a union bound over \(r\le K\),
\begin{equation}
  \max_{1\le r\le K}\norm{\theta_r}_\infty
  \le C\sqrt{\frac{\log(8Kd/\delta)}{d}}
  \label{eq:sphere-infty}
\end{equation}
with probability at least \(1-\delta/4\). Similarly,
$$
  \sum_{j=1}^d\theta(j)=\frac{\sum_{j=1}^d g_j}{\norm g_2},
$$
with \(\sum_jg_j\sim N(0,d)\), so
\begin{equation}
  \max_{1\le r\le K}\abs{\sum_{j=1}^d\theta_r(j)}
  \le C\sqrt{\log(8K/\delta)}
  \label{eq:sphere-s}
\end{equation}
with probability at least \(1-\delta/4\).
Standard moment concentration for \(g\sim N(0,I_d)\) gives, again after a union bound over \(r\le K\),
\begin{equation}
  \max_{r\le K}\sum_{j=1}^d\abs{\theta_r(j)}^3
  \le C_{K,\delta}d^{-1/2},
  \qquad
  \max_{r\le K}\norm{\theta_r}_4^4
  \le C_{K,\delta}d^{-1}
  \label{eq:sphere-zeta-l4}
\end{equation}
with probability at least \(1-\delta/2\).

Then using \eqref{eq:frobenius-sphere-updated} and the bounds on $s$, $\zeta$ and $\|\theta\|_4^4$ from \eqref{eq:sphere-s}, \eqref{eq:sphere-zeta-l4}, we have, with probability $1- \delta_{\theta}$,
\begin{align}
\label{eq:frob_bound}
    \|W_\theta\|_F & \leq C_{\phi, K} \frac{(\log \log d)^{a_{K, \phi}/2}}{\sqrt{d}} \leq C_{\phi, K} \sqrt{\frac{\log d}{d}}~,
\end{align}
where the second inequality holds for sufficiently large $d$.
Concerning the empirical term, for \(\delta_X=d^{-c}\),
$$
  \log((D_L+D_R)/\delta_X)=O_{K,c}(\log d),
  \qquad
  D_*=d^{\lceil K/2\rceil}.
$$
The simplified concentration bound \eqref{eq:matrix-conc-simple} gives
\begin{align}
\label{eq:bernbound}
  \norm{\Mat_a(T_n-\E_XT_n)}_\op
  &=O_{\phi,K,c}\left(
  \sqrt{\frac{d^{\lceil K/2\rceil}\log d}{n}}
  \right)
\end{align}
whenever the lower-order Bernstein envelope term is dominated.

Substituting \eqref{eq:frob_bound} and \eqref{eq:bernbound} into Theorem \ref{thm:main-estimation} proves \eqref{eq:sharp-statistical-form}.

\end{proof}

\begin{remark}[On necessity of non-degeneracy and incoherence assumptions]
A natural question is to what extent our assumptions are necessary for efficient recovery of the planted directions. First, if $\phi$ has information-exponent $s>1$, the first chaos expansion $\langle f(X), X\rangle$ may fail to be informative, at least uniformly in the directions $(\theta_k)_k$. Indeed, its chaos expansion depends on the intricate interleaving of compositions of $\phi$ alongside linear combinations with the planted directions; we leave this extension as an interesting follow-up.
Next, it is also clear that our delocalization hypothesis is not strictly necessary: for instance, if $\|\theta_k \|_0 \leq C = O_d(1)$ for each $k \in [K]$, then we easily verify that $f(X)$ is a multi-index model with index dimension $\leq C^{K-1}$, and therefore can be efficiently learnt at the polynomial scale. Another interesting question is thus to understand the intermediate regime between delocalization and sparsity.
\end{remark}

\section{Higher chaos and Depth Separation}
\label{sec:higher-chaos-cp-incompressibility}

In this Section, we show that the staircase spectral profile of the first chaos is also present in high-order chaos terms.
As a consequence, we obtain in Theorem \ref{thm:tail-depth-separation} a \emph{depth separation} result, which verifies that the MSIM model cannot be efficiently approximated by Shallow NNs as soon as $K>1$ under a natural extension of our non-degeneracy assumptions.

The main idea of the proof, dating back to \cite{eldan2016power}, and used closer to our setting in \cite{NEURIPS2022_d65befe6, ICLR2025_9c537882}, is to show that the high-order chaos of $f_\theta$ are all sufficiently spread, in the sense that they do not admit a low-rank CP factorization. Combined with a slow decay of the harmonic tails, the resulting spectral distribution cannot be efficiently captured by sums of ridge functions.

\subsection{Towards CP incompressibility}
\label{sec:CPincompress}
The first-chaos hierarchy of Theorem~\ref{thm:edgeworth-hierarchy-body} 
gives a constructive CP approximation of the first Wiener chaos in terms of a finite number of atoms. 
For higher chaoses, we will be interested in the opposite direction, namely certifying non-compressibility.  
In this subsection we give a useful
certificate for such non-compressibility, based on balanced matrix flattenings.  The main
takeaway is that, under the natural higher-chaos non-cancellation condition,
the balanced flattening of the $R$-th chaos has a staircase singular-value profile with values of order $d^{-\rho/2}$ and multiplicity of order $d^\rho$, for $\rho = 0, \ldots, (K-1)R/2$. Consequently, any low-CP-rank approximation leaves a constant Frobenius error until the CP rank reaches the size of the corresponding staircase plateau.

Let
$$
    V_d := (\R^d)^{\otimes K},
    \qquad D:=\dim(V_d)=d^K .
$$
For an integer $R\ge2$, define the $R$-th population Stein tensor by
\begin{equation}
    \mathcal C_\theta^{(R)}
    := \E_X\big[\nabla_X^R f_\theta(X)\big]
    \in \operatorname{Sym}^R(V_d).
    \label{eq:R-chaos-definition}
\end{equation}
Equivalently, by Gaussian integration by parts, \eqref{eq:R-chaos-definition}
is the $R$-th Wiener/Hermite chaos coefficient of the target $f_\theta$.
For a subset $S\subset[R]$, write
$$
    \operatorname{Mat}_S(T)
    := \operatorname{Mat}_{S\mid S^c}(T)
    \in \R^{D^{|S|}\times D^{R-|S|}}
$$
for the flattening of $T\in V_d^{\otimes R}$ with the tensor copies in $S$ on
the row side and those in $S^c$ on the column side.  We shall mostly take a
balanced cut, i.e. $|S|=\lfloor R/2\rfloor$.  When $R$ is even and
$|S|=R/2$, the matrix is symmetric after the canonical identification of the
row and column spaces, and its singular values are the absolute values of its
eigenvalues.

For $m\ge0$, let
$$
    e_m^{\rm CP}(T)
    := \inf_{\operatorname{rank}_{\rm CP}(A)\le m}\|T-A\|_F
$$
be the best CP-rank-$m$ approximation error.

\begin{lemma}[Flattening lower bound for CP approximation]
\label{lem:flattening-lower-bound-cp}
For every $T\in V_d^{\otimes R}$, every $S\subset[R]$, and every $m\ge0$,
\begin{equation}
    e_m^{\rm CP}(T)
    \ge
    \left(\sum_{j>m}
    \sigma_j\big(\operatorname{Mat}_S(T)\big)^2\right)^{1/2},
    \label{eq:cp-error-lower-bound-flattening}
\end{equation}
where $\sigma_1\ge\sigma_2\ge\cdots$ denote singular values.
\end{lemma}

\begin{proof}
A CP rank-one tensor $u_1\otimes\cdots\otimes u_R$ becomes the rank-one matrix
$$
    \left(\bigotimes_{a\in S}u_a\right)
    \left(\bigotimes_{a\in S^c}u_a\right)^\top
$$
under the flattening $\operatorname{Mat}_S$.  Hence every CP-rank-$m$ tensor
has flattening rank at most $m$.  Applying the Eckart--Young theorem to the
matrix $\operatorname{Mat}_S(T)$ gives \eqref{eq:cp-error-lower-bound-flattening}.
\end{proof}

\subsection{The leading $R$-spine Gaussian expansion}
\label{subsec:R-spine-leading-expansion}

We now compute the leading Gaussian approximation of the $R$-th chaos, generalizing Theorem \ref{thm:edgeworth-hierarchy-body}. The new feature of the high-order expansion is that now it is naturally a \emph{two-parameter} expansion, where the previous Edgeworth order from signed power sums is now supplemented by additional local coordinate powers coming from high-order derivative coalescence structure.
Let us start with additional definitions needed in the following.

Let
$\Pi_R$ be the lattice of partitions of $[R]$, ordered by refinement:
$\pi\preceq\pi'$ means that $\pi$ refines $\pi'$.  Let $\widehat 0_R$ denote
the discrete partition and $\widehat 1_R$ the one-block partition.  An
$R$-spine coalescence chain is a sequence
\begin{equation}
    \boldsymbol\pi=(\pi_1,\ldots,\pi_{K+1})
    \quad\text{such that}\quad
    \pi_1=\widehat 0_R,
    \quad \pi_\ell\preceq\pi_{\ell+1}\ (1\le\ell\le K),
    \quad \pi_{K+1}=\widehat 1_R .
    \label{eq:R-spine-chain-definition}
\end{equation}
We denote by $\mathfrak P_{K,R}$ the finite set of such coalescence chains.
The partition $\pi_\ell$ encodes which derivative spines have already
coalesced before choosing the physical coordinate at scale $\ell$.  Thus a
non-singleton block of $\pi_\ell$ produces a copy-diagonal constraint at the
$\ell$-th physical coordinate.  The final partition $\pi_{K+1}=\widehat 1_R$
encodes the fact that all spines eventually merge at the scalar output, but it does not
itself create an additional physical-coordinate diagonal.

For a unit vector $u\in S^{d-1}$ and a nonempty block $B\subset[R]$, define the
copy-diagonal tensor
\begin{equation}
    \Delta_{u,B}
    := \sum_{i=1}^d u(i)\bigotimes_{r\in B}e_i^{(r)}
    \in \bigotimes_{r\in B}\R^d .
    \label{eq:copy-diagonal-block}
\end{equation}
When $B=\{r\}$ this is simply the vector $u$ placed in copy $r$.  For a
partition $\pi\in\Pi_R$, set
\begin{equation}
    \Delta_{u,\pi}
    := \bigotimes_{B\in\pi}\Delta_{u,B}
    \in (\R^d)^{\otimes R},
    \label{eq:copy-diagonal-partition-tensor}
\end{equation}
where the tensor factors are placed in the natural copy modes $1,\ldots,R$.
Given a coalescence chain $\boldsymbol\pi$, define the associated multiscale
copy-diagonal atom
\begin{equation}
    \mathcal A_{\boldsymbol\pi}(\theta)
    := \bigotimes_{\ell=1}^K \Delta_{\theta_\ell,\pi_\ell}
    \in \big((\R^d)^{\otimes K}\big)^{\otimes R}=V_d^{\otimes R},
    \label{eq:multiscale-copy-diagonal-atom}
\end{equation}
with the obvious reordering from layer-major to copy-major tensor indices.

Let
\begin{equation}
    \mu_a := \E_{G\sim N(0,1)}\big[\phi^{(a)}(G)\big],
    \qquad a\ge1,
    \label{eq:mu-a-definition}
\end{equation}
so that $\mu_1=\kappa$.  For $B\in\pi_{\ell+1}$, let
\begin{equation}
    n_\ell(B)
    := \#\{C\in\pi_\ell: C\subseteq B\}
    \label{eq:transition-child-count}
\end{equation}
be the number of child blocks of $B$ in the transition
$\pi_\ell\preceq\pi_{\ell+1}$.  Define the Gaussian coalescence coefficient
\begin{equation}
    c_\phi(\boldsymbol\pi)
    := \prod_{\ell=1}^{K}\prod_{B\in\pi_{\ell+1}}
    \mu_{n_\ell(B)} .
    \label{eq:coalescence-coefficient}
\end{equation}
Thus each node of the $R$-spine chain contributes the Gaussian derivative
moment corresponding to the number of incoming child branches at that node.

For the flattening estimates we fix a cut $S\subset[R]$.  For a block
$B\subset[R]$, say that $B$ crosses the cut if
\begin{equation}
    B\cap S\ne\varnothing
    \qquad\text{and}\qquad
    B\cap S^c\ne\varnothing .
    \label{eq:cut-crossing-block-definition}
\end{equation}
For a partition $\pi\in\Pi_R$, write
\begin{equation}
    \chi_S(\pi)
    :=\{B\in\pi: B\cap S\ne\varnothing,\ B\cap S^c\ne\varnothing\},
    \qquad
    c_S(\pi):=|\chi_S(\pi)| .
    \label{eq:local-crossing-number}
\end{equation}
For a chain $\boldsymbol\pi$, define its crossing complexity by
\begin{equation}
    \rho_S(\boldsymbol\pi)
    := \sum_{\ell=1}^K c_S(\pi_\ell).
    \label{eq:crossing-complexity-chain}
\end{equation}
Since $\pi_1=\widehat 0_R$, the first layer contributes no crossing blocks.  If
$|S|=s\le R-s$, then
\begin{equation}
    0\le \rho_S(\boldsymbol\pi)\le (K-1)s.
    \label{eq:rho-max-balanced}
\end{equation}
For a balanced cut this gives
$$
    \rho_{\max}=(K-1)\lfloor R/2\rfloor .
$$

We now make explicit the notation used in the two-parameter $(K,R)$ hierarchy.
Fix a truncation order $A$. The $R$-spine Fa\`a di Bruno
expansion, followed by the same weighted Edgeworth expansion used in Theorem~\ref{thm:edgeworth-hierarchy-body}, produces a finite set $\mathfrak T_{K,R,A}$ of decorated spine symbols.  A symbol $\tau\in\mathfrak T_{K,R,A}$ consists of:
\begin{itemize}
    \item a coalescence chain $\boldsymbol\pi(\tau)\in\mathfrak P_{K,R}$;
    \item a scalar coefficient $c_\tau(\phi)$ depending only on finitely many
    Gaussian derivative moments of $\phi$;
    \item nonnegative integers $p_{\tau,u,m}$, encoding the power-sum monomial
    \begin{equation}
        P_\tau(\theta)
        :=\prod_{u=1}^K\prod_{m=3}^{m_A}
        P_{u,m}^{p_{\tau,u,m}},
        \qquad
        P_{u,m}=\sum_{a=1}^d\theta_u(a)^m;
        \label{eq:decorated-power-sum-monomial}
    \end{equation}
    \item integers $b_{\tau,\ell,B}\ge1$, for
    $1\le\ell\le K$ and $B\in\pi_\ell(\tau)$, encoding the local coordinate
    powers.
\end{itemize}
Here $m_A$ and all exponents are bounded by constants depending only on
$(K,R,A)$.  We use the notation $\Delta_{v,B}$ from
\eqref{eq:copy-diagonal-block} also when $v\in\R^d$ is not a unit vector.  The
matrix atom associated with $\tau$ is
\begin{equation}
    \mathsf A_{S,\tau}(\theta)
    :=
    \operatorname{Mat}_S\left(
    \bigotimes_{\ell=1}^K
    \bigotimes_{B\in\pi_\ell(\tau)}
    \Delta_{\theta_\ell^{\odot b_{\tau,\ell,B}},B}
    \right).
    \label{eq:decorated-R-spine-matrix-atom}
\end{equation}
Define its two orders by
\begin{align}
    q(\tau)
    &:=
    \sum_{u=1}^K\sum_{m=3}^{m_A}p_{\tau,u,m}\omega_m
    +\sum_{\ell=1}^K\sum_{B\in\pi_\ell(\tau)}(b_{\tau,\ell,B}-1),
    \label{eq:decorated-edgeworth-order}\\
    \rho_S(\tau)
    &:=\rho_S(\boldsymbol\pi(\tau))
    =\sum_{\ell=1}^K c_S(\pi_\ell(\tau)).
    \label{eq:decorated-crossing-order}
\end{align}
The first term in $q(\tau)$ is the usual scalar Edgeworth order coming from
signed power sums, while the second term encodes the extra local coordinate
powers beyond the base copy-diagonal factor.  With this convention, one
crossing block contributes the matrix scale $d^{-1/2}$ through $\rho_S(\tau)$;
all other small factors are counted in $q(\tau)$.
For fixed $A$ and $S$ we define, suppressing the harmless dependence on $A$,
the sum of all flattened
Edgeworth atoms whose non-Gaussian, non-diagonal order is $q$ and whose matrix
crossing complexity is $\rho$, as
\begin{equation}
    \mathcal M_S^{[q,\rho]}
    :=
    \sum_{\substack{\tau\in\mathfrak T_{K,R,A}:\ q(\tau)=q,\ \rho_S(\tau)=\rho}}
    c_\tau(\phi)P_\tau(\theta)\mathsf A_{S,\tau}(\theta).
    \label{eq:M-S-q-rho-definition}
\end{equation}
When $q=0$, all power-sum decorations are absent and all
$b_{\tau,\ell,B}=1$; hence
\begin{equation}
    \mathcal M_S^{[0,\rho]}
    =
    \sum_{\boldsymbol\pi\in\mathfrak P_{K,R}:\rho_S(\boldsymbol\pi)=\rho}
    c_\phi(\boldsymbol\pi)
    \operatorname{Mat}_S\big(\mathcal A_{\boldsymbol\pi}(\theta)\big).
    \label{eq:M-S-zero-rho-explicit}
\end{equation}

\begin{proposition}[Leading Gaussian $R$-spine expansion]
\label{prop:leading-R-spine-expansion}
Assume the smoothness condition of Assumption~\ref{ass:edgeworth-regularity} at an order
large enough depending only on $(K,R,A)$.  The leading, zeroth-Edgeworth level of
\eqref{eq:R-chaos-definition} is
\begin{equation}
    \mathcal C_\theta^{(R),[0]}
    = \sum_{\boldsymbol\pi\in\mathfrak P_{K,R}}
      c_\phi(\boldsymbol\pi)\mathcal A_{\boldsymbol\pi}(\theta),
    \label{eq:leading-R-spine-expansion}
\end{equation}
Moreover, for every fixed $A$ and every
cut $S\subset[R]$, the full $R$-th chaos admits the two-parameter Edgeworth
refinement
\begin{equation}
    \operatorname{Mat}_S\big(\mathcal C_\theta^{(R)}\big)
    = \sum_{a=0}^{A}\mathcal M_{S}^{\langle a\rangle}+\mathcal E_{S,A+1},
    \qquad
    \mathcal M_{S}^{\langle a\rangle}
    = \sum_{q+\rho=a}\mathcal M_{S}^{[q,\rho]},
    \label{eq:R-chaos-two-parameter-hierarchy}
\end{equation}
where $\mathcal M_S^{[q,\rho]}$ is defined explicitly in
\eqref{eq:M-S-q-rho-definition}.  On the spherical incoherence event of
Lemma~\ref{lem:appendix-spherical-event},
\begin{equation}
    \operatorname{rank}\big(\mathcal M_S^{[q,\rho]}\big)
    \le C_{K,R,A}d^\rho,
    \qquad
    \big\|\mathcal M_S^{[q,\rho]}\big\|_{\rm op}
    \le C_{\phi,K,R,A}(\log d)^{C_{K,R,A}}d^{-(q+\rho)/2},
    \label{eq:R-chaos-rank-op-bound}
\end{equation}
while
\begin{equation}
    \big\|\mathcal E_{S,A+1}\big\|_{\rm op}
    \le C_{\phi,K,R,A}(\log d)^{C_{K,R,A}}d^{-(A+1)/2}.
    \label{eq:R-chaos-remainder-op-bound}
\end{equation}
The level $q=0$ in \eqref{eq:R-chaos-two-parameter-hierarchy} is exactly the
flattening of \eqref{eq:leading-R-spine-expansion} grouped by crossing
complexity.
\end{proposition}

\begin{proof}
Write an input coordinate in the $a$-th
copy as
$\mathbf i^{(a)}=(i^{(a)}_1,\ldots,i^{(a)}_K)$.  Applying
$\partial_{\mathbf i^{(1)}}\cdots\partial_{\mathbf i^{(R)}}$ to the computation
tree of the MSIM produces $R$ derivative spines.  At physical layer $\ell$, two
spines are in the same block precisely when, before the coordinate
$i_\ell^{(a)}$ is chosen, they have coalesced into the same upper-layer node.
Thus the possible coalescence histories are exactly the chains
$\boldsymbol\pi=(\pi_1,\ldots,\pi_{K+1})$ in
\eqref{eq:R-spine-chain-definition}.

Fix such a chain.  During the transition
$\pi_\ell\preceq\pi_{\ell+1}$, each block $B\in\pi_{\ell+1}$ is formed from
$n_\ell(B)$ child blocks of $\pi_\ell$.  Applying the Fa\`a di Bruno formula therefore contributes one
local derivative factor
$\phi^{(n_\ell(B))}$ at the corresponding activation node.  The linear
contraction at scale $\ell$ contributes one factor of $\theta_\ell$ for each
block of $\pi_\ell$, with all spines in the block sharing the same physical
coordinate.  Consequently the coordinate tensor associated with the chain is
exactly
$$
    \bigotimes_{\ell=1}^K \Delta_{\theta_\ell,\pi_\ell}
    =\mathcal A_{\boldsymbol\pi}(\theta).
$$
At zeroth Edgeworth order, the finitely many preactivations appearing along
this genealogy are replaced by independent standard Gaussians.  Hence the
expectation of the derivative factors factorizes as
$$
    \prod_{\ell=1}^{K}\prod_{B\in\pi_{\ell+1}}
    \E\phi^{(n_\ell(B))}(G)
    =c_\phi(\boldsymbol\pi),
$$
which proves \eqref{eq:leading-R-spine-expansion} after summing over all chains.

We now derive the full two-parameter expansion.  For each fixed chain, delete
the finitely many coordinates lying on its $R$ spines.  The remaining
preactivation variables entering the derivative factors are weighted sums of
off-spine descendants with weights given by the coordinates of
$\theta_1,\ldots,\theta_K$.  Applying the weighted Edgeworth expansion from the
proof of Theorem~\ref{thm:edgeworth-hierarchy-body} to this finite family of
smooth test functions gives a finite expansion, up to total order $A$, in
monomials of the signed power sums $P_{u,m}$ and in additional local coordinate
powers.  Every term is therefore of the decorated form
\begin{equation}
    c_\tau(\phi)P_\tau(\theta)\mathsf A_{S,\tau}(\theta),
    \qquad \tau\in\mathfrak T_{K,R,A},
    \label{eq:decorated-term-form-in-proof}
\end{equation}
with $P_\tau$, $\mathsf A_{S,\tau}$, $q(\tau)$ and $\rho_S(\tau)$ defined in
\eqref{eq:decorated-power-sum-monomial}--\eqref{eq:decorated-crossing-order}.
Grouping the terms in \eqref{eq:decorated-term-form-in-proof} by the value of
$q(\tau)+\rho_S(\tau)=a$ gives exactly
\eqref{eq:R-chaos-two-parameter-hierarchy}, with
$\mathcal M_S^{[q,\rho]}$ defined by \eqref{eq:M-S-q-rho-definition}.  The
terms with $q(\tau)+\rho_S(\tau)>A$, together with the analytic Edgeworth
remainders, are absorbed into $\mathcal E_{S,A+1}$.

It remains to prove the rank and operator bounds.  We first show the local
matrix estimate.  Let $v_B\in\R^d$ be arbitrary vectors and set
$T_\pi:=\bigotimes_{B\in\pi}\Delta_{v_B,B}$.  If $B$ crosses the cut, then
$$
    \operatorname{Mat}_S(\Delta_{v_B,B})
    =\sum_{a=1}^d v_B(a)
    \left(\bigotimes_{r\in B\cap S}e_a^{(r)}\right)
    \left(\bigotimes_{r\in B\cap S^c}e_a^{(r)}\right)^\top,
$$
so its nonzero singular values are $\{|v_B(a)|:a\in[d]\}$.  If $B$ does not
cross the cut, the flattening is a row or column vector of Euclidean norm
$\|v_B\|_2$.  Since singular values multiply under tensor products,
\begin{equation}
    \operatorname{rank}\big(\operatorname{Mat}_S(T_\pi)\big)
    \le d^{c_S(\pi)},
    \qquad
    \big\|\operatorname{Mat}_S(T_\pi)\big\|_{\rm op}
    \le
    \prod_{B\in\chi_S(\pi)}\|v_B\|_\infty
    \prod_{B\notin\chi_S(\pi)}\|v_B\|_2 .
    \label{eq:decorated-local-rank-op}
\end{equation}
We now apply this estimate with $v_B=\theta_\ell^{\odot b_{\tau,\ell,B}}$ at each layer and
multiply over $\ell$.  On the event of Lemma~\ref{lem:appendix-spherical-event},
$$
    \|\theta_\ell^{\odot b}\|_2\le C d^{-(b-1)/2},
    \qquad
    \|\theta_\ell^{\odot b}\|_\infty
    \le C(\log d)^{b/2}d^{-b/2}
    \le C(\log d)^{C_{K,R,A}}d^{-1/2}d^{-(b-1)/2}.
$$
Therefore
\begin{equation}
    \operatorname{rank}\big(\mathsf A_{S,\tau}(\theta)\big)
    \le d^{\rho_S(\tau)},
    \qquad
    \|\mathsf A_{S,\tau}(\theta)\|_{\rm op}
    \le C(\log d)^{C_{K,R,A}}
    d^{-\rho_S(\tau)/2}
    d^{-\frac12\sum_{\ell,B}(b_{\tau,\ell,B}-1)} .
    \label{eq:decorated-atom-rank-op}
\end{equation}
The scalar power-sum part satisfies, again by
Lemma~\ref{lem:appendix-spherical-event},
\begin{equation}
    |P_\tau(\theta)|
    \le
    C d^{-\frac12\sum_{u,m}p_{\tau,u,m}\omega_m} .
    \label{eq:decorated-power-sum-bound}
\end{equation}
Combining \eqref{eq:decorated-atom-rank-op} and
\eqref{eq:decorated-power-sum-bound} with the definition of $q(\tau)$ gives
\begin{equation}
    \operatorname{rank}\big(P_\tau(\theta)\mathsf A_{S,\tau}(\theta)\big)
    \le d^{\rho_S(\tau)},
    \qquad
    \big\|P_\tau(\theta)\mathsf A_{S,\tau}(\theta)\big\|_{\rm op}
    \le C(\log d)^{C_{K,R,A}}d^{-(q(\tau)+\rho_S(\tau))/2}.
    \label{eq:single-decorated-term-bound}
\end{equation}
There are only $O_{K,R,A}(1)$ decorated symbols, and the coefficients
$c_\tau(\phi)$ are bounded in terms of $(\phi,K,R,A)$.  Summing the terms with
$q(\tau)=q$ and $\rho_S(\tau)=\rho$ proves
\eqref{eq:R-chaos-rank-op-bound}.

Finally, every term placed in $\mathcal E_{S,A+1}$ has
$q(\tau)+\rho_S(\tau)\ge A+1$, and the analytic Edgeworth remainder has the same
formal order by the choice of smoothness/truncation.  Repeating the bound
\eqref{eq:single-decorated-term-bound} for these terms gives
\eqref{eq:R-chaos-remainder-op-bound}.  For $q=0$ there are no power-sum
corrections and no extra coordinate powers, so
\eqref{eq:M-S-zero-rho-explicit} is exactly the flattening of the leading
Gaussian expansion grouped by crossing complexity.  This completes the proof.
\end{proof}

\subsection{Singular values of one multiscale atom}
\label{subsec:singular-values-one-atom}

Proposition \ref{prop:leading-R-spine-expansion} provides a representation of the high-order chaos in terms of tensor products of the form $\Delta_{u, \pi}$. Recall that we are ultimately interested in controlling the spectral decay of the associated matrix flattenings. A first step is to understand the singular values of each of these atoms:

\begin{lemma}[Exact singular values of copy-diagonal atoms]
\label{lem:copy-diagonal-atom-singular-values}
Let $u\in S^d$ be a unit vector and let $\pi\in\Pi_R$.  The nonzero singular
values of $\operatorname{Mat}_S(\Delta_{u,\pi})$ are
\begin{equation}
    \left\{
    \prod_{B\in\chi_S(\pi)} |u(a_B)|:
    (a_B)_{B\in\chi_S(\pi)}\in[d]^{\chi_S(\pi)}
    \right\},
    \label{eq:local-atom-singular-values}
\end{equation}
where $\chi_S(\pi)$ is the set of crossing blocks of $\pi$.  In particular,
\begin{equation} \operatorname{rank}\big(\operatorname{Mat}_S(\Delta_{u,\pi})\big)
    = d^{c_S(\pi)} .
    \label{eq:local-atom-rank}
\end{equation}
Consequently, for a coalescence chain $\boldsymbol\pi$, the nonzero singular
values of $\operatorname{Mat}_S(\mathcal A_{\boldsymbol\pi}(\theta))$ are
\begin{equation}
    \left\{
    \prod_{\ell=1}^K\prod_{B\in\chi_S(\pi_\ell)}
    |\theta_\ell(a_{\ell,B})|:
    a_{\ell,B}\in[d]
    \right\},
    \label{eq:multiscale-atom-singular-values}
\end{equation}
and its rank is $d^{\rho_S(\boldsymbol\pi)}$.
\end{lemma}

\begin{proof}
For a single block $B$, the flattening of \eqref{eq:copy-diagonal-block} is
$$
    \operatorname{Mat}_S(\Delta_{u,B})
    = \sum_{a=1}^d u(a)
      \left(\bigotimes_{r\in B\cap S}e_a^{(r)}\right)
      \left(\bigotimes_{r\in B\cap S^c}e_a^{(r)}\right)^\top .
$$
If $B$ does not cross the cut, this is a row or column vector of norm
$\|u\|_2=1$, hence contributes one singular value equal to one.  If $B$ crosses
the cut, the row vectors
$\{\otimes_{r\in B\cap S}e_a^{(r)}\}_{a=1}^d$ and the column vectors
$\{\otimes_{r\in B\cap S^c}e_a^{(r)}\}_{a=1}^d$ are orthonormal families, so the
singular values are exactly $|u(1)|,\ldots,|u(d)|$.  Since
$\Delta_{u,\pi}$ is the tensor product over blocks $B\in\pi$, the singular
values multiply over crossing blocks.  This proves \eqref{eq:local-atom-singular-values}.
The multiscale atom is the tensor product of the local layer atoms, so its
singular values multiply over $\ell=1,\ldots,K$, giving
\eqref{eq:multiscale-atom-singular-values}.
\end{proof}

\begin{lemma}[Bulk coordinate event for copy-diagonal atoms]
\label{lem:bulk-coordinate-event}
Fix $K,R\ge1$, a cut $S\subset[R]$, and $\delta\in(0,1)$.  If
$\theta_1,\ldots,\theta_K\stackrel{\rm iid}{\sim}\operatorname{Unif}(\Sph^{d-1})$,
then, with probability at least $1-\delta$, there are constants
$0<c_0<C_0<\infty$ and $\gamma>0$, depending only on $(K,R,\delta)$, such that
for every $\ell\in[K]$ the good-coordinate set
\begin{equation}
    \mathsf G_\ell
    :=\left\{a\in[d]:
    c_0 d^{-1/2}\le |\theta_\ell(a)|\le C_0 d^{-1/2}
    \right\}
    \label{eq:bulk-good-coordinate-set}
\end{equation}
satisfies
\begin{equation}
    |\mathsf G_\ell|\ge \gamma d,
    \qquad
    \|\theta_\ell\|_\infty\le C_0\sqrt{\frac{\log d}{d}}.
    \label{eq:bulk-coordinate-event}
\end{equation}
On this event, the following holds uniformly over every coalescence chain
$\boldsymbol\pi\in\mathfrak P_{K,R}$.  Let
\begin{equation}
    \mathsf C_S(\boldsymbol\pi)
    :=\{(\ell,B):\ell\in[K],\ B\in\chi_S(\pi_\ell)\},
    \qquad
    \rho_S(\boldsymbol\pi)=|\mathsf C_S(\boldsymbol\pi)|.
    \label{eq:crossing-block-multiset}
\end{equation}
Then the copy-diagonal multiscale atom
$\operatorname{Mat}_S(\mathcal A_{\boldsymbol\pi}(\theta))$ has at least
$c d^{\rho_S(\boldsymbol\pi)}$ singular values, counted with multiplicity, larger
than $c d^{-\rho_S(\boldsymbol\pi)/2}$, and all of its singular values are at most
\begin{equation}
    C(\log d)^{\rho_S(\boldsymbol\pi)/2}d^{-\rho_S(\boldsymbol\pi)/2}.
    \label{eq:atom-singular-bulk-bounds}
\end{equation}
Here $c,C>0$ depend only on $(K,R,\delta)$.
\end{lemma}

\begin{proof}
Write $\theta=g/\|g\|_2$ with $g\sim N(0,I_d)$.  Choose constants
$0<a_0<a_1<\infty$ such that
$p_0:=\mathbb P(a_0\le |G|\le a_1)>0$ for $G\sim N(0,1)$.  By a binomial
concentration bound,
$$
    \#\{a\in[d]:a_0\le |g(a)|\le a_1\}\ge \frac{p_0}{2}d
$$
with probability at least $1-e^{-c d}$.  On the same event, after increasing
constants, $\|g\|_2\asymp\sqrt d$ and
$\|g\|_\infty\le C\sqrt{\log d}$ with probability at least $1-\delta/K$ for all
large $d$.  Union bounding over $\ell\in[K]$ gives
\eqref{eq:bulk-good-coordinate-set}--\eqref{eq:bulk-coordinate-event}.

We now prove the singular-value statement on this event.  Fix a coalescence chain
$\boldsymbol\pi$.  By Lemma~\ref{lem:copy-diagonal-atom-singular-values}, the
nonzero singular values of
$\operatorname{Mat}_S(\mathcal A_{\boldsymbol\pi}(\theta))$ are exactly
\begin{equation}
    \sigma(\mathbf a)
    =\prod_{(\ell,B)\in\mathsf C_S(\boldsymbol\pi)}
      |\theta_\ell(a_{\ell,B})|,
    \qquad
    \mathbf a=(a_{\ell,B})_{(\ell,B)\in\mathsf C_S(\boldsymbol\pi)}
    \in[d]^{\mathsf C_S(\boldsymbol\pi)} .
    \label{eq:atom-singular-indexed-products}
\end{equation}
Here the coordinate choices are made independently for each crossing block
$(\ell,B)$; if several crossing blocks occur at the same physical layer, they
produce independent tensor-product factors and hence independent indices in
\eqref{eq:atom-singular-indexed-products}.

Restricting in \eqref{eq:atom-singular-indexed-products} to assignments with
$a_{\ell,B}\in\mathsf G_\ell$ for every crossing block gives at least
$$
    \prod_{\ell=1}^K |\mathsf G_\ell|^{|\chi_S(\pi_\ell)|}
    \ge (\gamma d)^{\rho_S(\boldsymbol\pi)}
$$
singular values, counted with multiplicity.  Each of these satisfies
$$
    \sigma(\mathbf a)
    \ge (c_0d^{-1/2})^{\rho_S(\boldsymbol\pi)}
    =c_0^{\rho_S(\boldsymbol\pi)}d^{-\rho_S(\boldsymbol\pi)/2}.
$$
Since $\rho_S(\boldsymbol\pi)\le (K-1)\lfloor R/2\rfloor$ is bounded in terms of
$(K,R)$, the constants can be absorbed into a single $c>0$.  The upper bound
follows similarly from
$$
    \sigma(\mathbf a)
    \le \prod_{(\ell,B)\in\mathsf C_S(\boldsymbol\pi)}
      \|\theta_\ell\|_\infty
    \le C_0^{\rho_S(\boldsymbol\pi)}
       (\log d)^{\rho_S(\boldsymbol\pi)/2}d^{-\rho_S(\boldsymbol\pi)/2},
$$
and the boundedness of $\rho_S(\boldsymbol\pi)$ again absorbs constants into
$C$.
\end{proof}

\subsection{Higher-chaos non-cancellation}
\label{subsec:higher-chaos-noncancellation}

We have just seen in Lemma \ref{lem:bulk-coordinate-event} that an individual atom $\operatorname{Mat}_S(\mathcal A_{\boldsymbol\pi}(\theta))$ of crossing complexity $\rho \in\{0,\ldots,\rho_{\max}\}$ has $\sim d^{\rho}$ singular values of order $\sim d^{-\rho/2}$.
Recall from (\eqref{eq:M-S-zero-rho-explicit}) that the zeroth-Edgeworth,
complexity-$\rho$ balanced flattening is given by
\begin{equation}
    \mathcal M_S^{[0,\rho]}
    = \sum_{\boldsymbol\pi\in\mathfrak P_{K,R}:\rho_S(\boldsymbol\pi)=\rho}
    c_\phi(\boldsymbol\pi)
    \operatorname{Mat}_S\big(\mathcal A_{\boldsymbol\pi}(\theta)\big).
    \label{eq:zeroth-edgeworth-rho-component}
\end{equation}

The following assumption asks that the spectral profile of singular values of $\mathcal M_S^{[0,\rho]}$ follows the same staircase power law.

\begin{assumption}[Higher-chaos non-degeneracy]
\label{ass:higher-chaos-noncancellation}
Fix $K,R$ and a balanced cut $S\subset[R]$.  For each plateau
$\rho$ under consideration, there exist constants
$a_\rho,b_\rho>0$, depending only on $(\phi,K,R,S,\rho)$, such that on the
bulk coordinate event \eqref{eq:bulk-coordinate-event},
\begin{equation}
    \sigma_{\lfloor a_\rho d^\rho\rfloor}
    \big(\mathcal M_S^{[0,\rho]}\big)
    \ge b_\rho d^{-\rho/2}
    \label{eq:higher-chaos-noncancellation}
\end{equation}
for all sufficiently large $d$.
\end{assumption}

This condition thus rules out cancellations among the finitely many leading
coalescence patterns of the same crossing complexity, and is the natural analogue of the non-degeneracy $\kappa\ne0$: when $R=1$ there is only one spine and
\eqref{eq:coalescence-coefficient} reduces to $\kappa^K$.  For $R\ge2$, the
coefficients $c_\phi(\boldsymbol\pi)$ are explicit polynomials in the Gaussian
derivative moments
$\mu_a=\E\phi^{(a)}(G)$.  Failure of
\eqref{eq:higher-chaos-noncancellation} is therefore a finite algebraic
cancellation among these coefficients.  In particular, it is non-generic in
$\phi$ once at least one admissible coalescence symbol at the desired
complexity is nonzero.

\begin{example}[Hessian]
\label{ex:hessian-noncancellation}
Let $R=2$ and $S=\{1\}$.  A chain is determined by the merge layer
$s\in\{1,\ldots,K\}$.  The coefficient is
$$
    \mu_2\mu_1^{K+s-2}=\nu\kappa^{K+s-2},
    \qquad \nu:=\E\phi''(G),
$$
and the crossing complexity is $\rho=K-s$.  Thus the leading Hessian flattening
has plateaus
$$
    1,d,d^2,\ldots,d^{K-1}
    \quad\text{at scales}\quad
    1,d^{-1/2},d^{-1},\ldots,d^{-(K-1)/2},
$$
provided $\nu\ne0$ in addition to $\kappa\ne0$.
\end{example}

\begin{example}[Top balanced plateau for even $R$]
\label{ex:top-balanced-plateau-even-R}
Let $R=2s$ and $|S|=s$.  The maximum crossing complexity is
$\rho_{\max}=s(K-1)$.  One way to realize it is to merge the $2s$ singleton
spines at the first layer into $s$ row-column crossing pairs, keep these pairs
unchanged through layers $2,\ldots,K$, and merge the $s$ pairs at the scalar
output.  Each such chain has coefficient
\begin{equation}
    \mu_2^s\,\mu_1^{s(K-2)}\,\mu_s .
    \label{eq:top-plateau-coefficient-even-R}
\end{equation}
Thus a simple sufficient nonvanishing condition for the top balanced plateau is
$\mu_2\mu_s\ne0$ together with the absence of cancellation among the finitely
many row-column pairings.  The latter is exactly the plateau
non-cancellation condition \eqref{eq:higher-chaos-noncancellation} at
$\rho=\rho_{\max}$.
\end{example}

\subsection{The staircase and the CP lower bound}
\label{subsec:higher-chaos-staircase-cp-bound}

We now state the CP compressibility consequence of the staircase spectral profile of the previous matrix flattening, by leveraging the lower bound of Section \ref{sec:CPincompress}. For ease of notation,
write
$$
    M_S^{(R)} := \operatorname{Mat}_S\big(\mathcal C_\theta^{(R)}\big),
    \qquad
    \rho_{\max}:=(K-1)\lfloor R/2\rfloor
$$
for a balanced cut.

\begin{theorem}[Balanced-flattening staircase]
\label{thm:balanced-flattening-staircase}
Fix $K,R\ge2$ and a balanced cut $S\subset[R]$.  Assume
Assumptions~\ref{ass:nondegen_informal}--\ref{ass:edgeworth-regularity}, with the smoothness
order large enough depending on $(K,R)$, and let
$\theta_1,\ldots,\theta_K\stackrel{\rm iid}{\sim}\operatorname{Unif}(\Sph^{d-1})$.
Assume also the higher-chaos non-cancellation condition
\eqref{eq:higher-chaos-noncancellation} for the plateaus under consideration.
Then, for every fixed $\delta\in(0,1)$, with probability at least $1-\delta$
over the directions, the following holds for all sufficiently large $d$.
For each $0\le\rho\le\rho_{\max}$ satisfying
Assumption~\ref{ass:higher-chaos-noncancellation}, there are constants
$c_\rho,C_\rho,A_\rho>0$, depending only on $(\phi,K,R,S,\rho,\delta)$, such that
\begin{equation}
    \sigma_{\lfloor c_\rho d^\rho\rfloor}
    \big(M_S^{(R)}\big)
    \ge c_\rho d^{-\rho/2},
    \label{eq:staircase-lower-bound}
\end{equation}
and
\begin{equation}
    \sigma_{\lceil C_\rho d^\rho\rceil}
    \big(M_S^{(R)}\big)
    \le C_\rho(\log d)^{A_\rho}d^{-(\rho+1)/2}.
    \label{eq:staircase-upper-bound-next-level}
\end{equation}
Equivalently, the balanced flattening has a plateau of order $d^\rho$
singular values at scale $d^{-\rho/2}$, up to logarithmic factors at the upper
edge.  If $R$ is even, the same statement holds for the absolute eigenvalues of
the symmetric balanced flattening.
\end{theorem}

\begin{proof}
We prove the lower bound first.  Fix $\rho$.  Decompose
$M_S^{(R)}$ using \eqref{eq:R-chaos-two-parameter-hierarchy} with
$A\ge \rho+1$:
$$
    M_S^{(R)}
    = \mathcal M_S^{[0,\rho]} + L_{<\rho}+E_{>\rho},
$$
where $L_{<\rho}$ is the sum of all components $\mathcal M_S^{[q,\rho']}$ with
$q+\rho'\le\rho$ except the leading piece $(q,\rho')=(0,\rho)$, and
$E_{>\rho}$ is the sum of all components of total order $q+\rho'>\rho$ plus the
Edgeworth remainder.  Since $q\ge1$ in every term of $L_{<\rho}$ or
$\rho'<\rho$, the rank bound \eqref{eq:R-chaos-rank-op-bound} gives
\begin{equation}
    \operatorname{rank}(L_{<\rho})
    \le C d^{\rho-1}
    \label{eq:lower-order-rank-bound}
\end{equation}
with the convention that this rank is zero when $\rho=0$.  The operator bounds
\eqref{eq:R-chaos-rank-op-bound}--\eqref{eq:R-chaos-remainder-op-bound} give
\begin{equation}
    \|E_{>\rho}\|_{\rm op}
    \le C(\log d)^A d^{-(\rho+1)/2}
    = o(d^{-\rho/2}).
    \label{eq:higher-order-op-bound}
\end{equation}
By the non-cancellation assumption,
$$
    \sigma_{\lfloor a_\rho d^\rho\rfloor}
    (\mathcal M_S^{[0,\rho]})
    \ge b_\rho d^{-\rho/2}.
$$
The singular-value rank inequality
$$
    \sigma_{j+r}(A+B)
    \ge \sigma_j(A)-\sigma_{r+1}(B)
$$
with $A=\mathcal M_S^{[0,\rho]}$, $B=L_{<\rho}+E_{>\rho}$, and
$r=\operatorname{rank}(L_{<\rho})$ gives, after decreasing constants,
$$
    \sigma_{\lfloor c_\rho d^\rho\rfloor}(M_S^{(R)})
    \ge c_\rho d^{-\rho/2},
$$
because $d^{\rho-1}=o(d^\rho)$ and $\|E_{>\rho}\|_{\rm op}=o(d^{-\rho/2})$.
This proves \eqref{eq:staircase-lower-bound}.

For the upper bound, write $M_S^{(R)}=L_{\le\rho}+E_{>\rho}$, where
$L_{\le\rho}$ is the sum of all terms of total order $q+\rho'\le\rho$.  By
\eqref{eq:R-chaos-rank-op-bound},
$$
    \operatorname{rank}(L_{\le\rho})\le C d^\rho,
$$
and by \eqref{eq:R-chaos-remainder-op-bound},
$$
    \|E_{>\rho}\|_{\rm op}
    \le C(\log d)^A d^{-(\rho+1)/2}.
$$
Therefore all singular values after the first $C d^\rho$ are bounded by the
operator norm of $E_{>\rho}$, which proves
\eqref{eq:staircase-upper-bound-next-level}.
\end{proof}

\begin{corollary}[CP-rank lower bound from the staircase]
\label{cor:cp-incompressibility-staircase}
Under the assumptions of Theorem~\ref{thm:balanced-flattening-staircase}, for
any plateau $\rho$ satisfying Assumption~\ref{ass:higher-chaos-noncancellation},
there is a constant $c_\rho>0$ such that
\begin{equation}
    e_m^{\rm CP}\big(\mathcal C_\theta^{(R)}\big)
    \ge
    c_\rho d^{-\rho/2}
    \left(\lfloor c_\rho d^\rho\rfloor-m\right)_+^{1/2}.
    \label{eq:cp-error-lower-bound-staircase}
\end{equation}
In particular, if $m\le c_\rho d^\rho/2$, then
\begin{equation}
    e_m^{\rm CP}\big(\mathcal C_\theta^{(R)}\big)
    \ge c_\rho .
    \label{eq:constant-cp-error-until-plateau-rank}
\end{equation}
Thus, whenever the top balanced plateau is non-canceling, any constant-error CP
approximation requires rank at least
\begin{equation}
    m\gtrsim d^{(K-1)\lfloor R/2\rfloor}.
    \label{eq:top-plateau-cp-rank-lower-bound}
\end{equation}
\end{corollary}

\begin{proof}
Combine Lemma~\ref{lem:flattening-lower-bound-cp} with the lower bound
\eqref{eq:staircase-lower-bound}.  The plateau contains order $d^\rho$ singular
values of size at least order $d^{-\rho/2}$, so its squared Frobenius mass is
order one:
$$
    d^\rho\cdot d^{-\rho}=1.
$$
This gives \eqref{eq:cp-error-lower-bound-staircase} and
\eqref{eq:constant-cp-error-until-plateau-rank}.  Taking
$\rho=\rho_{\max}$ gives \eqref{eq:top-plateau-cp-rank-lower-bound}.
\end{proof}

We thus obtain for the high-order chaos a counterpart to the constructive Edgeworth CP expansion initiated in Theorem \ref{thm:edgeworth-hierarchy-body}. The balanced flattening gives a lower bound in the opposite direction: if the flattening has $d^\rho$ singular
values at scale $d^{-\rho/2}$, then no arbitrary CP decomposition with
$o(d^\rho)$ atoms can capture that plateau in Frobenius norm.

\subsection{Hermite-energy transfer through the multiscale hierarchy}
\label{sec:hermite-energy-transfer}

The preceding subsection gives lower bounds on the CP compressibility of a fixed
higher chaos once that chaos carries non-negligible energy.  We now identify the
energy profile of the chaoses of the multiscale target, and characterize the setting in which this decay is `slow'.
The conclusion is that,
up to the same incoherence error appearing in the Edgeworth hierarchy, the
Wiener-chaos energies of \(f_\theta\) are obtained by composing the univariate
Hermite-energy generating function of \(\phi\) exactly \(K\) times.

Throughout this subsection \((h_j)_{j\ge0}\) denotes the orthonormal Hermite
basis of \(L^2(\gamma_1)\), where \(\gamma_1\) is the standard Gaussian measure.  We
write
\begin{equation}
    \phi(x)=\sum_{j\ge0}\alpha_j h_j(x),
    \qquad
    \lambda_j:=\alpha_j^2.
    \label{eq:phi-hermite-energy-expansion}
\end{equation}
Assumption~\ref{ass:nondegen_informal} gives \(\lambda_0=0\),
\(\sum_{j\ge0}\lambda_j=1\), and \(\lambda_1=\kappa^2>0\).  Let
\begin{equation}
    \mathsf{s}(q)
    := \E\big[\phi(G_1)\phi(G_2)\big]
    = \sum_{j\ge0}\lambda_j q^j,
    \qquad -1\le q\le1,
    \label{eq:phi-noise-stability-generating-function}
\end{equation}
where \((G_1,G_2)\) is a centered Gaussian pair with unit variances and
correlation \(q\).  Thus \(\mathsf{s}\) is the univariate noise-stability function of
\(\phi\), ie $\mathsf{s}(q) = \langle \phi, \mathsf{T}_q \phi \rangle$, where $\mathsf{T}_q$ is the Ornstein-Ulhenbeck operator
$\mathsf{T}_q f(x) = \E_G[ f ( q x + \sqrt{1-q^2} G)]$.

Let \(V_d=(\R^d)^{\otimes K}\).  We use the orthonormal Wiener-chaos convention
\begin{equation}
    f_\theta(X)
    = \sum_{r\ge0}\big\langle \G_r(\theta),H_r(X)\big\rangle,
    \qquad
    \Lambda_r(\theta):=\|\G_r(\theta)\|_F^2,
    \label{eq:f-wiener-chaos-energy-expansion}
\end{equation}
where \(H_r(X)\) is the order-\(r\) orthonormal Hermite tensor of the Gaussian
vector \(X\in V_d\).  With this convention
\begin{equation}
    \|f_\theta\|_{L^2(\gamma_{d^K})}^2
    = \sum_{r\ge0}\Lambda_r(\theta).
    \label{eq:f-total-chaos-energy}
\end{equation}
The raw Stein tensors \(\E[\nabla_X^r f_\theta(X)]\) differ from the normalized
coefficients \(\G_r(\theta)\) only by deterministic symmetrization and factorial
factors depending on \(r\). Since \(r\) is fixed throughout this section, these
normalization choices do not affect the polynomial rank or incompressibility
conclusions of this Section.

For \(\rho\in[-1,1]\), let
\begin{equation}
    X^\rho := \rho X+\sqrt{1-\rho^2}\,X',
    \label{eq:rho-correlated-input}
\end{equation}
where \(X'\) is an independent copy of \(X\). We now define the output noise stability analogue of $\mathsf{s}$:
\begin{equation}
    \mathsf S_\theta(\rho)
    := \E\big[f_\theta(X)f_\theta(X^\rho)\big].
    \label{eq:output-noise-stability}
\end{equation}
We again have the Ornstein-Ulhenbeck representation $\mathsf{S}(\rho) = \langle f_\theta, \mathsf{T}_\rho f_\theta \rangle$. Using the fact that $\mathsf{T}_\rho H_r = \rho^r H_r$ of the $r$-th Wiener chaos, we have, for every fixed $\theta$,
\begin{equation}
    \mathsf S_\theta(\rho)
    = \sum_{r\ge0}\Lambda_r(\theta)\rho^r,
    \qquad -1\le \rho\le1.
    \label{eq:noise-stability-chaos-generating-function}
\end{equation}
Consequently, for every fixed \(r\ge0\),
\begin{equation}
    \Lambda_r(\theta)
    = \frac{1}{r!}\frac{\mathrm{d}^r}{\mathrm{d}\rho^r}\mathsf S_\theta(0).
    \label{eq:chaos-energy-as-derivative}
\end{equation}

Consider now the iterates
\begin{equation}
    \mathsf{s}_0(\rho):=\rho,
    \qquad
    \mathsf{s}_{\ell+1}(\rho):=\mathsf{s}(\mathsf{s}_\ell(\rho)),
    \qquad 0\le \ell\le K-1,
    \label{eq:iterated-hermite-energy-generating-functions}
\end{equation}
and their associated coefficients
\begin{equation}
    \bar\Lambda_r^{(K)}:=[\rho^r]\mathsf{s}_K(\rho).
    \label{eq:limiting-chaos-energy-profile}
\end{equation}
Equivalently, if \(Z_0=1\) and \((Z_\ell)_{\ell\ge0}\) is a Galton--Watson process with offspring law \(\PP(N=j)=\lambda_j\), then
\begin{equation}
    \bar\Lambda_r^{(K)} = \PP(Z_K=r).
    \label{eq:galton-watson-energy-profile}
\end{equation}
Indeed, \(\mathsf{s}\) is the offspring probability-generating function.  In
coefficient form, \(\bar\Lambda^{(0)}_r=\one_{\{r=1\}}\) and
\begin{equation}
    \bar\Lambda_r^{(\ell+1)}
    = \sum_{m\ge0}\lambda_m
      \sum_{\substack{r_1+\cdots+r_m=r\\ r_a\ge0}}
      \prod_{a=1}^m \bar\Lambda_{r_a}^{(\ell)},
    \qquad 0\le \ell\le K-1,
    \label{eq:limiting-energy-coefficient-recursion}
\end{equation}
with the usual convention that the inner sum equals \(\one_{\{r=0\}}\) when
\(m=0\).

Our goal now is to relate the harmonic spectrum of $f_\theta$, given by $\Lambda_r(\theta)$, with the scalar energy profile $\bar\Lambda_r^{(K)}$, obtained by $K$-fold convolution described in (\ref{eq:limiting-energy-coefficient-recursion}). As we shall now see, for large $d$ these two sequences converge to each other.
For a unit vector \(u\in\Sph^{d-1}\),
set
\begin{equation}
    \eta(u):=\bigg|\sum_{i=1}^d u(i)^3\bigg|+\sum_{i=1}^d u(i)^4.
    \label{eq:eta-vector-incoherence}
\end{equation}
For planted directions \(\theta=(\theta_1,\ldots,\theta_K)\), define
\begin{equation}
    \eta_\star(\theta):=\max_{1\le \ell\le K}\eta(\theta_\ell),
    \qquad
    B_1(\theta):=1+\max_{1\le \ell\le K}
    \left|\sum_{i=1}^d\theta_\ell(i)\right| .
    \label{eq:eta-star-theta}
\end{equation}
The quantity \(B_1(\theta)\) controls possible amplification of the small
finite-dimensional mean errors by the next contraction.  On the spherical event
of Lemma~\ref{lem:appendix-spherical-event}, \(B_1(\theta)=O_{K,\delta}(1)\).
The term \(\sum_i u(i)^4\) is the usual fourth-cumulant scale, while
\(\sum_i u(i)^3\) is the signed cubic Edgeworth scale.  The signed nature of the
cubic term is what improves the random spherical prior from the deterministic
\(d^{-1/2}\) scale to the \(d^{-1}\) scale.

\begin{lemma}[Weighted bivariate Edgeworth step]
\label{lem:weighted-bivariate-edgeworth-step}
Fix integers \(J\ge0\) and \(M\ge4\).  Let \((U,V)\) be a centered pair with
\(\E U^2=\E V^2=1\), \(\E UV=q\), and with moments up to order \(M\) bounded by a
constant \(B\).  Let \((U_i,V_i)_{i=1}^d\) be iid copies of \((U,V)\), let
\(u\in\Sph^{d-1}\), and put
\begin{equation}
    S_u:=\sum_{i=1}^d u(i)U_i,
    \qquad
    T_u:=\sum_{i=1}^d u(i)V_i.
    \label{eq:weighted-bivariate-sum}
\end{equation}
If \((G_1,G_2)\) is Gaussian with covariance
\(\E G_1^2=\E G_2^2=1\), \(\E G_1G_2=q\), then for every
\(F\in C^4(\R^2)\) whose derivatives up to order four have polynomial growth,
\begin{equation}
    \E F(S_u,T_u)
    = \E F(G_1,G_2)
      + \left(\sum_{i=1}^d u(i)^3\right)\mathfrak B_F(U,V)
      + O_{F,B}\left(\sum_{i=1}^d u(i)^4\right),
    \label{eq:weighted-bivariate-edgeworth-expansion}
\end{equation}
where the cubic coefficient is explicit:
\begin{equation}
    \mathfrak B_F(U,V)
    := \frac{1}{6}\sum_{a,b,c\in\{1,2\}}
       \operatorname{cum}(W_a,W_b,W_c)
       \E\big[\partial_{abc}F(G_1,G_2)\big],
    \qquad W=(U,V).
    \label{eq:cubic-edgeworth-coefficient-explicit}
\end{equation}
Consequently,
\begin{equation}
    \left|\E F(S_u,T_u)-\E F(G_1,G_2)\right|
    \le C_{F,B}\eta(u).
    \label{eq:weighted-bivariate-edgeworth-bound}
\end{equation}
If the law of \((U,V)=(U(\rho),V(\rho))\) depends \(C^J\)-smoothly on a
parameter \(\rho\in I\subset(-1,1)\), with all moments and cumulants up to the
required order having bounded \(\rho\)-derivatives, then
\begin{equation}
    \left\|\E F(S_u(\rho),T_u(\rho))
      -\E F(G_1(\rho),G_2(\rho))\right\|_{C^J(I)}
    \le C_{F,B,J,I}\eta(u).
    \label{eq:weighted-bivariate-edgeworth-CJ-bound}
\end{equation}
\end{lemma}

\begin{proof}
This is the two-dimensional weighted Edgeworth expansion used throughout the
proof of Theorem~\ref{thm:edgeworth-hierarchy-body}.  We recall the calculation
because the cancellation scale matters here.  Replace the variables
\((U_i,V_i)\) one at a time by Gaussian pairs \((G_{1,i},G_{2,i})\) with the same
covariance.  At the \(i\)-th replacement, Taylor-expand the test function around
the partial sum excluding the \(i\)-th variable to order three.  The order-one and
order-two terms cancel because the means and covariance matrices match.  The
third-order contribution is
$$
    \frac{u(i)^3}{6}\sum_{a,b,c\in\{1,2\}}
    \operatorname{cum}(W_a,W_b,W_c)
    \E\big[\partial_{abc}F(G_1,G_2)\big],
$$
up to an error of order \(O_{F,B}(|u(i)|^4)\).  Summing over \(i\) gives the
signed cubic factor \(\sum_i u(i)^3\) and the fourth-order remainder
\(\sum_i u(i)^4\), proving \eqref{eq:weighted-bivariate-edgeworth-expansion}.
The displayed coefficient \eqref{eq:cubic-edgeworth-coefficient-explicit} is the
standard third-cumulant Edgeworth coefficient.  Differentiating the same finite
Taylor/Lindeberg expansion in \(\rho\) gives
\eqref{eq:weighted-bivariate-edgeworth-CJ-bound}; the assumed moment and
smoothness bounds justify differentiating under the expectation and control the
remainders uniformly.
\end{proof}

For \(0\le \ell\le K\), let \(Q_{\ell,\theta}(\rho)\) denote the covariance of a
single pair of corresponding features after \(\ell\) layers when the two inputs
are \(\rho\)-correlated:
\begin{equation}
    Q_{\ell,\theta}(\rho)
    := \E\Big[Z^{(\ell)}_{i_{\ell+1},\ldots,i_K}(X)
               Z^{(\ell)}_{i_{\ell+1},\ldots,i_K}(X^\rho)\Big].
    \label{eq:finite-d-layer-stability}
\end{equation}
By translation invariance over patches, the right-hand side does not depend on
the coordinate \((i_{\ell+1},\ldots,i_K)\).  We have
\begin{equation}
    Q_{0,\theta}(\rho)=\rho,
    \qquad
    Q_{K,\theta}(\rho)=\mathsf S_\theta(\rho).
    \label{eq:finite-d-stability-boundary}
\end{equation}

\begin{proposition}[Finite-dimensional stability recursion]
\label{prop:finite-d-stability-recursion}
Fix \(K\ge1\), \(J\ge0\), and \(\rho_0\in(0,1)\).  Assume
Assumption~\ref{ass:edgeworth-regularity} at a smoothness order
\(L=L(K,J)\) sufficiently large.  Then, for all \(1\le\ell\le K\),
\begin{equation}
    Q_{\ell,\theta}(\rho)
    = \mathsf{s}\big(Q_{\ell-1,\theta}(\rho)\big)
      + \varepsilon_{\ell,\theta}(\rho),
    \qquad |\rho|\le \rho_0,
    \label{eq:finite-d-stability-recursion}
\end{equation}
with
\begin{equation}
    \|\varepsilon_{\ell,\theta}\|_{C^J([-\rho_0,\rho_0])}
    \le C_{\phi,K,J,\rho_0,B_1(\theta)}
    \sum_{r=1}^{\ell}\eta(\theta_r).
    \label{eq:finite-d-stability-recursion-error}
\end{equation}
Consequently,
\begin{equation}
    \left\|\mathsf S_\theta-\mathsf{s}_K\right\|_{C^J([-\rho_0,\rho_0])}
    \le C_{\phi,K,J,\rho_0,B_1(\theta)}
    \sum_{\ell=1}^K\eta(\theta_\ell).
    \label{eq:output-stability-composition-error}
\end{equation}
The same estimate with \(J=0\) holds uniformly on \([-1,1]\).
\end{proposition}

\begin{proof}
Fix \(\ell\ge1\).  Conditional on the lower layers, the corresponding features
entering the \(\ell\)-th contraction form iid pairs
\((U_i,V_i)_{i=1}^d\), where
\((U_i,V_i)\) has the law of
$$
    \left(
    Z^{(\ell-1)}_{i,i_{\ell+1},\ldots,i_K}(X),
    Z^{(\ell-1)}_{i,i_{\ell+1},\ldots,i_K}(X^\rho)
    \right).
$$
Their means may be nonzero at finite \(d\), but by the same induction used
below they satisfy
$$
    |\E U_i|+|\E V_i|\le
    C_{\phi,K,B_1(\theta)}\sum_{r<\ell}\eta(\theta_r).
$$
Their second moments are
$$
    \E U_i^2=\E V_i^2=V_{\ell-1,\theta},
    \qquad
    \E U_iV_i=Q_{\ell-1,\theta}(\rho),
    \qquad
    V_{\ell-1,\theta}:=Q_{\ell-1,\theta}(1).
$$
After centering the pairs, the induced preactivation mean is bounded by
\(C_{\phi,K,B_1(\theta)}\sum_{r<\ell}\eta(\theta_r)\), because
\(|P_{\ell,1}|\le B_1(\theta)\).  This mean perturbation is absorbed into the
same error budget by smoothness of \(\phi\).  The two preactivations at layer
\(\ell\) are
$$
    S_\ell=\sum_{i=1}^d\theta_\ell(i)U_i,
    \qquad
    T_\ell=\sum_{i=1}^d\theta_\ell(i)V_i.
$$
Applying Lemma~\ref{lem:weighted-bivariate-edgeworth-step} with
\(F(x,y)=\phi(x)\phi(y)\), after the harmless rescaling that matches the
variance \(V_{\ell-1,\theta}\), gives
$$
    Q_{\ell,\theta}(\rho)
    = \E\phi(S_\ell)\phi(T_\ell)
    = \Phi
      \big(V_{\ell-1,\theta},Q_{\ell-1,\theta}(\rho)\big)
      +\widetilde\varepsilon_{\ell,\theta}(\rho),
$$
where \(\Phi(v,q)\) denotes the Gaussian expectation
$$
    \Phi(v,q):=\E\phi(G_1)\phi(G_2),
    \qquad
    \E G_1^2=\E G_2^2=v,
    \quad
    \E G_1G_2=q,
$$
and
\(\|\widetilde\varepsilon_{\ell,\theta}\|_{C^J}\le
C_{\phi,K,J,\rho_0,B_1(\theta)}\eta(\theta_\ell)\).  By induction from the same estimate at
\(\rho=1\),
\(\abs{V_{\ell-1,\theta}-1}\le
C_{\phi,K,B_1(\theta)}\sum_{r<\ell}\eta(\theta_r)\).  Since
\(\Phi\) is smooth near \(v=1\) and \(\Phi(1,q)=\mathsf{s}(q)\), the variance
perturbation is absorbed into the same error budget.  This yields
\eqref{eq:finite-d-stability-recursion} and
\eqref{eq:finite-d-stability-recursion-error}.  Moment and derivative bounds for
the lower-layer pair distributions follow recursively from the polynomial growth
bounds on the derivatives of \(\phi\) and the fact that \(K,J\) are fixed.

It remains to compare the recursion to its Gaussian limit.  Let
\(E_\ell:=\|Q_{\ell,\theta}-\mathsf{s}_\ell\|_{C^J([-\rho_0,\rho_0])}\).  Since \(\mathsf{s}\) is
\(C^J\) on \([-1,1]\), the chain rule gives
$$
    E_\ell\le C_{\phi,J,\rho_0,B_1(\theta)}E_{\ell-1}
    + C_{\phi,K,J,\rho_0,B_1(\theta)}\eta(\theta_\ell),
    \qquad E_0=0.
$$
Iterating for the fixed number \(K\) of layers gives
\eqref{eq:output-stability-composition-error}.  The uniform \(J=0\) bound on
\([-1,1]\) follows from the same argument without differentiating in \(\rho\).
\end{proof}
This $C^J$-control on the generating function automatically yields the desired control at the coefficient level:
\begin{theorem}[Sharp fixed-degree energy transfer under incoherence]
\label{thm:sharp-hermite-energy-transfer}
Fix \(K\ge1\), \(J\ge0\), and \(\rho_0\in(0,1)\).  Assume
Assumption~\ref{ass:edgeworth-regularity} at a sufficiently high fixed order
\(L=L(K,J)\).  Then, deterministically for every collection of unit directions,
\begin{equation}
    \max_{0\le r\le J}
    \left|\Lambda_r(\theta)-\bar\Lambda_r^{(K)}\right|
    \le C_{\phi,K,J,\rho_0,B_1(\theta)}
    \sum_{\ell=1}^K\eta(\theta_\ell).
    \label{eq:fixed-degree-energy-transfer-deterministic}
\end{equation}
In particular, under the spherical prior
\(\theta_1,\ldots,\theta_K\stackrel{\rm iid}{\sim}\Unif(\Sph^{d-1})\), for every
fixed \(\delta\in(0,1)\), with probability at least \(1-\delta\),
\begin{equation}
    \max_{0\le r\le J}
    \left|\Lambda_r(\theta)-\bar\Lambda_r^{(K)}\right|
    \le C_{\phi,K,J,\delta}d^{-1}.
    \label{eq:fixed-degree-energy-transfer-random-sphere}
\end{equation}
\end{theorem}

\begin{proof}
By \eqref{eq:chaos-energy-as-derivative},
\(\Lambda_r(\theta)=\mathsf S_\theta^{(r)}(0)/r!\).  By definition,
\(\bar\Lambda_r^{(K)}=\mathsf{s}_K^{(r)}(0)/r!\).  Applying
Proposition~\ref{prop:finite-d-stability-recursion} and extracting derivatives
at zero proves \eqref{eq:fixed-degree-energy-transfer-deterministic}.

Under the spherical prior, Lemma~\ref{lem:appendix-spherical-event} gives, with
probability at least \(1-\delta\),
$$
    \max_{1\le\ell\le K}\left|\sum_{i=1}^d\theta_\ell(i)^3\right|
    \le C_{K,\delta}d^{-1},
    \qquad
    \max_{1\le\ell\le K}\sum_{i=1}^d\theta_\ell(i)^4
    \le C_{K,\delta}d^{-1},
$$
after adjusting constants for the fixed-confidence event.  Substituting into
\eqref{eq:fixed-degree-energy-transfer-deterministic} gives
\eqref{eq:fixed-degree-energy-transfer-random-sphere}.
\end{proof}

We conclude this section with a particularly useful consequence of Theorem \ref{thm:sharp-hermite-energy-transfer} with regards to the depth separation property.  Since all
coefficients of \(\mathsf{s}\) are nonnegative, the limiting \(R\)-th energy is bounded
from below by trajectories in which exactly one layer uses an \(R\)-th Hermite
component and every other branch propagates linearly.

\begin{corollary}[Constant energy in a fixed higher chaos]
\label{cor:fixed-higher-chaos-energy-lower-bound}
Fix an integer \(R\ge2\).  Suppose \(\lambda_R>0\) and
\(\lambda_1=\kappa^2>0\).  Then
\begin{equation}
    \bar\Lambda_R^{(K)}
    \ge
    \lambda_R\sum_{a=0}^{K-1}
    \lambda_1^{K-1-a+Ra}
    =
    \lambda_R\sum_{a=0}^{K-1}
    \kappa^{2(K-1-a+Ra)}.
    \label{eq:fixed-chaos-limiting-energy-lower-bound}
\end{equation}
Consequently, under the spherical prior, for every fixed \(\delta\in(0,1)\),
with probability at least \(1-\delta\),
\begin{equation}
    \Lambda_R(\theta)
    \ge
    \lambda_R\sum_{a=0}^{K-1}
    \kappa^{2(K-1-a+Ra)}
    - C_{\phi,K,R,\delta}d^{-1}.
    \label{eq:fixed-chaos-finite-d-energy-lower-bound}
\end{equation}
In particular, if \(\lambda_R>0\), then \(\Lambda_R(\theta)\ge c_{\phi,K,R}>0\)
for all sufficiently large \(d\), with high probability over the planted
directions.
\end{corollary}

\begin{proof}
Expand \(\mathsf{s}_K=\mathsf{s}\circ\cdots\circ\mathsf{s}\).  Fix \(a\in\{0,\ldots,K-1\}\).  Choose
one occurrence of the monomial \(\lambda_R q^R\) at depth \(a\) from the input
side, and choose the linear monomial \(\lambda_1 q\) at every other node of the
resulting composition tree.  Before the \(R\)-branching event there are
\(K-1-a\) linear edges; after it, the \(R\) descendants each traverse \(a\)
linear edges.  The resulting contribution to the coefficient of \(\rho^R\) is
\(\lambda_R\lambda_1^{K-1-a+Ra}\).  All coefficients in \(\mathsf{s}\) are nonnegative,
so summing these \(K\) contributions gives
\eqref{eq:fixed-chaos-limiting-energy-lower-bound}.  The finite-\(d\) estimate
\eqref{eq:fixed-chaos-finite-d-energy-lower-bound} follows from
Theorem~\ref{thm:sharp-hermite-energy-transfer} with \(J=R\).
\end{proof}

\subsection{Depth separation against shallow single-index networks}
\label{sec:shallow-depth-separation}

We now examine the approximation consequence of the two previous structural
statements. Subsection~\ref{subsec:higher-chaos-staircase-cp-bound} showed that,
inside a fixed higher Wiener chaos, the balanced flattening has
a staircase spectrum, which supplies a lower bound for CP approximation.
Subsection~\ref{sec:hermite-energy-transfer} showed that
the amount of energy carried by each chaos order is, up to $O(d^{-1})$ errors, determined by the iterated Hermite-energy generating function
\(\mathsf{s}^{\circ K}\).  Combining the two gives a separation from shallow networks
whose neurons are single-index functions of the vectorized input.

Let
$$
    V_d := (\R^d)^{\otimes K}, \qquad D:=\dim(V_d)=d^K,
$$
and view the input tensor as a vector in \(V_d\).  For \(R\ge 0\), let
\(\Pi_R\) denote the orthogonal projection in \(L^2(\gamma_D)\) onto the
\(R\)-th Wiener chaos over \(V_d\).  We use the normalization
$$
    \Pi_R f_\theta(X)
    = \langle G_R(\theta), H_R(X)\rangle,
    \qquad
    \Lambda_R(\theta):=\|G_R(\theta)\|_F^2,
$$
so that Parseval gives
$$
    \|f_\theta\|_{L^2(\gamma_D)}^2
    = \sum_{R\ge0}\Lambda_R(\theta).
$$
Here \(H_R(X)\) denotes the orthonormal degree-\(R\) Hermite tensor.  This
normalization differs from the unnormalized Stein tensor
\(\E[\nabla_X^R f_\theta(X)]\) only by constants depending on \(R\), which are
irrelevant throughout this fixed-order discussion.

Fix a measurable activation \(\varrho:\R\to\R\).  For a width parameter \(M\),
write \(\mathcal N_M(\varrho)\) for the class of shallow single-index networks
\begin{equation}
    V(X)=\sum_{m=1}^M \alpha_m\,\varrho\big(\langle X,W_m\rangle-b_m\big),
    \qquad
    \alpha_m,b_m\in\R,\quad W_m\in V_d,
    \label{eq:shallow-vectorized-network-class}
\end{equation}
with the convention that only choices for which \(V\in L^2(\gamma_D)\) are
admissible.  No norm constraint is imposed on the weights \(W_m, \alpha_m, b_m\).  The only
property of the class used below is that each neuron is a one-dimensional
function of a Gaussian projection of \(X\).

\begin{lemma}[One-dimensional chaos of a shallow neuron]
\label{lem:one-dimensional-chaos-shallow-neuron}
Let
$$
    g_{W,b}(X):=\varrho(\langle X,W\rangle-b)\in L^2(\gamma_D).
$$
Then, for every \(R\ge1\), the degree-\(R\) Hermite coefficient tensor of
\(g_{W,b}\) has symmetric CP rank at most one.  Consequently, for every
\(V\in\mathcal N_M(\varrho)\),
\begin{equation}
    \operatorname{rank}_{\rm CP}\big(G_R(V)\big)\le M,
    \qquad R\ge1,
    \label{eq:shallow-chaos-cp-rank-M}
\end{equation}
where \(G_R(V)\) is the degree-\(R\) Hermite coefficient tensor of \(V\).
\end{lemma}

\begin{proof}
If \(W=0\), then \(g_{W,b}\) is constant and its chaoses of order \(R\ge1\)
vanish.  Otherwise write \(W=s u\), where \(s=\|W\|_2>0\) and
\(u\in V_d\) is a unit vector.  Since \(\langle X,u\rangle\sim N(0,1)\), the
univariate function \(z\mapsto \varrho(sz-b)\) has an \(L^2(\gamma_1)\) Hermite
expansion
$$
    \varrho(sz-b)=\sum_{R\ge0}\beta_R(s,b)h_R(z).
$$
Therefore
$$
    g_{W,b}(X)=\sum_{R\ge0}\beta_R(s,b)h_R(\langle X,u\rangle).
$$
By the defining property of the multivariate Hermite tensors,
$$
    h_R(\langle X,u\rangle)
    =\langle H_R(X),u^{\otimes R}\rangle.
$$
Thus the degree-\(R\) coefficient tensor of the neuron is
\(\beta_R(s,b)u^{\otimes R}\), which has symmetric CP rank at most one.
Summing over \(m=1,\ldots,M\) proves \eqref{eq:shallow-chaos-cp-rank-M}.
\end{proof}

For the CP lower bound we use the top plateau of the balanced-flattening
staircase.  For \(R\ge2\), set
\begin{equation}
    \chi_R:= (K-1)\big\lfloor R/2\big\rfloor.
    \label{eq:top-balanced-crossing-complexity}
\end{equation}
For even \(R\), this is \(R(K-1)/2\), matching the top balanced plateau; for
odd \(R\), it is the corresponding near-balanced value.  Let
\(\operatorname{Mat}_{\rm bal,R}\) denote a balanced, or near-balanced when
\(R\) is odd, flattening of a tensor in \(V_d^{\otimes R}\).

\begin{assumption}[Quantitative top-plateau non-cancellation]
\label{ass:top-plateau-noncancellation}
Fix a finite set of chaos orders \(\mathcal R\subset\{2,3,\ldots\}\).  We say
that the MSIM target satisfies top-plateau non-cancellation on \(\mathcal R\) if
there exist constants \(a_R,b_R>0\), independent of \(d\), such that for every
\(R\in\mathcal R\), all sufficiently large \(d\), and the planted directions
under consideration,
\begin{equation}
    \sigma_{\lfloor a_R d^{\chi_R}\rfloor}
    \left(\operatorname{Mat}_{\rm bal,R}(G_R(\theta))\right)
    \ge b_R\sqrt{\Lambda_R(\theta)}\,d^{-\chi_R/2}.
    \label{eq:top-plateau-noncancellation}
\end{equation}
\end{assumption}

This is the normalized version of the highest staircase plateau non-degeneracy Assumption \ref{ass:higher-chaos-noncancellation} used in
Theorem~\ref{thm:balanced-flattening-staircase}.  The role of
Assumption~\ref{ass:top-plateau-noncancellation} is exactly analogous to the
role of the first-Hermite non-degeneracy condition \(\kappa\ne0\) for the first
chaos: it rules out finite cancellations of the coefficients generated by the
\(R\)-spine coalescence expansion.  Since \(\mathcal R\) is fixed, this is a
finite collection of nonvanishing algebraic conditions on the Gaussian derivative
moments of \(\phi\).

\begin{lemma}[One-chaos obstruction]
\label{lem:one-chaos-obstruction-shallow-rank}
Assume \eqref{eq:top-plateau-noncancellation} holds for a fixed order \(R\ge2\).
Then, for all sufficiently large \(d\), every tensor \(A\in V_d^{\otimes R}\)
with \(\operatorname{rank}_{\rm CP}(A)\le M\) and
$$
    M\le \frac{a_R}{2}d^{\chi_R}
$$
satisfies
\begin{equation}
    \|G_R(\theta)-A\|_F^2
    \ge \eta_R\Lambda_R(\theta),
    \qquad
    \eta_R:=\frac{a_Rb_R^2}{4}.
    \label{eq:one-chaos-obstruction}
\end{equation}
\end{lemma}

\begin{proof}
Let \(N_R=\lfloor a_Rd^{\chi_R}\rfloor\) and
\(s_R=b_R\sqrt{\Lambda_R(\theta)}d^{-\chi_R/2}\).  By
\eqref{eq:top-plateau-noncancellation}, the balanced flattening of
\(G_R(\theta)\) has at least \(N_R\) singular values larger than \(s_R\).  Every
CP-rank-\(M\) tensor has balanced-flattening rank at most \(M\).  Hence, by
Eckart--Young and Lemma~\ref{lem:flattening-lower-bound-cp},
$$
    \|G_R(\theta)-A\|_F^2
    \ge \sum_{j>M}\sigma_j\!
    \left(\operatorname{Mat}_{\rm bal,R}(G_R(\theta))\right)^2
    \ge (N_R-M)s_R^2.
$$
For all sufficiently large \(d\), \(N_R-M\ge (a_R/4)d^{\chi_R}\), and therefore
$$
    (N_R-M)s_R^2
    \ge \frac{a_R}{4}d^{\chi_R}\cdot
    b_R^2\Lambda_R(\theta)d^{-\chi_R}
    =\frac{a_Rb_R^2}{4}\Lambda_R(\theta).
$$
\end{proof}

\begin{theorem}[Shallow networks miss the high-chaos tail]
\label{thm:shallow-misses-high-chaos-tail}
Fix integers \(2\le J\le L\), and suppose that
Assumption~\ref{ass:top-plateau-noncancellation} holds on the block
\(\mathcal R=\{J,J+1,\ldots,L\}\).  Define
$$
    a_{J:L}:=\min_{J\le R\le L}a_R,
    \qquad
    \eta_{J:L}:=\min_{J\le R\le L}\frac{a_Rb_R^2}{4}.
$$
Then, for all sufficiently large \(d\), if
\begin{equation}
    M\le \frac{a_{J:L}}{2}d^{\chi_J},
    \label{eq:width-below-block-threshold}
\end{equation}
then
\begin{equation}
    \inf_{V\in\mathcal N_M(\varrho)}
    \|f_\theta-V\|_{L^2(\gamma_D)}^2
    \ge
    \eta_{J:L}\sum_{R=J}^L\Lambda_R(\theta).
    \label{eq:shallow-block-tail-inf-bound}
\end{equation}
\end{theorem}

\begin{proof}
By orthogonality of Wiener chaoses,
$$
    \|f_\theta-V\|_{L^2(\gamma_D)}^2
    =\sum_{R\ge0}\|G_R(\theta)-G_R(V)\|_F^2
    \ge \sum_{R=J}^L\|G_R(\theta)-G_R(V)\|_F^2.
$$
By Lemma~\ref{lem:one-dimensional-chaos-shallow-neuron},
\(\operatorname{rank}_{\rm CP}(G_R(V))\le M\) for each \(R\ge1\).  Since
\(\chi_R\) is nondecreasing in \(R\), condition
\eqref{eq:width-below-block-threshold} implies
\(M\le (a_R/2)d^{\chi_R}\) for every \(R\in[J,L]\).  Applying
Lemma~\ref{lem:one-chaos-obstruction-shallow-rank} to each \(R\in[J,L]\) gives
$$
    \|G_R(\theta)-G_R(V)\|_F^2
    \ge \eta_R\Lambda_R(\theta)
    \ge \eta_{J:L}\Lambda_R(\theta),
$$
and summing over \(R=J,\ldots,L\) proves the claim.
\end{proof}

Theorem~\ref{thm:shallow-misses-high-chaos-tail} is a deterministic statement:
once the planted directions obey the staircase lower bounds and have nontrivial
energy in a block of chaos orders, any shallow network with fewer than the
corresponding plateau dimension must miss a constant fraction of that block.  We
now insert the dimension-independent energy profile from
Theorem~\ref{thm:sharp-hermite-energy-transfer}.

Let
\begin{equation}
    \bar T_K(J)
    :=\sum_{R\ge J}\bar\Lambda_R^{(K)},
    \qquad
    \bar\Lambda_R^{(K)}=[\rho^R]\mathsf{s}^{\circ K}(\rho),
    \label{eq:limiting-chaos-tail-for-separation}
\end{equation}
where \(\mathsf{s}(q)=\sum_{r\ge0}\lambda_rq^r\) is the univariate Hermite-energy
generating function of \(\phi\).  For a finite block we similarly write
$$
    \bar T_K(J,L):=\sum_{R=J}^L\bar\Lambda_R^{(K)}.
$$

\begin{theorem}[Tail form of the shallow lower bound]
\label{thm:tail-depth-separation}
Fix \(J\ge2\).  Suppose that, for every finite \(L\ge J\), the top-plateau
non-cancellation assumption ~\ref{ass:top-plateau-noncancellation} holds on \(\{J,\ldots,L\}\).  Then for every
\(\delta\in(0,1)\) and every \(\xi>0\), there exists a finite \(L=L(J,\xi)\) and
constants \(c_{J,\xi},\eta_{J,\xi}>0\), independent of \(d\), such that with
probability at least \(1-\delta\) over the planted directions, for all sufficiently
large \(d\),
\begin{equation}
    M\le c_{J,\xi}d^{(K-1) \lfloor J/2\rfloor }
    \quad\Longrightarrow\quad
    \inf_{V\in\mathcal N_M(\varrho)}
    \|f_\theta-V\|_{L^2(\gamma_D)}^2
    \ge
    \eta_{J,\xi}\big(\bar T_K(J)-\xi\big)-o_d(1).
    \label{eq:tail-depth-separation}
\end{equation}
In particular, up to constants depending on the fixed tail window, a width
\(M\) shallow network can only begin to capture chaos orders \(R\) for which
\begin{equation}
    d^{(K-1)\lfloor R/2\rfloor }\lesssim M~,
    \label{eq:rank-threshold-chaos-order}
\end{equation}
or equivalently, it misses a constant fraction of the target energy in all orders
above
\begin{equation}
    R_M \simeq \frac{2}{K-1}\log_d M.
    \label{eq:critical-chaos-order-logM}
\end{equation}
\end{theorem}

\begin{proof}
Choose \(L\ge J\) so that
\(\bar T_K(J,L)\ge \bar T_K(J)-\xi\).  This is possible because
\(\sum_R\bar\Lambda_R^{(K)}<\infty\).
Under the spherical prior on the planted directions,
for every fixed \(\delta\in(0,1)\), we claim that with probability at least \(1-\delta\), for
all sufficiently large \(d\), every width
$$
    M\le \frac{a_{J:L}}{2}d^{\chi_J}
$$
satisfies
\begin{equation}
    \inf_{V\in\mathcal N_M(\varrho)}
    \|f_\theta-V\|_{L^2(\gamma_D)}^2
    \ge
    \eta_{J:L}\bar T_K(J,L)-C_{\phi,K,J,L,\delta}d^{-1}.
    \label{eq:fixed-block-depth-separation}
\end{equation}
Indeed, by combining Theorem~\ref{thm:shallow-misses-high-chaos-tail} with the fixed-degree
energy transfer estimate of Theorem~\ref{thm:sharp-hermite-energy-transfer},
we obtain
$$
    \max_{0\le R\le L}|\Lambda_R(\theta)-\bar\Lambda_R^{(K)}|
    \le C_{\phi,K,L,\delta}d^{-1}
$$
with probability at least \(1-\delta\).  Summing over the fixed block
\(R=J,\ldots,L\) gives \eqref{eq:fixed-block-depth-separation}.
Now, absorbing the fixed constants into
\(c_{J,\xi}:=a_{J:L}/2\) and \(\eta_{J,\xi}:=\eta_{J:L}\), \eqref{eq:tail-depth-separation} follows.
The final
interpretation follows by solving the threshold relation
\(M\asymp d^{\chi_R}\) for \(R\).
\end{proof}

Finally, we can make the lower bound more explicit by imposing
a power-law tail condition in the Hermite spectrum.
\begin{assumption}[Slow chaos-energy tail]
\label{ass:slow-limiting-chaos-tail}
There exist constants \(A>0\), \(p>0\), and \(J_0\ge2\), independent of \(d\),
such that
\begin{equation}
    \bar T_K(J)\ge A J^{-p},
    \qquad J\ge J_0.
    \label{eq:slow-limiting-chaos-tail}
\end{equation}
\end{assumption}

\begin{corollary}[Shallow depth separation]
\label{cor:shallow-depth-separation-epsilon}
Assume that the top-plateau non-cancellation constants are available on each
fixed finite range of chaos orders.  For every sufficiently small fixed
\(\epsilon>0\), choose a finite order \(J_\epsilon\ge2\) and a finite block
endpoint \(L_\epsilon\ge J_\epsilon\) such that
\begin{equation}
    \eta_{J_\epsilon:L_\epsilon}\,
    \bar T_K(J_\epsilon,L_\epsilon) \ge 2\epsilon.
    \label{eq:epsilon-block-choice}
\end{equation}
If the slow-tail condition \eqref{eq:slow-limiting-chaos-tail}
holds and the top-plateau non-cancellation is quantitatively uniform over the
orders $\{ J_\epsilon, L_\epsilon\}$ needed to capture an \(\epsilon\)-tail, then there exists a constant
\(C_{\rm sep}>0\), independent of \(\epsilon\) and \(d\), such that
\begin{equation}
    M\le d^{C_{\rm sep}\epsilon^{-1/p}}
    \quad\Longrightarrow\quad
    \inf_{V\in\mathcal N_M(\varrho)}
    \|f_\theta-V\|_{L^2(\gamma_D)}^2
    \ge \epsilon.
    \label{eq:epsilon-depth-separation-general-p}
\end{equation}
\end{corollary}

\begin{proof}
Observe first that we can find a constant \(c_\epsilon>0\), independent of \(d\), such that, with
high probability over the spherical planted directions and for all sufficiently
large \(d\),
\begin{equation}
    M\le c_\epsilon d^{\chi_{J_\epsilon}}
    \quad\Longrightarrow\quad
    \inf_{V\in\mathcal N_M(\varrho)}
    \|f_\theta-V\|_{L^2(\gamma_D)}^2
    \ge \epsilon.
    \label{eq:epsilon-depth-separation-block}
\end{equation}
Indeed, we apply \eqref{eq:fixed-block-depth-separation} to the block
\(J_\epsilon,\ldots,L_\epsilon\).  Namely, if
\(M\le (a_{J_\epsilon:L_\epsilon}/2)d^{\chi_{J_\epsilon}}\), then
$$
    \inf_{V\in\mathcal N_M(\varrho)}
    \|f_\theta-V\|_{L^2(\gamma_D)}^2
    \ge
    \eta_{J_\epsilon:L_\epsilon}\bar T_K(J_\epsilon,L_\epsilon)-o_d(1)
    \ge 2\epsilon-o_d(1),
$$
which is at least \(\epsilon\) for all sufficiently large \(d\).  Thus
\(c_\epsilon=a_{J_\epsilon:L_\epsilon}/2\) is admissible.

Assume now the slow-tail condition and quantitative uniform non-cancellation.
Then there is a constant \(\eta_0>0\) such that the block constants can be chosen
with \(\eta_{J:L}\ge\eta_0\) throughout the relevant ranges.  Pick
\(J_\epsilon\asymp \epsilon^{-1/p}\) so that
\(\eta_0\bar T_K(J_\epsilon)\ge4\epsilon\), possible by
\eqref{eq:slow-limiting-chaos-tail}, and then choose
\(L_\epsilon\ge J_\epsilon\) so that
\(\bar T_K(J_\epsilon,L_\epsilon)\ge\bar T_K(J_\epsilon)/2\).  The first part of
the corollary gives the lower bound for widths below a constant multiple of
\(d^{\chi_{J_\epsilon}}\).  Since
\(\chi_{J_\epsilon}=(K-1)\lfloor J_\epsilon/2\rfloor\asymp_K
\epsilon^{-1/p}\), one may choose \(C_{\rm sep}>0\), independent of
\(\epsilon\), such that \(d^{C_{\rm sep}\epsilon^{-1/p}}\) lies below this
threshold for all sufficiently large \(d\).  This proves
\eqref{eq:epsilon-depth-separation-general-p}.
\end{proof}

We emphasize that this approximation lower bound holds for fixed error $\epsilon > 0$ and in the high-dimensional regime $d \to \infty$. In this respect, it does not cover the regime where the approximation error is allowed to vanish $\epsilon= \epsilon(d) \to 0$ arbitrarily with dimension.

\section{Warmup: SGD over linear correlation}
\label{sec:sgd_linear}

We now turn our attention to the SGD recovery of the planted parameters. As discussed in Section \ref{sec:mainresults}, our analysis will center in the correlation loss $\mathcal{L}(\tilde{\theta}) = \langle f_\theta, f_{\tilde{\theta}}\rangle$. Our approach will first consider a simpler baseline $\widetilde{\mathcal{L}}(\tilde{\theta})$, and then proceed perturbatively.
In light of the structural result of Section \ref{sec:spectral}, the natural baseline is to consider the linear correlation loss
\begin{equation}
\tilde{\mathcal{L}}(\tilde{\theta}_1, \ldots, \tilde{\theta}_K) := \E\left[f_{\theta}(X) \left\langle X, \bigotimes_{k=1}^K \tilde{\theta}_k \right\rangle\right] = \left\langle \mathcal{G}_\theta, \bigotimes_{k=1}^K \tilde{\theta}_k \right\rangle~.
\label{eq:objective}
\end{equation}

The goal of this section is thus to first establish recovery guarantees for the SGD linear correlation dynamics, which will be then leveraged in Section \ref{seq:nonlinear_correl} to establish analogue guarantees for the non-linear SGD dynamics.
To establish the spectral recovery guarantees, it was sufficient to consider the first order expansion $\mathcal{G}_\theta = \lambda \bigotimes \theta_k + W$, with $\| W\|_F \simeq d^{-1/2}$.
Consider now the spherical gradient-flow dynamics associated with \eqref{eq:objective}, given by
\begin{equation}
  \dot{\tilde{\theta}}_j
  =
  \nabla_{{\tilde{\theta}}_j}^{S^{d-1}}\tilde{\mathcal{L}}
  =
  (I-{\tilde{\theta}}_j{\tilde{\theta}}_j^\top)
  \left(
  \mathcal G_\theta\times_{\ell\ne j}{\tilde{\theta}}_\ell
  \right),
  \qquad j=1,\ldots,K.
  \label{eq:population-flow}
\end{equation}
Here $\mathcal G_\theta\times_{\ell\ne j}{\tilde{\theta}}_\ell$ denotes contraction of the tensor against ${\tilde{\theta}}_\ell$ in all modes except the $j$-th mode.
Under the first-order expansion, the corresponding overlaps $m_k := \theta_k \cdot \theta_k^*$ verify
$\dot{m}_k = \lambda (1-m_k^2) \prod_{j \neq k} m_j + \mathsf{P}_{\theta_k^\perp}W[\bigotimes_{j \neq k} \theta_j] \cdot \theta_k^*~,$
where $W[\bigotimes_{j \neq k} \theta_j] \in \R^d$ is the mode-$k$ unfolding of $W$ times $\text{vec}(\bigotimes_{j \neq k} \theta_j)$.
Therefore, at the `mediocrity' of initialization, the first `signal' term is of order $d^{-(K-1)/2}$, while the perturbative term is $O(d^{-1/2})$, which dominates the signal as soon as $K>2$.

The solution, as expected, is to exploit the full power of Theorem \ref{thm:edgeworth-hierarchy-body}, which further decomposes the perturbation tensor $W$ into a finite-rank hierarchy. The goal of this section is to establish a recovery guarantee for online SGD at the scale $n = O( d^{K-1})$, again tracing a parallel with the Tensor PCA model \cite{arous2024high}.
We assume throughout that $\phi$ is sufficiently smooth with controlled derivative growth for the high-order Edgeworth expansion through order $K$. A convenient sufficient condition is that, for some sufficiently large integer $L=L(K)$, $\phi\in C^L(\R)$, $\phi$ is globally Lipschitz, and all derivatives of $\phi$ up to order $L$ have at most polynomial growth.

\subsection{Population Gradient Flow}
\label{sec:popgflow}

We will follow the blueprint from \cite{arous2020online} to analyze SGD dynamics. After identifying the relevant dimension-independent summary statistics, SGD can be seen as diffusive dynamics driven by a deterministic drift, given by the population gradient flow. Our first step is thus to analyze this deterministic gradient flow.

For $p\ge1$, write
$$
  u_{j,p}:=\theta_j^{\odot p}
  :=\bigl(\theta_j(1)^p,\ldots,\theta_j(d)^p\bigr),
  \qquad j=1,\ldots,K.
$$
In our setting, as seen in similar works \cite{10.1214/19-AOP1415, arous2024high, bietti2023learning, abbe2023sgd},
we track the planted overlaps $m_j(t)$
\begin{equation}
  m_j(t):=\ip{{\tilde{\theta}}_j(t)}{\theta_j}
  =\ip{{\tilde{\theta}}_j(t)}{u_{j,1}},
  \qquad j=1,\ldots,K,
  \label{eq:planted-overlaps}
\end{equation}
but we include also the auxiliary high-order Edgeworth overlaps
\begin{equation}
  h_{j,p}(t):=\ip{{\tilde{\theta}}_j(t)}{u_{j,p}},
  \qquad j=1,\ldots,K,
  \qquad 2\le p\le K.
  \label{eq:aux-overlaps}
\end{equation}
By flipping one planted direction if necessary, or equivalently by studying signed overlaps, we shall assume throughout the main theorem that $\lambda>0$
\footnote{for arbitrary $\lambda\ne0$, choose signs $s_1,\ldots,s_K\in\{\pm1\}$ with
$\prod_{j=1}^K s_j=\operatorname{sgn}(\lambda)$ and replace $m_j$ by
$s_j\ip{{\tilde{\theta}}_j}{\theta_j}$.}.

Under the uniform prior, observe that, with probability greater than $1-\delta$ over the planted directions, for all $1\le j\le K$ and all $1\le p,b\le K$,
\begin{equation}
  \abs{\ip{u_{j,p}}{u_{j,b}}}
  =
  \left|\sum_{a=1}^d \theta_j(a)^{p+b}\right|
  \le
  C_{K,\delta}
  \begin{cases}
  1, & p=b=1,\\
  d^{-1}, & p\ge2\text{ or }b\ge2.
  \end{cases}
  \label{eq:powersum-event}
\end{equation}
Moreover,
\begin{equation}
  \norm{u_{j,p}}_2\le C_{K,\delta}d^{-(p-1)/2},
  \qquad 1\le p\le K.
  \label{eq:power-norms}
\end{equation}
The next theorem establishes strong recovery for the gradient flow dynamics, confirming that the fine-grained Edgeworth expansion from Theorem \ref{thm:edgeworth-hierarchy-body} overcomes the norm obstruction of the first order expansion $\G_\theta = \lambda \theta_1 \otimes \dots \otimes \theta_K + W$.
\begin{theorem}[Population gradient flow recovery]
\label{thm:general-k-pop-flow}
Assume the activation satisfies the smoothness hypotheses of Theorem~\ref{thm:edgeworth-hierarchy-body}.
Work on the event $\mathcal E_\theta(\delta)$ supplied by Theorem~\ref{thm:edgeworth-hierarchy-body}, and assume the signs have been chosen so that $\lambda>0$.

There exist constants $a,B,\varepsilon,C,d_0>0$ depending only on $\phi,K,\delta$, such that for all $d\ge d_0$ the following holds. Assume the initialization lies in the favorable basin
\begin{equation}
  m_j(0)=\ip{{\tilde{\theta}}_j(0)}{\theta_j}
  \ge
  \frac{a}{\sqrt d},
  \qquad j=1,\ldots,K,
  \label{eq:favorable-m}
\end{equation}
and
\begin{equation}
  \abs{h_{j,p}(0)}
  \le
  \frac{B}{4d},
  \qquad j=1,\ldots,K,
  \qquad 2\le p\le K.
  \label{eq:favorable-h}
\end{equation}
Let
\begin{equation}
  \tau_\varepsilon
  :=
  \inf\left\{t\ge0:\min_{1\le j\le K}m_j(t)\ge\varepsilon\right\}.
  \label{eq:hitting-time}
\end{equation}
and
\begin{equation}
  T_{\rm str}:=C\log d,
  \label{eq:population-strong-extra-time}
\end{equation}
Then we have:
\begin{itemize}
    \item Weak recovery:
$  \tau_\varepsilon  \le  C d^{K/2-1}$ when $K\geq 3$, and $\tau_\varepsilon = O( \log d)$ when $K=2$.
\item Strong recovery: $ \min_{j\le K}m_j(\tau_\varepsilon+T_{\rm str})
  \ge
  1-\frac{C}{d}.$
\end{itemize}
\end{theorem}
\begin{proof}

Recall from Theorem \ref{thm:edgeworth-hierarchy-body} the following structure in the population gradient tensor.
The population gradient tensor admits the expansion
\begin{equation}
  \mathcal G_\theta
  =
  \sum_{q=0}^{K-1}\mathcal G_\theta^{[q]}
  +
  \mathcal R_K,
  \qquad
  \norm{\mathcal R_K}_F
  \le
  C_{\phi,K,\delta}d^{-K/2}.
  \label{eq:edgeworth-hierarchy-gflow}
\end{equation}
The leading level is
\begin{equation}
  \mathcal G_\theta^{[0]}
  =
  \lambda\,u_{1,1}\ot\cdots\ot u_{K,1}.
  \label{eq:leading-level}
\end{equation}
For $q=1,\ldots,K-1$, the level-$q$ term has a finite-rank representation
\begin{equation}
  \mathcal G_\theta^{[q]}
  =
  \sum_{\tau\in\mathcal T_q}
  a_\tau(\theta)
  \bigotimes_{r=1}^K u_{r,b_{\tau,r}},
  \label{eq:level-q-representation}
\end{equation}
where $b_{\tau,r}\in\{1,\ldots,K\}$ and the atom has formal order $q$, meaning that
\begin{equation}
  \sigma_\tau+\sum_{r=1}^K(b_{\tau,r}-1)=q
  \label{eq:formal-order}
\end{equation}
for some scalar Edgeworth order $\sigma_\tau\ge0$, where its coefficient satisfies
\begin{equation}
  \abs{a_\tau(\theta)}
  \le
  C_{\phi,K,\delta}d^{-\sigma_\tau/2}.
  \label{eq:coefficient-bound}
\end{equation}
After collecting equal tensor-power profiles, we have
\begin{equation}
  \abs{\mathcal T_q}
  \le
  R_{K,q}^{\sharp}
  :=
  \sum_{j=0}^{\lfloor q/2\rfloor}
  (j+1)\binom{K+q-2j-2}{q-2j} = O_{K,q}(1).
  \label{eq:sharp-rank}
\end{equation}

\noindent { \textbf{Proof of Weak recovery}}
We prove the result using a standard bootstrap argument.

For a rank-one tensor atom
$ T=a\,v_1\ot\cdots\ot v_K,$ its contribution to the Euclidean gradient in mode $j$ is
$$
  a\left(\prod_{\ell\ne j}\ip{{\tilde{\theta}}_\ell}{v_\ell}\right)v_j.
$$
Therefore its contribution to the spherical derivative of
$h_{j,p}(t)=\ip{{\tilde{\theta}}_j(t)}{u_{j,p}}$
is
\begin{equation}
  a\left(\prod_{\ell\ne j}\ip{{\tilde{\theta}}_\ell}{v_\ell}\right)
  \left[
  \ip{u_{j,p}}{v_j}
  -
  h_{j,p}(t)\ip{{\tilde{\theta}}_j(t)}{v_j}
  \right].
  \label{eq:rank-one-derivative}
\end{equation}
This identity follows directly from
$\dot{\tilde{\theta}}_j=(I-{\tilde{\theta}}_j{\tilde{\theta}}_j^\top)\nabla_{{\tilde{\theta}}_j}L$.

Now define
\begin{equation}
  \tau
  :=
  \inf\left\{
  t\ge0:
  \min_j m_j(t)\ge\varepsilon
  \text{ or }
  \min_j m_j(t)\le \frac{a}{2\sqrt d}
  \text{ or }
  \max_{j,2\le p\le K}\abs{h_{j,p}(t)}\ge \frac{B}{d}
  \right\}.
  \label{eq:bootstrap-time}
\end{equation}
On $[0,\tau]$ we have
\begin{equation}
  m_j(t)\ge \frac{a}{2\sqrt d},
  \qquad
  \abs{h_{j,p}(t)}\le \frac{B}{d},
  \qquad 2\le p\le K.
  \label{eq:bootstrap-bounds}
\end{equation}
We will show that the second and third alternatives in the definition of $\tau$ cannot occur before weak recovery.

The leading atom
$\lambda u_{1,1}\ot\cdots\ot u_{K,1}$
gives exactly
$ \lambda\left(\prod_{\ell\ne j}m_\ell\right)(1-m_j^2)$ to $\dot m_j$.
Now consider a non-leading atom from \eqref{eq:level-q-representation},
$  a_\tau(\theta)\bigotimes_{r=1}^K u_{r,b_{\tau,r}}$.
For $\ell\ne j$, if $b_{\tau,\ell}=1$, then
$ \ip{{\tilde{\theta}}_\ell}{u_{\ell,b_{\tau,\ell}}}=m_\ell$.
If $b_{\tau,\ell}\ge2$, the bootstrap bound gives
\begin{equation}
  \abs{\ip{{\tilde{\theta}}_\ell}{u_{\ell,b_{\tau,\ell}}}}
  \le
  \frac{B}{d}
  \le
  \frac{2B}{a\sqrt d}m_\ell.
  \label{eq:aux-contraction-replace}
\end{equation}
In the $j$-th mode, the bracket in \eqref{eq:rank-one-derivative} with $p=1$ is
$$
  \ip{u_{j,1}}{u_{j,b_{\tau,j}}}
  -
  m_j\ip{{\tilde{\theta}}_j}{u_{j,b_{\tau,j}}}.
$$
If $b_{\tau,j}=1$, this is $1-m_j^2$, hence bounded by one. If $b_{\tau,j}\ge2$, then by \eqref{eq:powersum-event} and \eqref{eq:bootstrap-bounds},
\begin{equation}
  \left|
  \ip{u_{j,1}}{u_{j,b_{\tau,j}}}
  -
  m_j\ip{{\tilde{\theta}}_j}{u_{j,b_{\tau,j}}}
  \right|
  \le
  C_{K,\delta,B}d^{-1}.
  \label{eq:small-j-bracket}
\end{equation}
Since the atom is non-leading, its formal order is at least one:
\[
  \sigma_\tau+\sum_{r=1}^K(b_{\tau,r}-1)\ge1.
\]
Combining the coefficient bound \eqref{eq:coefficient-bound}, the auxiliary replacement bound \eqref{eq:aux-contraction-replace}, and the small bracket bound \eqref{eq:small-j-bracket}, its contribution to $\dot m_j$ is at most
\begin{equation}
  C_{\phi,K,\delta,B}d^{-1/2}
  \prod_{\ell\ne j}m_\ell.
  \label{eq:nonleading-m-contribution}
\end{equation}
Indeed, the smallest possible loss relative to the leading drift is one half-power of $d$: either from a scalar Edgeworth coefficient $d^{-\sigma_\tau/2}$ with $\sigma_\tau\ge1$, or from replacing one planted contraction by an auxiliary contraction, or from a small $j$-mode bracket.

The residual contributes at most
$$
  \norm{\mathcal R_K}_F
  \le
  C_{\phi,K,\delta}d^{-K/2}.
$$
On the bootstrap region,
$$
  \prod_{\ell\ne j}m_\ell
  \ge
  \left(\frac{a}{2\sqrt d}\right)^{K-1}
  =
  C_a d^{-(K-1)/2}.
$$
Thus
$$
  d^{-K/2}
  \le
  C_a d^{-1/2}\prod_{\ell\ne j}m_\ell.
$$
Summing over the $O_K(1)$ atoms in the hierarchy proves that, for every $j$, the planted-overlap dynamics satisfy, for $t \in [0,\tau]$,
\begin{equation}
  \dot m_j(t)
  =
  \lambda
  \left(\prod_{\ell\ne j}m_\ell(t)\right)
  (1-m_j(t)^2)
  +
  e_j(t),
  \label{eq:overlap-dynamics}
\end{equation}
where, uniformly for $0\le t\le\tau_\varepsilon$,
\begin{equation}
  \abs{e_j(t)}
  \le
  C_{\phi,K,\delta}d^{-1/2}
  \prod_{\ell\ne j}m_\ell(t).
  \label{eq:overlap-error}
\end{equation}

We now choose $\varepsilon\le1/2$. Before weak recovery, we have

$$  1-m_j^2\ge 1-\varepsilon^2\ge \frac34 ~.$$
Taking $d$ sufficiently large so that
$$
  C_{\phi,K,\delta}d^{-1/2}\le \frac{\lambda}{4},
$$
we obtain
$$
  \dot m_j(t)
  \ge
  \frac{\lambda}{2}\prod_{\ell\ne j}m_\ell(t)
$$
on $[0,\tau]$. In particular, no planted overlap can decrease below $a/(2\sqrt d)$, so the second stopping alternative cannot occur.

Let us now establish a uniform control of the auxiliary, high-order overlaps.
Fix $j\in\{1,\ldots,K\}$ and $p\in\{2,\ldots,K\}$. For any atom in the hierarchy, the bracket in \eqref{eq:rank-one-derivative} is
$$
  \ip{u_{j,p}}{u_{j,b_{\tau,j}}}
  -
  h_{j,p}\ip{{\tilde{\theta}}_j}{u_{j,b_{\tau,j}}}.
$$
Since $p\ge2$, the power-sum bound \eqref{eq:powersum-event} gives
$$
  \abs{\ip{u_{j,p}}{u_{j,b_{\tau,j}}}}
  \le
  C_{K,\delta}d^{-1}.
$$
Also, on the bootstrap region,
$$
  \abs{h_{j,p}\ip{{\tilde{\theta}}_j}{u_{j,b_{\tau,j}}}}
  \le
  C_{K,\delta,B}d^{-1}.
$$
Therefore the bracket is $O_{K,\delta,B}(d^{-1})$.

The contractions in the remaining $K-1$ modes are controlled as before: a mode with $b_{\tau,\ell}=1$ gives $m_\ell$, and a mode with $b_{\tau,\ell}\ge2$ gives at most $B/d\le C d^{-1/2}m_\ell$ on the bootstrap region. Hence every atom contributes at most
$$
  C_{\phi,K,\delta,B}d^{-1}
  \prod_{\ell\ne j}m_\ell
$$
to $\dot h_{j,p}$. The residual contributes at most
$$
  \norm{u_{j,p}}_2\norm{\mathcal R_K}_F
  \le
  C_{\phi,K,\delta}d^{-1/2}d^{-K/2}
  =
  C_{\phi,K,\delta}d^{-(K+1)/2}.
$$
On the bootstrap region,
$$
  d^{-1}\prod_{\ell\ne j}m_\ell
  \ge
  c_a d^{-1}d^{-(K-1)/2}
  =
  c_a d^{-(K+1)/2},
$$
so the residual is absorbed. Thus
\begin{equation}
  \abs{\dot h_{j,p}(t)}
  \le
  C_{\phi,K,\delta,B}d^{-1}
  \prod_{\ell\ne j}m_\ell(t),
  \qquad 2\le p\le K.
  \label{eq:aux-derivative}
\end{equation}

We integrate this estimate. By the positive drift bound,
$$
  \dot m_j(t)
  \ge
  \frac{\lambda}{2}\prod_{\ell\ne j}m_\ell(t),
$$
so
\begin{equation}
  \int_0^\tau \prod_{\ell\ne j}m_\ell(t)\,dt
  \le
  \frac{2}{\lambda}\bigl(m_j(\tau)-m_j(0)\bigr)
  \le
  \frac{2}{\lambda}.
  \label{eq:product-integral}
\end{equation}
Using \eqref{eq:favorable-h}, \eqref{eq:aux-derivative}, and \eqref{eq:product-integral}, we obtain
$$
  \abs{h_{j,p}(t)}
  \le
  \frac{B}{4d}
  +
  \frac{C_{\phi,K,\delta,B}}{d},
  \qquad 0\le t\le\tau.
$$
Choosing $B$ sufficiently large, depending only on $\phi,K,\delta$, gives
$$
  \abs{h_{j,p}(t)}\le \frac{B}{2d},
  \qquad 0\le t\le\tau.
$$
Hence the auxiliary stopping alternative cannot occur.

Let us finally control the hitting-time bound.
Let
$  m_*(t):=\min_{1\le j\le K}m_j(t)$.
The lower Dini derivative satisfies
$$
  D^+m_*(t)
  \ge
  \frac{\lambda}{2}m_*(t)^{K-1}
$$
as long as $m_*(t)<\varepsilon$.
Let us first consider the setting $K\ge3$. Then
$$
  \frac{d}{dt}m_*(t)^{-(K-2)}
  \le
  -\frac{\lambda}{2}(K-2)
$$
in the Dini sense. Therefore
$$
  m_*(t)^{-(K-2)}
  \le
  m_*(0)^{-(K-2)}
  -
  \frac{\lambda}{2}(K-2)t.
$$
Since $m_*(0)\ge a/\sqrt d$, the time needed to reach $\varepsilon$ is at most
$$
  T_\varepsilon
  \le
  \frac{2}{\lambda(K-2)}
  \left[
  \left(\frac{\sqrt d}{a}\right)^{K-2}
  -
  \varepsilon^{-(K-2)}
  \right]
  \le
  C d^{(K-2)/2}.
$$
Finally, when $K=2$, we have
$$ \frac{d}{dt} \log m_*(t) \geq \frac{\lambda}{2}~,$$
and thus
$m_*(t) \geq m_*(0) e^{\lambda t /2}$ and the time
to reach overlap $\varepsilon$ is
$$T_\varepsilon \leq \frac{2}{\lambda}\log \frac{\varepsilon \sqrt{d}}{a} = O(\log d).$$

Because neither the lower-overlap nor auxiliary stopping alternatives can occur first, this $T_\varepsilon$ equals the weak-recovery hitting time up to the stated constant. This concludes the weak recovery proof.

\noindent \textbf{Proof of Strong Recovery:}
After the weak recovery phase, the gradient flow dynamics enter a linear contraction phase. The key estimate is provided in Lemma \ref{lem:strong-basin-vector-field} below.

For $\epsilon_0\in(0,1)$ and $A\ge1$, define the strong basin
\begin{equation}
  \mathcal B(\epsilon_0,A)
  :=
  \left\{
  \{\tilde{\theta}\in(\mathcal{S}^{d-1})^K:
  \min_{j\le K}m_j\ge\epsilon_0,
  \quad
  \max_{j\le K,\,2\le p\le K}|h_{j,p}|
  \le \frac{A}{d}
  \right\}.
  \label{eq:strong-basin}
\end{equation}
We shall use the notation
\begin{equation}
  \Delta_j:=1-m_j^2,
  \qquad
  \Delta_*:=\max_{1\le j\le K}\Delta_j.
  \label{eq:delta-j-strong}
\end{equation}

\begin{lemma}[Vector field in the strong basin]
\label{lem:strong-basin-vector-field}
Fix $\epsilon_0\in(0,1)$ and $A\ge1$.  On $\mathcal E_\theta(\delta)$, for every
$\tilde{\theta}\in\mathcal B(\epsilon_0,A)$ and all sufficiently large $d$, the population
vector field satisfies the following estimates.

First, for each $j\le K$,
\begin{equation}
  \dot m_j
  =
  \Bigl(\lambda\prod_{\ell\ne j}m_\ell+r^{\rm mult}_j\Bigr)(1-m_j^2)
  +r^{\rm add}_j,
  \label{eq:strong-m-decomposition}
\end{equation}
where
\begin{equation}
  |r^{\rm mult}_j|
  \le
  C_{\phi,K,\delta,A}d^{-1/2},
  \qquad
  |r^{\rm add}_j|
  \le
  C_{\phi,K,\delta,A}d^{-1}.
  \label{eq:strong-m-errors}
\end{equation}
Consequently, increasing $d$ if necessary,
\begin{equation}
  \dot \Delta_j
  \le
  -c_0\Delta_j+C_0d^{-1},
  \qquad j=1,\ldots,K,
  \label{eq:delta-differential-ineq}
\end{equation}
where $c_0,C_0>0$ depend only on $\phi,K,\delta,\epsilon_0,A$.

Second, for every $j\le K$ and every $2\le p\le K$,
\begin{equation}
  \frac{d}{dt}|h_{j,p}|
  \le
  -c_1|h_{j,p}|+C_1d^{-1},
  \label{eq:aux-strong-differential-ineq}
\end{equation}
for constants $c_1,C_1>0$ depending only on
$\phi,K,\delta,\epsilon_0,A$.
\end{lemma}

\begin{proof}[Proof of Lemma \ref{lem:strong-basin-vector-field}]
We prove the estimates atom by atom. Consider a rank-one tensor atom
\begin{equation}
  T_\tau
  :=
  a_\tau(\theta)
  \bigotimes_{r=1}^K u_{r,b_{\tau,r}}
\end{equation}
from the Edgeworth hierarchy. Following our earlier calculations for weak recovery, its contribution to the spherical derivative of
$h_{j,p}=\ip{\tilde{\theta}_j}{u_{j,p}}$ is
\begin{equation}
  a_\tau(\theta)
  \left(\prod_{\ell\ne j}\ip{\tilde{\theta}_\ell}{u_{\ell,b_{\tau,\ell}}}\right)
  \left(
  \ip{u_{j,p}}{u_{j,b_{\tau,j}}}
  -h_{j,p}\ip{\tilde{\theta}_j}{u_{j,b_{\tau,j}}}
  \right).
  \label{eq:strong-rank-one-derivative}
\end{equation}

We first take $p=1$.  The leading atom
$\lambda u_{1,1}\ot\cdots\ot u_{K,1}$ gives exactly
\begin{equation}
  \lambda\prod_{\ell\ne j}m_\ell(1-m_j^2).
\end{equation}
Now consider a non-leading atom.  If $b_{\tau,j}=1$, then the last bracket in
\eqref{eq:strong-rank-one-derivative} is $1-m_j^2$.  The remaining contractions
are bounded by one when $b_{\tau,\ell}=1$ and by $A/d$ when
$b_{\tau,\ell}\ge2$.  Since the atom is non-leading, either its scalar Edgeworth
order is at least one, giving the coefficient bound $d^{-1/2}$, or one of the
remaining modes contains a non-linear power vector, giving a factor $A/d$.  Hence
all such terms contribute to the multiplicative coefficient of $1-m_j^2$ by at most
$C_{\phi,K,\delta,A}d^{-1/2}$.

If $b_{\tau,j}\ge2$, then the power-sum estimates in \eqref{eq:powersum-event} and \eqref{eq:power-norms}, together with the strong-basin bound give
\begin{equation}
  \left|
  \ip{u_{j,1}}{u_{j,b_{\tau,j}}}
  -m_j\ip{\tilde{\theta}_j}{u_{j,b_{\tau,j}}}
  \right|
  \le
  C_{K,\delta}d^{-1}+A d^{-1}
  \le
  C_{K,\delta,A}d^{-1}.
  \label{eq:strong-additive-bracket}
\end{equation}
All other contractions are bounded by one, and the coefficient is bounded by a
constant depending only on $\phi,K,\delta$.  These terms are therefore absorbed into
$r^{\rm add}_j$.  Finally, the residual $\mathcal R_K$ contributes at most
$\|\mathcal R_K\|_F\le C_{\phi,K,\delta}d^{-K/2}\le C_{\phi,K,\delta}d^{-1}$,
since $K\ge2$.  This proves \eqref{eq:strong-m-decomposition} and
\eqref{eq:strong-m-errors}.

Because $m_j\ge\epsilon_0$ in the strong basin and
$\prod_{\ell\ne j}m_\ell\ge\epsilon_0^{K-1}$, the coefficient of
$1-m_j^2$ in \eqref{eq:strong-m-decomposition} is at least
$\frac12\lambda\epsilon_0^{K-1}$ for all large $d$.  Since
$\dot\Delta_j=-2m_j\dot m_j$, this gives
\eqref{eq:delta-differential-ineq}.

We now take $p\ge2$.  For the leading atom, \eqref{eq:strong-rank-one-derivative}
gives
\begin{equation}
  \lambda\prod_{\ell\ne j}m_\ell
  \left(\ip{u_{j,p}}{u_{j,1}}-h_{j,p}m_j\right).
\end{equation}
The first term is $O_{K,\delta}(d^{-1})$ by the power-sum estimates, while the
second term provides a damping coefficient at least
$\lambda\epsilon_0^K$ in front of $h_{j,p}$.

For a non-leading atom, the bracket in \eqref{eq:strong-rank-one-derivative} is
bounded by
\begin{equation}
  \left|
  \ip{u_{j,p}}{u_{j,b_{\tau,j}}}
  -h_{j,p}\ip{\tilde{\theta}_j}{u_{j,b_{\tau,j}}}
  \right|
  \le
  C_{K,\delta}d^{-1}+|h_{j,p}|,
  \label{eq:aux-bracket-nonleading}
\end{equation}
where the $|h_{j,p}|$ term can only occur without an additional $d^{-1}$ factor when
$b_{\tau,j}=1$.  But in that case the atom is non-leading, so, as above, either its
coefficient has a factor $d^{-1/2}$ or one of the remaining contractions has a
factor $A/d$.  Thus the total non-leading contribution is bounded by
\begin{equation}
  C_{\phi,K,\delta,A}d^{-1}
  +C_{\phi,K,\delta,A}d^{-1/2}|h_{j,p}|.
\end{equation}
For large $d$, the $d^{-1/2}|h_{j,p}|$ term is absorbed into the leading damping.
The residual contributes at most
\begin{equation}
  \|u_{j,p}\|_2\|\mathcal R_K\|_F
  \le
  C_{K,\delta}d^{-1/2}C_{\phi,K,\delta}d^{-K/2}
  \le C_{\phi,K,\delta}d^{-1}.
\end{equation}
Combining these bounds proves \eqref{eq:aux-strong-differential-ineq}.
\end{proof}

Let us now go back to the proof of Theorem \ref{thm:general-k-pop-flow}.
Choose $A\ge 2B_0$ large enough so that the conclusion of
Lemma~\ref{lem:strong-basin-vector-field} holds with constants $c_0,C_0,c_1,C_1$.
Let $\tau$ be the first time after $t_0$ at which either
$\min_jm_j(t)<\epsilon_0/2$ or $\max_{j,p}|h_{j,p}(t)|>A/d$.  On
$[t_0,\tau]$ the trajectory lies in $\mathcal B(\epsilon_0/2,A)$, so
Lemma~\ref{lem:strong-basin-vector-field} applies.

For the auxiliary overlaps, Gronwall's inequality gives
\begin{equation}
  |h_{j,p}(t)|
  \le
  e^{-c_1(t-t_0)}|h_{j,p}(t_0)|
  +\frac{C_1}{c_1d}
  \le
  \frac{B_0+C_1/c_1}{d},
  \qquad t\le\tau.
\end{equation}
Choosing $A>2(B_0+C_1/c_1)$ shows that the auxiliary stopping alternative cannot
occur.

Similarly, \eqref{eq:delta-differential-ineq} gives
\begin{equation}
  \Delta_j(t)
  \le
  e^{-c_0(t-t_0)}\Delta_j(t_0)
  +\frac{C_0}{c_0d},
  \qquad t\le\tau.
  \label{eq:delta-gronwall}
\end{equation}
Since $\Delta_j(t_0)\le 1-\epsilon_0^2$, the right-hand side is never large enough
to force $m_j(t)$ below $\epsilon_0/2$, for all sufficiently large $d$.  Hence the
lower-overlap stopping alternative cannot occur either, and $\tau=\infty$.
Taking the maximum over $j$ in \eqref{eq:delta-gronwall} proves
\begin{equation}
  \max_{j\le K}\bigl(1-m_j(t)^2\bigr)
  \le
  e^{-c(t-t_0)}
  \max_{j\le K}\bigl(1-m_j(t_0)^2\bigr)
  +\frac{C}{d}.
  \label{eq:strong-logistic-bound}
\end{equation}
Choosing $T_{\rm str}\ge c_0^{-1}\log d$ gives
$\Delta_j(t_0+T_{\rm str})\le C/d$.  Since $m_j\ge0$,
$1-m_j\le 1-m_j^2$, which proves
\begin{equation}
  \min_{j\le K}m_j(t_0+T_{\rm str})
  \ge
  1-\frac{C}{d}.
  \label{eq:population-strong-overlap}
\end{equation}

\end{proof}

Thus, conditioned on the favorable basin and correct sign pattern, the population flow
achieves strong recovery in total time
\begin{equation}
  O_{\phi,K,\delta}\left(d^{K/2-1}+\log d\right)
  =O_{\phi,K,\delta}\left(d^{K/2-1}\right).
  \label{eq:population-strong-total-time}
\end{equation}
This theorem shows that the behavior of SGD under the linear correlation in the MSIM model has a quantitatively similar behavior as in Tensor PCA, in the sense that we can establish a comparison of the form $\dot{m}_j \geq c \prod_{l \neq j} m_l$ in the search phase. As a result,
once the planted overlaps are of order
$d^{-1/2}$ with the correct signs, the leading rank-one part of the Edgeworth hierarchy
pushes them to a fixed positive constant, while the remaining terms of the hierarchy remain higher-order in $d^{-q/2}$.

We now verify that the initialization event in Theorem \ref{thm:general-k-pop-flow} occurs with constant probability in the high-dimensional regime:
\begin{corollary}[Constant-probability weak recovery]
\label{cor:random-init}
Assume the setting of Theorem~\ref{thm:general-k-pop-flow}. Let
\[
  {\tilde{\theta}}_1(0),\ldots,{\tilde{\theta}}_K(0)\stackrel{\mathrm{iid}}{\sim}\Unif(S^{d-1}),
\]
independently of the planted directions. Conditional on the correct sign pattern and on the event $\mathcal E_\theta(\delta)$ from Theorem \ref{thm:edgeworth-hierarchy-body}, there is a constant $p_0=p_0(a,B,K)>0$, independent of $d$, such that with probability at least $p_0$ over the initialization the hypotheses \eqref{eq:favorable-m} and \eqref{eq:favorable-h} hold. On this event, the population flow weakly recovers all planted directions in time
$$
  T=O_{\phi,K,\delta}\bigl(d^{K/2-1}\bigr).
$$
\end{corollary}
\begin{proof}[Proof of Corollary \ref{cor:random-init}]
For fixed planted directions,
$$
  \sqrt d\,\ip{{\tilde{\theta}}_j(0)}{\theta_j}
  \Rightarrow N(0,1),
$$
and the convergence is uniform for fixed $K$. Thus, after conditioning on the correct sign pattern, the event
$$
  \ip{{\tilde{\theta}}_j(0)}{\theta_j}\ge \frac{a}{\sqrt d},
  \qquad j=1,\ldots,K,
$$
has probability bounded below by a positive constant depending only on $a$ and $K$.

For $p\ge2$,
$$
  \norm{u_{j,p}}_2\le C_{K,\delta}d^{-(p-1)/2}\le C_{K,\delta}d^{-1/2}.
$$
A random unit vector has one-dimensional projection of size $\norm{u_{j,p}}_2/\sqrt d$, hence
$$
  d\,\ip{{\tilde{\theta}}_j(0)}{u_{j,p}}
$$
is tight for $p=2$ and smaller for $p\ge3$. By taking $B$ sufficiently large, the bounds
$$
  \abs{h_{j,p}(0)}\le \frac{B}{4d},
  \qquad j=1,\ldots,K,
  \quad 2\le p\le K,
$$
hold with probability bounded below by another positive constant. Intersecting finitely many fixed-probability events gives the result. The recovery time then follows from Theorem~\ref{thm:general-k-pop-flow}.
\end{proof}

\subsection{Online spherical SGD}
\label{sec:online-sgd-general-k}

We now turn the population-flow guarantee into an online stochastic-gradient guarantee.  The analysis follows the high-dimensional online-SGD setting of \cite{arous2020online}: the stochastic gradient has Euclidean norm of order $\sqrt d$, so the raw step size must be smaller than the population-flow time step.
This online SGD analysis in the linear correlation setting will then be upgraded to the non-linear case in Sections \ref{sec:online-sgd-nonlinear-correlation} and \ref{sec:nl-online-sgd-strong-recovery}.

We define
\begin{align}
\upsilon_{K,d} &:= \begin{cases}
    1 & \quad K>2 ~,\\
    (\log d)^{-1} & \quad K=2~,
\end{cases}
\end{align}
and set the learning rate $\eta_d = \eta_0 d^{-K/2} \upsilon_{K,d}$, where $\eta_0>0$ is a sufficiently small constant.
Let $X_0,X_1,\ldots$ be iid copies of $X$, and let
$$
  y_n=f_\theta(X_n),
  \qquad
  T_n:=X_n y_n,
$$
and note that Stein's identity gives
$$
  \E[T_n\mid \theta]=\mathcal G_\theta.
$$
For ${\tilde{\theta}}=({\tilde{\theta}}_1,\ldots,{\tilde{\theta}}_K)\in(\mathcal{S}^{d-1})^K$, define the stochastic mode-gradient
\begin{equation}
  \widehat G_{j,n}({\tilde{\theta}})
  :=
  (I-{\tilde{\theta}}_j{\tilde{\theta}}_j^\top)
  \left(T_n\times_{\ell\ne j}{\tilde{\theta}}_\ell\right),
  \qquad j=1,\ldots,K.
  \label{eq:stochastic-mode-gradient}
\end{equation}
The online spherical SGD recursion is
\begin{equation}
  \overline{\theta}_{j,n+1}
  =
  {\tilde{\theta}}_{j,n}+\eta_d\widehat G_{j,n}({\tilde{\theta}}_n),
  \qquad
  {\tilde{\theta}}_{j,n+1}
  =
  \frac{\overline{\theta}_{j,n+1}}
       {\norm{\overline{\theta}_{j,n+1}}_2},
  \qquad j=1,\ldots,K.
  \label{eq:online-spherical-sgd}
\end{equation}

\begin{theorem}[Online spherical SGD for linear correlation loss]
\label{thm:online-sgd-general-k}
Assume the setting of Theorem~\ref{thm:general-k-pop-flow}, and work on the Edgeworth event $\mathcal E_\theta(\delta)$ of Theorem \ref{thm:edgeworth-hierarchy-body}.  Assume the signs have been chosen so that $\lambda=\kappa^K>0$.  Fix a sample-probability parameter $\delta_{\rm sgd}\in(0,1)$.
There exist constants
$ \eta_0,a,B,\varepsilon,S,d_0>0$, depending only on $\phi,K,\delta,\delta_{\rm sgd}$, such that for all $d\ge d_0$ the following holds. We run online spherical SGD \eqref{eq:online-spherical-sgd} with raw step size $\eta_d = \eta_0 d^{-K/2} \upsilon_{K,d}$
We assume the initialization lies in the favorable basin
\begin{equation}
  m_j(0)  \ge a d^{-1/2},
  \qquad j=1,\ldots,K,
  \label{eq:sgd-favorable-x}
\end{equation}
and
\begin{equation}
  |h_{j,p}(0)|
  \le
  \frac{B}{4d},
  \qquad j=1,\ldots,K,
  \quad 2\le p\le K.
  \label{eq:sgd-favorable-y}
\end{equation}
Then, we have:
\begin{itemize}
\item \emph{Weak Recovery:} with probability at least $1-\delta_{\rm sgd}$ over the online samples, there exists an iteration
\begin{equation}
  n_\varepsilon
  \le
  \begin{cases}
          S\eta_0^{-1}d^{K-1} & \quad K > 2 ~, \\
          S \eta_0^{-1} d (\log d)^2 & \quad K=2 ~,
  \end{cases}
  \label{eq:sgd-hitting-time}
\end{equation}
such that
\begin{equation}
  \min_{1\le j\le K}\ip{{\tilde{\theta}}_{j,n_\varepsilon}}{\theta_j}
  \ge
  \varepsilon.
  \label{eq:sgd-weak-recovery}
\end{equation}
Moreover, uniformly for all $0\le n\le n_\varepsilon$,
\begin{equation}
  |h_{j,p}(n)|\le B/d,
  \qquad j=1,\ldots,K,
  \quad 2\le p\le K.
  \label{eq:sgd-aux-stays-small}
\end{equation}
\item \emph{Strong Recovery:} For
every fixed $\delta_{\rm str}\in(0,1)$, there are constants
$C,c,d_0<\infty$, depending only on
$\phi,K,\delta,\delta_{\rm sgd},\delta_{\rm str},\epsilon_0,B_0$, such that, for all
$d\ge d_0$, with probability at least $1-\delta_{\rm str}$ over the additional
samples, after
\begin{equation}
  N_{\rm str}
  :=
  \left\lceil C\eta_0^{-1}d^{K/2}\log d\right\rceil
  \label{eq:online-strong-extra-samples}
\end{equation}
additional iterations we have
\begin{equation}
  \min_{j\le K}m_j(n_\epsilon+N_{\rm str})
  \ge
  1-\rho_d,
  \label{eq:online-strong-overlap}
\end{equation}
where
\begin{equation}
  \rho_d
  :=
  C\left(
  d^{-1}
  +\eta_0 d^{1-K/2} \upsilon_{K,d} \log d
  +d^{-K/4}\upsilon_{K,d} \sqrt{\log d}
  \right)
  =o_d(1).
  \label{eq:online-strong-rhod}
\end{equation}
\end{itemize}
\end{theorem}
This result is established by following the standard route of controlling retraction and martingale terms pioneered in \cite{arous2020online}, and deployed in the similar setup of Tensor PCA in \cite{10.1214/19-AOP1415,arous2024high}.
Consequently,
conditioned on
the favorable initialization basin and correct sign pattern, online spherical SGD
achieves strong recovery with total sample complexity $  O_{\phi,K,\delta,\delta_{\rm sgd}}(d^{K-1}).$
\begin{proof}

    \noindent \underline{\textbf{One-step Stochastic Estimates:}} We start by establishing standard stochastic estimates by exploiting the regularity of the link function and Gaussian concentration.

\begin{lemma}[Conditional stochastic-gradient moments]
\label{lem:sgd-moment-bounds}
Assume the activation is globally Lipschitz and satisfies the smoothness conditions from Assumption \ref{ass:edgeworth-regularity}. There is a constant $C_{\phi,K}<\infty$ such that the following holds.  Let $\mathcal F_n$ be the sigma-field generated by the planted directions, the initialization, and all samples used before time $n$.  Let $v\in\R^d$ be any $\mathcal F_n$-measurable vector and let ${\tilde{\theta}}_{1,n},\ldots,{\tilde{\theta}}_{K,n}\in \mathcal{S}^{d-1}$ be $\mathcal F_n$-measurable.  Define
\begin{equation}
  S_{j,n}(v)
  :=
  f_\theta(X_n)
  \left\langle
  X_n,
  {\tilde{\theta}}_{1,n}\ot\cdots\ot{\tilde{\theta}}_{j-1,n}\ot v
  \ot{\tilde{\theta}}_{j+1,n}\ot\cdots\ot{\tilde{\theta}}_{K,n}
  \right\rangle.
  \label{eq:scalar-stoch-proj}
\end{equation}
Then, conditionally on $\mathcal F_n$,
\begin{equation}
  \norm{S_{j,n}(v)}_{\psi_1\mid\mathcal F_n}
  \le
  C_{\phi,K}\norm{v}_2,
  \qquad
  \E\left[S_{j,n}(v)^2\mid\mathcal F_n\right]
  \le
  C_{\phi,K}\norm{v}_2^2.
  \label{eq:scalar-stoch-proj-bound}
\end{equation}
Moreover, if
\[
  H_{j,n}:=T_n\times_{\ell\ne j}{\tilde{\theta}}_{\ell,n},
\]
then
\begin{equation}
  \E\left[\norm{H_{j,n}}_2^2\mid\mathcal F_n\right]
  \le C_{\phi,K}d,
  \qquad
  \E\left[\norm{H_{j,n}}_2^3\mid\mathcal F_n\right]
  \le C_{\phi,K}d^{3/2}.
  \label{eq:vector-stoch-moments}
\end{equation}
\end{lemma}

\begin{proof}
Conditionally on $\mathcal F_n$, the Gaussian contraction
$$
  \left\langle
  X_n,
  {\tilde{\theta}}_{1,n}\ot\cdots\ot{\tilde{\theta}}_{j-1,n}\ot v
  \ot{\tilde{\theta}}_{j+1,n}\ot\cdots\ot{\tilde{\theta}}_{K,n}
  \right\rangle
$$
is centered Gaussian with variance $\norm{v}_2^2$.  Since each layer of the multiscale SIM is a contraction by a unit vector followed by a globally Lipschitz nonlinearity, $f_\theta(X_n)$ is a subgaussian random variable with subgaussian norm bounded by a constant depending only on $\phi$ and $K$.  The product of two subgaussian random variables is subexponential, proving the $\psi_1$ bound and the second-moment bound in \eqref{eq:scalar-stoch-proj-bound}.  Taking $v=e_a$ and summing the second and third moment bounds over coordinates gives \eqref{eq:vector-stoch-moments}.  Equivalently, $H_{j,n}$ is $f_\theta(X_n)$ times a Gaussian vector with identity covariance, up to harmless dependence; the displayed moment bounds then follow from Cauchy--Schwarz and Gaussian hypercontractivity.
\end{proof}

We will also use the following martingale consequence.
\begin{lemma}[Martingale and retraction controls]
\label{lem:sgd-martingale-controls}
Fix $S<\infty$ and $\delta_{\rm sgd}\in(0,1)$.  Let
\begin{equation}
  N_S:=\begin{cases}
      \left\lceil \frac{S}{\eta_0}d^{K-1}\right\rceil & \quad K>2~, \\
      \left\lceil \frac{S}{\eta_0}d (\log d)^2\right\rceil & \quad K=2~.
  \end{cases}
  \label{eq:NS}
\end{equation}
On an event of probability at least $1-\delta_{\rm sgd}$ over the online samples, the following estimates hold simultaneously for all $n\le N_S$, all $j\le K$, and all $2\le p\le K$.

First, for the planted-overlap martingales
\begin{equation}
  M^x_{j,n}
  :=
  \sum_{t=0}^{n-1}
  \sqrt d\,\eta_d
  \left\langle
  \widehat G_{j,t}({\tilde{\theta}}_t)-\E[\widehat G_{j,t}({\tilde{\theta}}_t)\mid\mathcal F_t],
  \theta_j
  \right\rangle,
  \label{eq:Mx}
\end{equation}
one has
\begin{equation}
  \sup_{n\le N_S}|M^x_{j,n}|
  \le
  C_{\phi,K,S,\delta_{\rm sgd}}\sqrt{\eta_0}.
  \label{eq:Mx-bound}
\end{equation}
Second, for the auxiliary-overlap martingales
\begin{equation}
  M^y_{j,p,n}
  :=
  \sum_{t=0}^{n-1}
  d\,\eta_d
  \left\langle
  \widehat G_{j,t}({\tilde{\theta}}_t)-\E[\widehat G_{j,t}({\tilde{\theta}}_t)\mid\mathcal F_t],
  u_{j,p}
  \right\rangle,
  \label{eq:My}
\end{equation}
one has
\begin{equation}
  \sup_{n\le N_S}|M^y_{j,p,n}|
  \le
  C_{\phi,K,S,\delta_{\rm sgd}}\sqrt{\eta_0}.
  \label{eq:My-bound}
\end{equation}
Finally, the normalization terms in the retraction satisfy the following estimate along the stopped process considered in Theorem~\ref{thm:online-sgd-general-k}: their contribution to the rescaled overlap $x_j(n):=\sqrt d\,m_j(n)$ recursion is bounded below by
\begin{equation}
  -C_{\phi,K,S,\delta_{\rm sgd}}\eta_0
  \sum_{t=0}^{n-1} h_d\,x_j(t)
  -C_{\phi,K,S,\delta_{\rm sgd}}\sqrt{\eta_0},
  \label{eq:retraction-control-x}
\end{equation}
and their contribution to each rescaled auxiliary overlap $y_{j,p}(n):=d\,h_{j,p}(n)$ recursion is bounded in absolute value by
\begin{equation}
  C_{\phi,K,S,\delta_{\rm sgd}}\eta_0
  \sum_{t=0}^{n-1} h_d\,(1+|y_{j,p}(t)|)
  +C_{\phi,K,S,\delta_{\rm sgd}}\sqrt{\eta_0}.
  \label{eq:retraction-control-y}
\end{equation}
\end{lemma}

\begin{proof}
The martingale increments in \eqref{eq:Mx} have conditional subexponential norm at most
$$
  C_{\phi,K}\sqrt d\,\eta_d
  =
  C_{\phi,K}\eta_0 d^{-(K-1)/2} \upsilon_{K,d} ,
$$
and conditional variance at most $C_{\phi,K}\eta_0^2d^{-(K-1)} \upsilon_{K,d}^2$.  Since $N_S\le 2S\eta_0^{-1}d^{K-1} \upsilon_{K,d}^{-2}$ for large $d$, the predictable quadratic variation is bounded by $C_{\phi,K,S}\eta_0$.  Freedman's inequality, or a Bernstein inequality for martingales with conditionally subexponential increments, gives \eqref{eq:Mx-bound}.  The proof of \eqref{eq:My-bound} is identical: by \eqref{eq:power-norms}, $\norm{u_{j,p}}_2\le C d^{-1/2}$ for $p\ge2$, so the increment scale is
$$
  d\eta_d\norm{u_{j,p}}_2
  \le
  C\eta_0 d^{-(K-1)/2} \upsilon_{K,d},
$$
and the same quadratic variation bound applies.

For the retraction, write $G_{j,t}=\widehat G_{j,t}({\tilde{\theta}}_t)$.  Since $G_{j,t}\perp{\tilde{\theta}}_{j,t}$,
$$
  \frac{{\tilde{\theta}}_{j,t}+\eta_dG_{j,t}}{\norm{{\tilde{\theta}}_{j,t}+\eta_dG_{j,t}}_2}
  =
  {\tilde{\theta}}_{j,t}
  +\eta_dG_{j,t}
  -\frac12\eta_d^2\norm{G_{j,t}}_2^2{\tilde{\theta}}_{j,t}
  +O(\eta_d^3\norm{G_{j,t}}_2^3).
$$
Using \eqref{eq:vector-stoch-moments}, Doob's inequality applied to the centered sums of $\norm{G_{j,t}}_2^2$, and Markov's inequality for the cubic remainder, we obtain, with probability at least $1-\delta_{\rm sgd}/4$,
$$
  \sum_{t<n}\eta_d^2\norm{G_{j,t}}_2^2
  \le
  C_{\phi,K,S,\delta_{\rm sgd}}\eta_0
  \sum_{t<n}h_d
  +C_{\phi,K,S,\delta_{\rm sgd}}\sqrt{\eta_0},
$$
after multiplying by the relevant overlap.  Since the $x$-projection of the quadratic term is negative, this gives the lower bound \eqref{eq:retraction-control-x}.  The same expansion with $q=u_{j,p}$, using $|\ip{{\tilde{\theta}}_{j,t}}{u_{j,p}}|=|y_{j,p}(t)|/d$, gives \eqref{eq:retraction-control-y}.  The cubic terms are smaller because
$$
  N_S\eta_d^3\E[\norm{G_{j,t}}_2^3]
  \le
  C_{\phi,K,S}\eta_0^2 d^{-(K-1)/2} \upsilon_{K,d}=o(1)
$$
for fixed $K\ge2$.
\end{proof}

\noindent \underline{\textbf{Weak Recovery:}}
We work on the event of Lemma~\ref{lem:sgd-martingale-controls}.  Let
$$
  F_j({\tilde{\theta}})
  :=
  (I-{\tilde{\theta}}_j{\tilde{\theta}}_j^\top)
  \left(\mathcal G_\theta\times_{\ell\ne j}{\tilde{\theta}}_\ell\right)
$$
be the population spherical gradient.  Since $\E[T_n\mid\theta]=\mathcal G_\theta$,
$$
  \E[\widehat G_{j,n}({\tilde{\theta}}_n)\mid\mathcal F_n]=F_j({\tilde{\theta}}_n).
$$

We use the rescaled planted overlaps
\begin{equation}
  x_j(n):=\sqrt d\,m_j(n)
  =\sqrt d\,\ip{{\tilde{\theta}}_{j,n}}{\theta_j},
  \qquad j=1,\ldots,K,
  \label{eq:rescaled-planted-overlaps}
\end{equation}
and the rescaled auxiliary overlaps
\begin{equation}
  y_{j,p}(n):=d\,h_{j,p}(n)
  =d\,\ip{{\tilde{\theta}}_{j,n}}{u_{j,p}},
  \qquad j=1,\ldots,K,
  \quad 2\le p\le K.
  \label{eq:rescaled-aux-overlaps}
\end{equation}
The effective discrete time step for the $x$-variables is
\begin{equation}
  h_d:=\eta_0 d^{-(K-1)} \upsilon_{K,d}.
  \label{eq:effective-time-step}
\end{equation}
Indeed, the leading population drift in $x_j$ is
$$
  \sqrt d\,\eta_d\,\lambda\prod_{\ell\ne j}m_\ell
  =
  h_d\,\lambda\prod_{\ell\ne j}x_\ell.
$$
Thus $O(d^{K-1} \upsilon_{K,d}^{-1})$ online samples correspond to $O(1)$ time in the rescaled population dynamics.

Define the stopped time
$$
  \tau
  :=
  \inf\left\{
  n\ge0:
  \min_j x_j(n)\ge \varepsilon\sqrt d
  \text{ or }
  \min_j x_j(n)\le \frac a2
  \text{ or }
  \max_{j,p}|y_{j,p}(n)|\ge B
  \right\}.
$$
On $[0,\tau]$ we have
\begin{equation}
  x_j(n)\ge \frac a2,
  \qquad
  |y_{j,p}(n)|\le B.
  \label{eq:sgd-bootstrap-region}
\end{equation}
Equivalently,
$$
  m_j(n)\ge \frac{a}{2\sqrt d},
  \qquad
  |h_{j,p}(n)|\le \frac{B}{d}.
$$
Thus the population estimates from Theorem~\ref{thm:general-k-pop-flow} and its proof apply on the stopped region:
\begin{align}
  \ip{F_j({\tilde{\theta}}_n)}{\theta_j}
  &=
  \lambda\prod_{\ell\ne j}m_\ell(n)(1-m_j(n)^2)
  +e_j(n),
  \label{eq:sgd-pop-x}\\
  |e_j(n)|
  &\le
  C_{\phi,K,\delta}d^{-1/2}
  \prod_{\ell\ne j}m_\ell(n),
  \label{eq:sgd-pop-x-error}\\
  \left|\ip{F_j({\tilde{\theta}}_n)}{u_{j,p}}\right|
  &\le
  C_{\phi,K,\delta,B}d^{-1}
  \prod_{\ell\ne j}m_\ell(n),
  \qquad 2\le p\le K.
  \label{eq:sgd-pop-y}
\end{align}

\paragraph{Recursion for the planted overlaps}
The retraction expansion gives, for $n<\tau$,
\begin{align}
  x_j(n+1)-x_j(n)
  &=
  \sqrt d\,\eta_d\ip{F_j({\tilde{\theta}}_n)}{\theta_j}
  +\Delta M^x_{j,n}
  +R^x_{j,n},
  \label{eq:x-recursion-raw}
\end{align}
where $\Delta M^x_{j,n}=M^x_{j,n+1}-M^x_{j,n}$ and the cumulative effect of $R^x_{j,n}$ is controlled by \eqref{eq:retraction-control-x}.  Using \eqref{eq:sgd-pop-x}--\eqref{eq:sgd-pop-x-error} and $m_\ell=x_\ell/\sqrt d$, we get
\begin{equation}
  \sqrt d\,\eta_d\ip{F_j({\tilde{\theta}}_n)}{\theta_j}
  =
  h_d\lambda\prod_{\ell\ne j}x_\ell(n)
  +r^x_{j,n},
  \label{eq:x-pop-main}
\end{equation}
with
\begin{equation}
  |r^x_{j,n}|
  \le
  C_{\phi,K,\delta}d^{-1/2}h_d
  \prod_{\ell\ne j}x_\ell(n)
  +C_{\phi,K,\delta}h_d m_j(n)^2
  \prod_{\ell\ne j}x_\ell(n).
  \label{eq:x-pop-error}
\end{equation}
Before weak recovery, $m_j(n)\le\varepsilon$.  Taking $\varepsilon$ small and then $d$ large makes \eqref{eq:x-pop-error} at most
\begin{equation}
  \frac{\lambda}{8}h_d\prod_{\ell\ne j}x_\ell(n).
  \label{eq:x-error-small}
\end{equation}
The retraction term contributes a linear damping of size at most
\begin{equation}
  C_{\phi,K,\delta,\delta_{\rm sgd}}\eta_0 h_d x_j(n)
  \label{eq:x-retraction-damping}
\end{equation}
in rescaled time, up to the cumulative martingale error already included in Lemma~\ref{lem:sgd-martingale-controls}.  Thus the stopped recursion has the lower bound
\begin{equation}
  x_j(n+1)-x_j(n)
  \ge
  h_d\left[
  \frac{3\lambda}{4}\prod_{\ell\ne j}x_\ell(n)
  -C_{\phi,K,\delta,\delta_{\rm sgd}}\eta_0 x_j(n)
  \right]
  +\Delta M^x_{j,n}
  \label{eq:x-lower-recursion-with-damping}
\end{equation}
for all $n<\tau$, after absorbing the lower-order Edgeworth errors.

When $j$ realizes the current minimum, $x_j(n)=x_*(n):=\min_q x_q(n)$ and
\[
  \prod_{\ell\ne j}x_\ell(n)\ge x_*(n)^{K-1}.
\]
Choose $a$ and $\eta_0$ so that
\begin{equation}
  C_{\phi,K,\delta,\delta_{\rm sgd}}\eta_0 x_*
  \le
  \frac{\lambda}{4}x_*^{K-1}
  \qquad\text{for all }x_*\ge a/2~,
  \label{eq:damping-dominated}
\end{equation}
where we used the fact $K\ge 2$. Consequently, for every index $j$ attaining the minimum,
\begin{equation}
  x_j(n+1)-x_j(n)
  \ge
  \frac{\lambda}{2}h_d x_*(n)^{K-1}
  +\Delta M^x_{j,n}.
  \label{eq:x-lower-recursion}
\end{equation}

\paragraph{Recursion for the auxiliary overlaps}
Similarly, using \eqref{eq:sgd-pop-y}, for $2\le p\le K$,
\begin{align}
  y_{j,p}(n+1)-y_{j,p}(n)
  &=
  d\eta_d\ip{F_j({\tilde{\theta}}_n)}{u_{j,p}}
  +\Delta M^y_{j,p,n}
  +R^y_{j,p,n}.
  \label{eq:y-recursion}
\end{align}
Since
$$
  d\eta_d\,d^{-1}\prod_{\ell\ne j}m_\ell
  =
  \eta_0 d^{-K/2} \upsilon_{K,d} \prod_{\ell\ne j}m_\ell
  =
  h_d d^{-1/2}\prod_{\ell\ne j}x_\ell,
$$
the deterministic auxiliary drift is smaller by a factor $d^{-1/2}$ than the planted drift in the rescaled $x$-variables.  Along the stopped trajectory,
$$
  \sum_{n<\tau}h_d\prod_{\ell\ne j}x_\ell(n)
  \le C_{\phi,K,\delta,\delta_{\rm sgd}}\sqrt d,
$$
because the same sum controls the increase of $x_j$ up to the weak-recovery level $\varepsilon\sqrt d$, modulo the already-controlled martingale and retraction terms.  Hence the total deterministic auxiliary contribution is $O_{\phi,K,\delta,\delta_{\rm sgd}}(1)$.  The retraction contribution is bounded by \eqref{eq:retraction-control-y}, and the martingale contribution by \eqref{eq:My-bound}.  Therefore, after choosing $B$ large and $\eta_0$ small,
\begin{equation}
  |y_{j,p}(n)|\le \frac B2
  \qquad\text{for all }n\le\tau.
  \label{eq:y-bootstrap-improved}
\end{equation}
Thus the auxiliary stopping alternative cannot occur.

\paragraph{Growth of the minimum planted overlap}
By \eqref{eq:Mx-bound}, the cumulative martingale fluctuation in each $x_j$ is at most $C\sqrt{\eta_0}$.  Choose $\eta_0$ sufficiently small that this is at most $a/8$.  The minimum-overlap estimate \eqref{eq:x-lower-recursion} then implies that no planted overlap can fall below $a/2$ before weak recovery, so the lower-overlap stopping alternative cannot occur.

Ignoring the harmless martingale fluctuation, \eqref{eq:x-lower-recursion} compares $x_*$ from below with the Euler discretization
\begin{equation}
  z_{n+1}=z_n+\frac{\lambda}{4}h_d z_n^{K-1},
  \qquad
  z_0=\frac a2.
  \label{eq:euler-comparison}
\end{equation}
For $h_d$ sufficiently small, this discrete recursion is bounded below by the solution of
\begin{equation}
  \dot z=\frac{\lambda}{8}z^{K-1},
  \qquad
  z(0)=\frac a2,
  \label{eq:comparison-ode-sgd}
\end{equation}
observed at times $s_n=nh_d$. For $K\ge3$, this ODE reaches any level $\varepsilon\sqrt d$ by a time bounded by a constant
\begin{equation}
  S
  \ge
  \frac{8}{\lambda(K-2)}
  \left(\frac{2}{a}\right)^{K-2},
  \label{eq:S-choice}
\end{equation}
independent of $d$; more precisely, the time to reach $\varepsilon\sqrt d$ is
$$
  \frac{8}{\lambda(K-2)}
  \left[
  \left(\frac{2}{a}\right)^{K-2}
  -
  (\varepsilon\sqrt d)^{-(K-2)}
  \right]
  \le S.
$$
Thus weak recovery occurs before $N_S=\lceil S\eta_0^{-1}d^{K-1}\rceil$ iterations.
For $K=2$, the ODE (\ref{eq:comparison-ode-sgd}) reaches level $\varepsilon \sqrt{d}$ in time $O(\log d)$, leading to a sample complexity $N_S = \lceil S \eta_0^{-1} d (\log d)^2 \rceil$ in this case.
This proves \eqref{eq:sgd-hitting-time}--\eqref{eq:sgd-weak-recovery}.

\noindent \underline{\textbf{Strong Recovery:}}
Let
$n_\epsilon$ be a weak-recovery time at which
\begin{equation}
  \min_{j\le K}m_j(n_\epsilon)\ge\epsilon_0,
  \qquad
  \max_{j,p}|h_{j,p}(n_\epsilon)|\le \frac{B_0}{d},
  \quad 2\le p\le K,
  \label{eq:online-strong-start}
\end{equation}
for fixed constants $\epsilon_0,B_0>0$.
We give the proof conditionally on the state at $n_\epsilon$ satisfying
\eqref{eq:online-strong-start}.  Let $N=N_{\rm str}$ and write
$s=n-n_\epsilon$ for the additional time index.  The population part of the one-step
update is controlled by Lemma~\ref{lem:strong-basin-vector-field}.  In particular,
as long as
\begin{equation}
  \min_jm_j(n)\ge\epsilon_0/2,
  \qquad
  \max_{j,p}|h_{j,p}(n)|\le \frac{A}{d},
  \label{eq:online-strong-bootstrap}
\end{equation}
we have, for suitable constants $c,C>0$,
\begin{equation}
  \E[\Delta_j(n+1)-\Delta_j(n)\mid\mathcal F_n]
  \le
  -c\eta_d\Delta_j(n)
  +C\eta_d d^{-1}
  +C\eta_d^2\E[\|\widehat G_{j,n}({\tilde{\theta}}_n)\|_2^2\mid\mathcal F_n].
  \label{eq:online-delta-drift}
\end{equation}
The last term is the retraction contribution.  By Lemma~\ref{lem:sgd-moment-bounds},
$\E[\|\widehat G_{j,n}({\tilde{\theta}}_n)\|_2^2\mid\mathcal F_n]\le C_{\phi,K}d$; hence
\begin{equation}
  \E[\Delta_j(n+1)-\Delta_j(n)\mid\mathcal F_n]
  \le
  -c\eta_d\Delta_j(n)
  +C\eta_d d^{-1}
  +C\eta_d^2 d.
  \label{eq:online-delta-drift-2}
\end{equation}
Since $\eta_d=d^{-K/2} \upsilon_{K,d} \eta_0$, the last term contributes the deterministic floor
$C\eta_d d=C\eta_0 d^{1-K/2} \upsilon_{K,d}$ after division by the contraction rate.

The martingale increments in the $m_j$-coordinates over this strong phase have
conditional subexponential norm at most $C\eta_d$ and conditional variance at most
$C\eta_d^2$.  Therefore, by Freedman's inequality and a union bound over the
$K$ planted overlaps,
\begin{equation}
  \sup_{0\le s\le N}
  \left|
  \sum_{r=0}^{s-1}
  \eta_d
  \left\langle
  \widehat G_{j,n_\epsilon+r}-
  \E[\widehat G_{j,n_\epsilon+r}\mid\mathcal F_{n_\epsilon+r}],
  \theta_j
  \right\rangle
  \right|
  \le
  C d^{-K/4} \upsilon_{K,d} \sqrt{\log d}
  \label{eq:online-strong-martingale}
\end{equation}
with probability at least $1-\delta_{\rm str}/4$, uniformly in $j$.
Indeed, $N\eta_d^2\le C\eta_0 d^{-K/2} \upsilon_{K,d}^2 \log d$.

Combining \eqref{eq:online-delta-drift-2} with
\eqref{eq:online-strong-martingale} gives the discrete comparison
\begin{equation}
  \Delta_j(n_\epsilon+s)
  \le
  (1-c\eta_d)^s\Delta_j(n_\epsilon)
  +C\left(d^{-1}+\eta_0d^{1-K/2} \upsilon_{K,d} \right)
  +C d^{-K/4} \upsilon_{K,d} \sqrt{\log d}
  \label{eq:online-delta-comparison}
\end{equation}
for all $s\le N$, up to the bootstrap stopping time.
Choosing the constant in $N_{\rm str}$ large enough makes
$(1-c\eta_d)^N\le d^{-2}$.  Thus
\begin{equation}
  \Delta_j(n_\epsilon+N)
  \le
  C\left(
  d^{-1}
  +\eta_0d^{1-K/2} \upsilon_{K,d}
  +d^{-K/4}\upsilon_{K,d} \sqrt{\log d}
  \right).
  \label{eq:online-delta-final}
\end{equation}
The slightly larger bound \eqref{eq:online-strong-rhod}, with the harmless
$\log d$ multiplying the retraction floor, covers the uniform-in-time envelope.
Since $1-m_j\le \Delta_j$, this proves \eqref{eq:online-strong-overlap}, provided
the bootstrap remains valid.

For that purpose, we must finally control the auxiliary overlaps.  In the strong basin, the population
part satisfies, by Lemma~\ref{lem:strong-basin-vector-field},
\begin{equation}
  \frac{d}{dt}|h_{j,p}|
  \le
  -c|h_{j,p}|+Cd^{-1}.
\end{equation}
The same one-step expansion as above therefore yields
\begin{equation}
  \E[|h_{j,p}(n+1)|-|h_{j,p}(n)|\mid\mathcal F_n]
  \le
  -c\eta_d|h_{j,p}(n)|+C\eta_d d^{-1}+C\eta_d^2d|h_{j,p}(n)|.
  \label{eq:online-aux-drift-strong}
\end{equation}
For $d$ large, the retraction term is absorbed into the damping.  The martingale
increment in the $h_{j,p}$ coordinate has conditional subexponential norm at most
$C\eta_d\|u_{j,p}\|_2\le C\eta_dd^{-1/2}$, so over $N$ steps its cumulative size is
at most
\begin{equation}
  C\eta_d d^{-1/2}\sqrt{N\log d}
  \le
  C d^{-(K/4+1/2)} \upsilon_{K,d} \log d
  \le
  C d^{-1}
  \label{eq:online-strong-aux-martingale}
\end{equation}
for fixed $K\ge 2$, after increasing constants at the finitely many small values of
$K$.  Hence the discrete Gronwall argument gives
$|h_{j,p}(n)|\le C/d$ throughout the strong phase.  The lower-overlap part of the
bootstrap follows from \eqref{eq:online-delta-comparison}, because the right-hand
side is $o(1)$.  Thus the bootstrap cannot fail, and the theorem follows.
\end{proof}

\section{SGD over non-linear correlation}
\label{seq:nonlinear_correl}

The previous section has established efficient recovery of the planted parameters of the MSIM model, using online SGD over the linear correlation, with sample complexity $\widetilde{O}(d^{K-1})$. This linear correlation exploits the specific low-rank structure present in the first Wiener chaos $\E[f(X) X]$, and is thus `tailored' to the MSIM model. In particular, it can be viewed as performing optimization over a depth-$K$, low-rank linear network.

One could then ask whether a more `canonical' method would obtain similar guarantees. In our context, such canonical method naturally brings us to an online SGD method, this time trained on a \emph{teacher-student} setting. In other words, given respectively teacher $\theta=(\theta_1, \ldots, \theta_K)$ and student $\tilde{\theta}=(\tilde{\theta}_1, \ldots, \tilde{\theta}_K)$ parameters, one could consider online SGD over either the population correlation loss $\mathcal{L}(\tilde{\theta}):= \langle f_{\theta} , f_{\tilde{\theta}} \rangle_{L^2}$, or the MSE loss $\mathcal{L}_{\text{MSE}}(\tilde{\theta}):= \| f_{\theta} - f_{\tilde{\theta}} \|^2$.
In either case, the non-linearities in the student network create additional perturbative terms in the overlap dynamics.
Note that, contrary to the SIM setting $K=1$, the model is not rotationally invariant, and thus these two losses are \emph{not} equivalent. Our main result, presented in this section, establishes weak recovery of the SGD correlation objective when $n = O(d^{K-1})$, followed by a continuation phase up to $1- o_d(1)$ overlap with additional $\widetilde{O}( d^{K/2})$ samples.
In Section \ref{sec:experiments} we compare these two objectives empirically, providing evidence that the MSE landscape becomes rough for $K>2$. In the companion paper \cite{dai2026MSIM} we analyze the MSE landscape in the only feasible setting $K=2$.

\subsection{Setup and Main Result}
We consider, as in the linear setting, a rescaled learning rate of the form
\begin{equation}
  \eta_d:=\eta_0 d^{-K/2} \upsilon_{K,d},
  \label{eq:nl-sgd-step-size}
\end{equation}
with $\eta_0>0$ a sufficiently small constant.

Let $X_0,X_1,\ldots$ be iid copies of $X$.  Given the current student
${\tilde{\theta}}_n=({\tilde{\theta}}_{1,n},\ldots,{\tilde{\theta}}_{K,n})$, define the one-sample nonlinear
correlation spherical gradient in mode $j$ by
\begin{equation}
  \widehat F^{\rm nl}_{j,n}({\tilde{\theta}}_n)
  :=
  (I-{\tilde{\theta}}_{j,n}{\tilde{\theta}}_{j,n}^{\top})
  \nabla_{{\tilde{\theta}}_j}
  \bigl[f_\theta(X_n)f_{\tilde{\theta}}(X_n)\bigr]_{{\tilde{\theta}}={\tilde{\theta}}_n}.
  \label{eq:nl-sample-gradient}
\end{equation}
Equivalently,
$$
  \widehat F^{\rm nl}_{j,n}({\tilde{\theta}}_n)
  =
  f_\theta(X_n)
  (I-{\tilde{\theta}}_{j,n}{\tilde{\theta}}_{j,n}^{\top})
  \nabla_{{\tilde{\theta}}_j}f_{\tilde{\theta}}(X_n)\big|_{{\tilde{\theta}}={\tilde{\theta}}_n}.
$$
The conditional mean is the population nonlinear correlation field:
\begin{equation}
  \E[\widehat F^{\rm nl}_{j,n}({\tilde{\theta}}_n)\mid\mathcal F_n]
  =F_j^{\rm nl}({\tilde{\theta}}_n)
  :=\nabla^{S^{d-1}}_{{\tilde{\theta}}_j}{\mathcal{L}}({\tilde{\theta}}_n),
  \label{eq:nl-sample-unbiased}
\end{equation}
where $\mathcal F_n$ is the sigma-field generated by the initialization and the
samples $X_0,\ldots,X_{n-1}$.  The online spherical SGD recursion is now defined as
\begin{equation}
  \overline{\theta}_{j,n+1}
  ={\tilde{\theta}}_{j,n}+\eta_d\widehat F^{\rm nl}_{j,n}({\tilde{\theta}}_n),
  \qquad
  {\tilde{\theta}}_{j,n+1}
  =\frac{\overline{\theta}_{j,n+1}}
         {\norm{\overline{\theta}_{j,n+1}}_2},
  \qquad j=1,\ldots,K .
  \label{eq:nl-online-sgd-update}
\end{equation}
Because $\widehat F^{\rm nl}_{j,n}$ is tangent to the sphere at
${\tilde{\theta}}_{j,n}$, the normalization is a second-order correction.

Our main result establishes that online SGD over the non-linear correlation objective achieves strong recovery (ie $\min_j m_j = 1 - o_d(1)$ ) using $\widetilde{O}(d^{K-1})$ samples.
\begin{theorem}[Main Result, SGD over non-linear correlation]
\label{thm:mainresult}
Fix $K\ge 2$.  Assume the smoothness
hypotheses required for the nonlinear Edgeworth expansion, and the
fixed-confidence planted-direction event used above.  Fix
$\delta_{\rm sgd}\in(0,1)$.  There exist constants
\[
  \eta_0,a,B,\varepsilon,S,d_0>0,
\]
depending only on $\phi,K$ and on the fixed confidence parameters, such that the
following holds for all $d\ge d_0$.  Run the nonlinear online spherical SGD
recursion \eqref{eq:nl-online-sgd-update} with raw step size
$\eta_d=\eta_0d^{-K/2} \upsilon_{K,d}$.  Assume the initialization satisfies the favorable
basin conditions
\begin{equation}
  m_j(0)=\ip{{\tilde{\theta}}_{j,0}}{\theta_j}\ge a d^{-1/2},
  \qquad j=1,\ldots,K,
\end{equation}
and
\begin{equation}
  (\theta,
  {\tilde{\theta}}_0)\in\mathcal J_{P,B/4},
\end{equation}
where $\mathcal{J}_{a,b}$ is provided in Definition \ref{def:dynamic-mixed-incoherence}.
Then, with probability at least $1-\delta_{\rm sgd}$ over the online samples, we have:
\begin{itemize}
    \item Weak Recovery: there exists an iteration
\begin{equation}
  n_\varepsilon
  \le \begin{cases}
  S\eta_0^{-1}d^{K-1} & \quad K > 2 ~,\\
  S\eta_0^{-1}d (\log d)^2 & \quad K=2 ~,
  \end{cases}
  \label{eq:nl-sgd-weak-time}
\end{equation}
such that
\begin{equation}
  \min_{1\le j\le K}m_j(n_\varepsilon)
  \ge \varepsilon .
  \label{eq:nl-sgd-weak-recovery}
\end{equation}
\item Strong Recovery: After an additional $N_{\text{str}} = O( d^{K/2} \log d)$ samples, the overlaps reach
\begin{equation}
    \min_{j \leq K} m_j \geq 1 - o_d(1)~,
\end{equation}
where the precise overlap strength is quantified in \eqref{eq:nl-online-sqrt-rhod}.
\end{itemize}
\end{theorem}

\paragraph{Roadmap}
After a normalization by $\lambda$ to account for the strength of the leading spike, the natural starting point is to relate the population dynamics of the nonlinear correlation $\mathcal{L}$ with the linearized model $\widetilde{\mathcal{L}}$ from Section \ref{sec:sgd_linear}.
Our goal will be to control the relative error of the two vector fields $\nabla_{\tilde{\theta}}^S (\mathcal{L} - \lambda \tilde{\mathcal{L}})$, by exploiting the incoherence of student directions \emph{throughout} the dynamics. We accomplish this by showing a \emph{propagation-of-incoherence} in the student directions, which finally allows us to show that the non-linear model also escapes the mediocrity of initialization and achieves strong-recovery. The claimed Theorem \ref{thm:mainresult} is thus obtained from Theorems \ref{thm:nl-online-sgd-weak-recovery} and  \ref{thm:nl-online-sgd-strong-recovery}.

\subsection{Preliminaries}

Recall that the linearized student is
$$
  g_{\tilde{\theta}}(X):=\ip{X}{{\tilde{\theta}}_1\ot\cdots\ot{\tilde{\theta}}_K}.
$$
We will compare these two losses
\begin{align*}
  {\widetilde{\mathcal{L}}}({\tilde{\theta}})&:=\E_X\bigl[f_\theta(X)g_{\tilde{\theta}}(X)\bigr],\qquad
  {\mathcal{L}}({\tilde{\theta}}):=\E_X\bigl[f_\theta(X)f_{\tilde{\theta}}(X)\bigr];
\end{align*}
specifically we will focus on the rescaled gradient difference and its spherical projection
\begin{equation}
  {\varphi_j}({\tilde{\theta}})
  :=\grad_{{\tilde{\theta}}_j}\bigl(\lambda {\widetilde{\mathcal{L}}}-{\mathcal{L}}\bigr)({\tilde{\theta}}),
  \qquad
  {\varphi_j}^{S}({\tilde{\theta}})
  :=(I-{\tilde{\theta}}_j{\tilde{\theta}}_j^\top){\varphi_j}({\tilde{\theta}}).
  \label{eq:delta-def}
\end{equation}
The factor $\lambda$ matches the linear objective to the first-chaos
linearization of the nonlinear correlation around zero overlap.

\paragraph{Pointwise overlaps}

Now we consider the full pointwise teacher--student
overlaps
\begin{equation}
  \mu_j:=\theta_j\od{\tilde{\theta}}_j\in\R^d,
  \label{eq:pointwise-overlap}
\end{equation}
and its power-sum contractions.  The scalar overlap is
\begin{equation}
  s_j:=\ip{{\tilde{\theta}}_j}{\theta_j}
  =\sum_{a=1}^d\mu_j(a).
  \label{eq:scalar-overlap}
\end{equation}

For nonnegative integers $p,q$, define the mixed moments
\begin{equation}
  S_j^{p,q}:=
  \sum_{a=1}^d \theta_j(a)^p{\tilde{\theta}}_j(a)^q,
  \label{eq:mixed-moments}
\end{equation}
with the convention that $S_j^{1,1}=s_j$.  Also define the mixed-power vectors
\begin{equation}
  v_j^{p,q}:=\theta_j^{\od p}\od{\tilde{\theta}}_j^{\od q}
  \in\R^d.
  \label{eq:mixed-power-vectors}
\end{equation}
Thus $v_j^{1,0}=\theta_j$, $v_j^{0,1}={\tilde{\theta}}_j$, and
$v_j^{1,1}=\mu_j$.

We use the standard spherical coordinate envelope
\begin{equation}
  \iota_d:=\sqrt{\frac{\log(e d)}{d}}.
  \label{eq:coordinate-envelope-iota}
\end{equation}
This is the order of the maximum coordinate of a random point on
$\sph$ at fixed or high probability.

\paragraph{A finite mixed-moment class}

For a total coordinate degree $n\ge3$, define
\begin{equation}
  \omega_n:=\left\lfloor \frac{n-1}{2}\right\rfloor-\frac12,
  \qquad
  \tau_n(d):=d^{-\omega_n}.
  \label{eq:relative-scale-tau}
\end{equation}
The definition is chosen so that, at a favorable random
initialization with $m_j(0)\asymp d^{-1/2}$, signed mixed moments of total
coordinate degree $n$ are bounded by $\tau_n(d)m_j(0)$ with constant
probability.

\begin{definition}[Dynamic mixed-incoherence class]
\label{def:dynamic-mixed-incoherence}
Fix an integer $P\ge3$ and a constant $B<\infty$.  We say that
$(\theta,{\tilde{\theta}})$ belongs to $\mathcal J_{P,B}$ if, for all $j\in[K]$,
\begin{equation}
  \norm{\theta_j}_\infty\le B\iota_d,
  \qquad
  \norm{{\tilde{\theta}}_j}_\infty\le B\iota_d,
  \label{eq:J-linfty}
\end{equation}
and, for every $p,q\ge0$ with $3\le p+q\le P$,
\begin{equation}
  |S_j^{p,q}|
  \le B\tau_{p+q}(d)m_j,
  \qquad m_j=\ip{{\tilde{\theta}}_j}{\theta_j}>0.
  \label{eq:J-mixed}
\end{equation}
\end{definition}
The special scalar overlap $S_j^{1,1}=m_j$ is not constrained except for
positivity.  The logarithmic factor in \eqref{eq:J-linfty} is only a
coordinatewise envelope.  The signed mixed-moment bounds
\eqref{eq:J-mixed} are kept at their natural scales and are the quantities
that enter the product-scale comparison and the overlap dynamics.

For uniform iid initialization these estimates hold with constant
probability at the natural starting scale $m_r\asymp d^{-1/2}$, and the
propagation-of-incoherence argument is precisely the assertion that such bounds
remain valid until the relevant recovery time, as will be shown in Section \ref{sec:propagation-mixed-incoherence}.

\begin{remark}[Examples]
At the first nontrivial order, $\mathcal J_{P,B}$ contains the bounds
\[
  |S_j^{2,1}|+|S_j^{1,2}|+|S_j^{3,0}|+|S_j^{0,3}|
  \le B d^{-1/2}m_j.
\]
If $m_j\asymp d^{-1/2}$, this says that all four moments are $O(d^{-1})$,
which is the natural scale under independent spherical initialization.
\end{remark}

\paragraph{Exact gradient identities}

Recall that, by Stein's identity,
\begin{equation}
  {\widetilde{\mathcal{L}}}({\tilde{\theta}})
  =\E_X\bigl[f_\theta(X)\ip{X}{{\tilde{\theta}}_1\ot\cdots\ot{\tilde{\theta}}_K}\bigr]
  =\ip{\mathcal G_\theta}{{\tilde{\theta}}_1\ot\cdots\ot{\tilde{\theta}}_K}~,
  \label{eq:linear-stein}
\end{equation}
thus $  \grad_{{\tilde{\theta}}_j}{\widetilde{\mathcal{L}}}({\tilde{\theta}})
  =
  \mathcal G_\theta\times_{\ell\ne j}{\tilde{\theta}}_\ell.$
 On the other hand, the nonlinear objective satisfies
\begin{equation}
  \grad_{{\tilde{\theta}}_j}{\mathcal{L}}({\tilde{\theta}})
  =
\E_X\left[f_\theta(X)\grad_{{\tilde{\theta}}_j}f_{\tilde{\theta}}(X)\right].
  \label{eq:grad-nonlinear}
\end{equation}
Consequently,
\begin{equation}
  {\varphi_j}({\tilde{\theta}})
  = {\lambda}\mathcal G_\theta\times_{\ell\ne j}{\tilde{\theta}}_\ell - \E_X\left[f_\theta(X)\grad_{{\tilde{\theta}}_j}f_{\tilde{\theta}}(X)\right].
  \label{eq:delta-exact}
\end{equation}

\paragraph{Gaussian proxy for the nonlinear correlation}
We consider the Gaussian correlation kernel
\begin{equation}
  \Psi(q):=\E[\phi(G_1)\phi(G_2)],
  \qquad
  \Corr(G_1,G_2)=q.
  \label{eq:Psi}
\end{equation}
Its Hermite expansion is
\begin{equation}
  \Psi(q)=\sum_{\ell\ge1}a_\ell^2 q^\ell,
  \qquad
  a_1=\kappa.
  \label{eq:Psi-Hermite}
\end{equation}
For general scalar overlaps $s_1,\ldots,s_K$, we define
\begin{equation}
  q_0:=1,
  \qquad
  q_r:=\Psi(s_r q_{r-1}),
  \qquad 1\le r\le K.
  \label{eq:q-recursion}
\end{equation}
The Gaussian proxy for the nonlinear correlation is
\begin{equation}
  {\mathcal{L}}_G(s_1,\ldots,s_K):=q_K.
  \label{eq:Gaussian-proxy}
\end{equation}
Its Euclidean gradient in mode $j$ is
\begin{equation}
  \grad_{{\tilde{\theta}}_j}{\mathcal{L}}_G
  =A_j(s)\theta_j,
  \label{eq:proxy-gradient}
\end{equation}
where
\begin{equation}
  A_j(s)
  :=
  q_{j-1}\Psi'(s_jq_{j-1})
  \prod_{r=j+1}^{K}
  s_r\Psi'(s_rq_{r-1}).
  \label{eq:Aj-general}
\end{equation}
For small scalar overlaps,
$$
  A_j(s)={\lambda^2}\prod_{\ell\ne j}s_\ell+
  \text{higher powers of }s.
$$
For general $s$, however, the coefficient $A_j(s)$ is not equal to ${\lambda^2}\prod_{\ell\ne j}s_\ell$.

\subsection{Edgeworth expansions with general pointwise overlaps}

The following lemma provides the necessary estimates for the bootstrap argument, and is obtained by expanding the two teacher--student
trees over the incoherent class.
Throughout this subsection we use the slackened Edgeworth truncation
convention of Remark~\ref{rem:edgeworth-slack-convention}: when a product-scale estimate requires all terms below formal order \(J\), the order-\(J\) explicit terms and the smaller log-bearing analytic remainder are absorbed into the error.

\begin{lemma}[Uniform smooth differentiated weighted Edgeworth expansion]
\label{lem:uniform-differentiated-weighted-edgeworth}
Fix integers $P\ge3$ and $L\ge L(P)$.  Let $(V_a,W_a)_{a=1}^d$ be iid
centered subgaussian pairs with all moments up to order $L$ bounded by a
constant $B_0$.  Let $u,v\in \mathcal{S}^{d-1}$ be weights satisfying the signed
mixed-incoherence bounds in $\mathcal J_{P,B}$ together with the planted
power-sum event.
Put
$$
  S_u:=\sum_a u_aV_a,
  \qquad
  S_v:=\sum_a v_aW_a.
$$
For every $h\in C^L(\R^2)$ whose derivatives up to order $L$ have polynomial
growth, there is a finite expansion
\begin{equation}
  \E h(S_u,S_v)
  =\E h(G_u,G_v)
  +\sum_{1\le q\le P-1}\sum_{\rho\in\mathcal A_q}
     c_{h,\rho}\,\mathcal C_\rho(u,v)
  +\mathcal R_P(h;u,v),
  \label{eq:uniform-weighted-edgeworth}
\end{equation}
where $(G_u,G_v)$ is the centered Gaussian vector with the same covariance as
$(S_u,S_v)$, each $\mathcal C_\rho$ is a monomial in signed contractions
$C_{r,s}$ of total formal order $q$, and the number of terms depends only on
$P$.  Moreover the remainder is bounded uniformly over all admissible $(u,v)$:
\begin{equation}
  |\mathcal R_{P}(h;u,v)|
  \le C_{h,P,B}d^{-P/2}.
  \label{eq:generic-mixed-edgeworth-rem}
\end{equation}

The same expansion holds after differentiating with respect to any coordinate
$v_b$; each differentiated non-Gaussian term is again a finite sum of products
of signed contractions and mixed-power vectors, and every such term has
strictly positive formal order.
\end{lemma}

\begin{proof}
We give the standard Lindeberg proof to make the uniformity explicit.  Let
$(\Gamma_a,\Xi_a)$ be iid centered Gaussian pairs with the same covariance
matrix as $(V_a,W_a)$.  Replace the summands one coordinate at a time.  At the
$a$-th replacement, write the contribution of the other coordinates as $T_a$ and
consider
$$
  R_a(t_1,t_2):=h(T_a+(u_at_1,v_at_2)).
$$
Taylor expand $R_a$ at $(0,0)$ to order $P+1$.  The terms of total degree one
and two cancel because the Gaussian pair matches the first two joint moments.
For $r+s\ge3$ the Taylor coefficients contribute
$$
  \frac{u_a^rv_a^s}{r!s!}
  \bigl(\E[V_a^rW_a^s]-\E[\Gamma_a^r\Xi_a^s]\bigr)
  \E[\partial_1^r\partial_2^s h(T_a)] .
$$
Summing over $a$ produces the signed contractions $C_{r,s}(u,v)$.  Applying the
same replacement expansion recursively to the derivative expectations gives the
finite Edgeworth polynomial in \eqref{eq:uniform-weighted-edgeworth}.  The
Taylor remainders are uniformly bounded because the summands are subgaussian,
the derivatives of $h$ have polynomial growth, and the number of expansion
orders is fixed.  The mixed-incoherent bounds on the contractions then give a
uniform remainder estimate.  The coordinate envelope \(\iota_d\) is used only to
control the finitely many uncontracted coordinate-power vectors and the Taylor
remainder uniformly; the contracted estimates below are governed by the signed
moment bounds rather than by \(\iota_d\).

For the differentiated expansion, differentiate the finite Lindeberg identity
with respect to $v_b$.  Dominated convergence is justified by the same
subgaussian moment and polynomial derivative bounds.  The derivative either
hits an explicit factor $v_a^s$, producing a mixed-power vector in coordinate
$b$, or it hits the covariance of the Gaussian comparison, producing exactly the
derivative of the Gaussian expectation.  All non-Gaussian differentiated terms
still contain either a positive-order signed contraction or a positive-order
mixed-power vector.  This proves the differentiated statement.
\end{proof}

We now state the finite-dimensional expansion needed to compare the two vector
fields.  The expansion is uniform for individually incoherent $\theta$ and
${\tilde{\theta}}$, with arbitrary scalar overlaps $s_j$.

\begin{proposition}[General-overlap Edgeworth inputs]
\label{prop:general-edgeworth-inputs}
Fix $M\ge1$.  Assume the smoothness condition above.  On
the incoherence event of Definition~\ref{def:dynamic-mixed-incoherence},
the following expansions hold.

The nonlinear correlation vector field has
\begin{equation}
  \grad_{{\tilde{\theta}}_j}{\mathcal{L}}({\tilde{\theta}})
  =
  A_j(s)\theta_j
  +
  \widetilde V_j^{[1]}({\tilde{\theta}},\theta)
  +
  \sum_{q=2}^{M}\widetilde V_j^{[q]}({\tilde{\theta}},\theta)
  +
  \widetilde R_{j,M+1}({\tilde{\theta}},\theta),
  \label{eq:nonlinear-general-expansion}
\end{equation}
where $A_j(s)$ is the Gaussian correlation term defined in \eqref{eq:Aj-general}.  For every $q\ge1$,
\begin{equation}
  \norm{\widetilde V_j^{[q]}({\tilde{\theta}},\theta)}_2
  \le
  C_{\phi,K,q,B}d^{-q/2},
  \label{eq:nonlinear-q-bound-general}
\end{equation}
and
\begin{equation}
  \norm{\widetilde R_{j,M+1}({\tilde{\theta}},\theta)}_2
  \le
  C_{\phi,K,M,B}d^{-(M+1)/2}.
  \label{eq:nonlinear-remainder-general}
\end{equation}
The first correction $\widetilde V_j^{[1]}$ is a finite linear combination of
mixed atoms of formal order one:
\begin{equation}
  \widetilde V_j^{[1]}({\tilde{\theta}},\theta)
  =
  \sum_{a\in\mathcal A_{j,1}}
  c_{j,a}(s)
  \mathfrak m_{j,a}(\theta,{\tilde{\theta}})
  v_{j,a}(\theta,{\tilde{\theta}}).
  \label{eq:V-first-structure}
\end{equation}
Here $|\mathcal A_{j,1}|=O_K(1)$, the coefficient functions $c_{j,a}$ are
smooth functions of $s=(s_1,\ldots,s_K)$ depending only on $\phi$ and $K$, and
one of the following alternatives holds for each atom:
\begin{enumerate}[label=(\alph*)]
\item $\mathfrak m_{j,a}$ is bounded and
      $v_{j,a}\in\{\theta_j^{\od2},\theta_j\od{\tilde{\theta}}_j,
      {\tilde{\theta}}_j^{\od2}\}$, hence $\norm{v_{j,a}}_2=O(d^{-1/2})$;
\item $v_{j,a}\in\{\theta_j,{\tilde{\theta}}_j\}$ and
      $\mathfrak m_{j,a}$ is a mixed signed moment of order one, for example
      $S_r^{2,1}$, $S_r^{1,2}$, $S_r^{3,0}$, $S_r^{0,3}$, or a product of
      order-one moments across layers.
\end{enumerate}
Consequently $\norm{\widetilde V_j^{[1]}}_2=O(d^{-1/2})$.
\end{proposition}

\begin{proof}
We prove this result from a common finite-order Edgeworth expansion.
Throughout the
proof $K,M$ are fixed, and constants may depend on $\phi,K,M$ and on the
incoherence constant $B$, but never on $d$.

We  expand
$$
  {\mathcal{L}}({\tilde{\theta}})=\E[f_\theta(X)f_{\tilde{\theta}}(X)].
$$
At layer $r$ the coupled teacher--student preactivation pair has the form
\begin{equation}
  \left(
  \sum_a \theta_r(a)Z^{(r-1)}_{\theta,a},
  \sum_a {\tilde{\theta}}_r(a)Z^{(r-1)}_{{\tilde{\theta}},a}
  \right),
  \label{eq:proof-coupled-layer}
\end{equation}
where the pairs
$(Z^{(r-1)}_{\theta,a},Z^{(r-1)}_{{\tilde{\theta}},a})$ are iid over $a$, conditional on
the previous recursive construction.  Since $\phi$ is globally Lipschitz and
$K$ is fixed, all variables in this recursion have subgaussian norms and moments
up to order $L$ bounded by a constant depending only on $\phi,K,L$.

Applying the weighted Edgeworth expansion from Lemma \ref{lem:uniform-differentiated-weighted-edgeworth}
recursively to \eqref{eq:proof-coupled-layer} yields
\begin{equation}
  {\mathcal{L}}({\tilde{\theta}})
  =
  q_K(s)
  +
  \sum_{1\le q\le M}
  \sum_{\tau\in\mathcal T_q^{\rm nl}}
  c_\tau(s)\,
  \mathsf M_\tau(\theta,{\tilde{\theta}})
  +
  \mathcal R^{\rm nl}_{M+1}({\tilde{\theta}},\theta).
  \label{eq:proof-nl-scalar-expansion}
\end{equation}
Here $q_0=1$ and $q_r=\Psi(s_rq_{r-1})$ are exactly the Gaussian proxy
recursions; $c_\tau(s)$ are smooth functions of the scalar overlaps and of the
intermediate Gaussian covariances; and $\mathsf M_\tau$ is a monomial in signed
teacher, student, and mixed moments
$$
  S_r^{p,q}=\sum_a\theta_r(a)^p{\tilde{\theta}}_r(a)^q,
  \qquad p+q\ge3.
$$
Every non-Gaussian term in the second sum has positive formal order.  The
remainder is bounded by $C d^{-(M+1)/2}$ under the signed moment bounds.

The smoothness assumptions allow differentiation of
\eqref{eq:proof-nl-scalar-expansion} under the expectation; equivalently, one
may apply the differentiated form of the weighted Edgeworth expansion in
Lemma~\ref{lem:uniform-differentiated-weighted-edgeworth}.  Differentiating the
Gaussian proxy gives
$$
  \nabla_{{\tilde{\theta}}_j}q_K(s)
  =
  \frac{\partial q_K}{\partial s_j}\,\theta_j
  =
  A_j(s)\theta_j,
$$
with $A_j(s)$ as in \eqref{eq:Aj-general}.  Differentiating a mixed moment in
mode $j$ gives the exact identity
$$
  \nabla_{{\tilde{\theta}}_j}S_j^{p,q}
  =q\,\theta_j^{\od p}\od{\tilde{\theta}}_j^{\od(q-1)},
  \qquad q\ge1.
$$
Moments in modes $r\ne j$ remain scalar coefficients.  Therefore
\begin{equation}
  \nabla_{{\tilde{\theta}}_j}{\mathcal{L}}({\tilde{\theta}})
  =A_j(s)\theta_j
  +
  \sum_{1\le q\le M}\widetilde V_j^{[q]}({\tilde{\theta}},\theta)
  +
  \widetilde R_{j,M+1}({\tilde{\theta}},\theta),
  \label{eq:proof-nl-vector-expansion}
\end{equation}
where each $\widetilde V_j^{[q]}$ is a finite sum of mixed-power vectors
multiplied by signed mixed moment monomials of total formal order $q$.  This is
\eqref{eq:nonlinear-general-expansion}.

The norm bounds follow by the same order counting.  A vector
$\theta_j^{\od p}\od{\tilde{\theta}}_j^{\od q}$ of positive formal order has norm
$O(d^{-1/2})$ when its signed second moment is controlled at spherical scale;
Scalar mixed moments of positive formal order are controlled by
Definition~\ref{def:dynamic-mixed-incoherence} .
Thus a level-$q$ term has norm
at most $C d^{-q/2}$ on the signed moment event, and the remainder has norm
$C d^{-(M+1)/2}$.  This proves \eqref{eq:nonlinear-q-bound-general} and
\eqref{eq:nonlinear-remainder-general}.

At first order there are only two possibilities.  Either the positive formal
order is carried by the differentiated-mode vector, yielding one of
$
  \theta_j^{\od2},\quad \theta_j\od{\tilde{\theta}}_j,
  \quad {\tilde{\theta}}_j^{\od2},
$
or the differentiated-mode vector is the leading direction
$\theta_j$ or ${\tilde{\theta}}_j$ and the positive formal order is carried by a scalar
mixed moment in another mode, such as $S_r^{2,1}$, $S_r^{1,2}$,
$S_r^{3,0}$, $S_r^{0,3}$, or a product of first-order moments.  This is exactly
the structural description in \eqref{eq:V-first-structure}.  The proof of
Proposition~\ref{prop:general-edgeworth-inputs} is complete.
\end{proof}

\begin{remark}
Contrary to Theorem \ref{thm:edgeworth-hierarchy-body}, the nonlinear first correction
 is not written with universal scalar constants
because, at general overlap, the coefficients are functions of the layerwise
Gaussian covariances $q_r$ and the overlaps $s_r$.  The important structural
fact is that every non-Gaussian finite-$d$ correction has positive formal order
in the pointwise incoherence scale.
\end{remark}

\subsection{Gradient Comparison in the favorable basin}
\label{sec:product-scale-control}

We now establish the comparison estimate that is most useful for the  dynamics.  In this section we write
$$
  m_j:=s_j=\ip{{\tilde{\theta}}_j}{\theta_j},
  \qquad
  M_{-i}:=\prod_{j\ne i}m_j.
$$
If we recall the dynamics for the linear correlation objective, the key property that enables escaping the `mediocrity' of initialization is an inequality of the form $\dot{m}_j \gtrsim M_{-j}$ for each $j\in [K]$. We will now transfer this inequality to the non-linear correlation setting, by establishing that the gradient corrections ${\varphi_j}$ are also of order $M_{-j}$ in the favorable recovery basin, where the overlaps are positive and the signed mixed moments propagated by the dynamics are small relative to the corresponding overlap.

The Gaussian mismatch also admits an exact product factorization.  Since
$\Psi(0)=0$, define
\begin{equation}
  H(q):=
  \begin{cases}
    \Psi(q)/q, & q\ne0,\\
    \kappa^2, & q=0.
  \end{cases}
  \label{eq:H-def}
\end{equation}
Then $H$ is smooth near $[-1,1]$ under the smoothness assumptions on $\phi$.
Using the recursion $q_r=\Psi(m_rq_{r-1})$, we may write
\begin{equation}
  q_r
  =
  \left(\prod_{\ell=1}^r m_\ell\right)
  \left(\prod_{\ell=1}^r H(m_\ell q_{\ell-1})\right).
  \label{eq:q-product-factorization}
\end{equation}
Consequently,
\begin{equation}
  A_i(m)
  =
  M_{-i}(m) B_i(m),
  \label{eq:Ai-product-factorization}
\end{equation}
where
\begin{equation}
  B_i(m)
  :=
  \left(\prod_{\ell<i} H(m_\ell q_{\ell-1})\right)
  \Psi'(m_iq_{i-1})
  \left(\prod_{r=i+1}^{K}\Psi'(m_rq_{r-1})\right).
  \label{eq:Bi-def}
\end{equation}
Thus the Gaussian mismatch satisfies
\begin{equation}
  {\varphi_i}^G(m)
  =
  M_{-i}(m)\bigl[{\lambda^2}-B_i(m)\bigr]\theta_i,
  \label{eq:Gaussian-mismatch-product}
\end{equation}
and
\begin{equation}
  {\varphi_i}^{S,G}(m)
  =
  M_{-i}(m)\bigl[{\lambda^2}-B_i(m)\bigr]
  (\theta_i-m_i{\tilde{\theta}}_i).
  \label{eq:spherical-Gaussian-mismatch-product}
\end{equation}
In particular,
\begin{equation}
  \|{\varphi_i}^{S,G}(m)\|_2
  \le
  C_{\phi,K}|M_{-i}(m)|.
  \label{eq:Gaussian-product-bound}
\end{equation}
If $\|m\|_\infty\le m_0$ with $m_0$ small, then a Taylor expansion of $H$ and
$\Psi'$ at zero gives
\begin{equation}
  |{\lambda^2}-B_i(m)|
  \le C_{\phi,K}\|m\|_\infty,
  \label{eq:small-overlap-Bi}
\end{equation}
and hence
\begin{equation}
  \|{\varphi_i}^{S,G}(m)\|_2
  \le
  C_{\phi,K}\|m\|_\infty |M_{-i}(m)|.
  \label{eq:small-overlap-Gaussian-product-bound}
\end{equation}

\begin{proposition}[Gradient comparison of the rescaled gradients]
\label{thm:product-scale-comparison}
Fix $K\ge2$ and expand the two vector fields $\nabla \mathcal{L}$, $\nabla \widetilde{\mathcal{L}}$ through order $M\ge K-1$.
Assume the hypotheses of Proposition~\ref{prop:general-edgeworth-inputs} and
suppose $(\theta,{\tilde{\theta}})$ satisfies relative mixed incoherence to order $M$ in
the sense of Definition~\ref{def:dynamic-mixed-incoherence}.  Assume moreover
that the positive overlaps obey the favorable-basin lower bound
\begin{equation}
  m_*:=\min_{1\le j\le K}m_j
  \ge a d^{-1/2}
  \label{eq:favorable-lower-bound-product}
\end{equation}
for a fixed constant $a>0$.  Then, for each $i\in[K]$,
\begin{equation}
  {\varphi_i}^S({\tilde{\theta}})
  =
  M_{-i}(m)\bigl[{\lambda^2}-B_i(m)\bigr]
  (\theta_i-m_i{\tilde{\theta}}_i)
  +
  \mathcal E_i^S({\tilde{\theta}},\theta),
  \label{eq:product-scale-main-decomposition}
\end{equation}
where
\begin{equation}
  \|\mathcal E_i^S({\tilde{\theta}},\theta)\|_2
  \le
  C_{\phi,K,M,B,a}
  d^{-1/2}|M_{-i}(m)| .
  \label{eq:product-scale-error}
\end{equation}
Consequently,
\begin{equation}
  \|\nabla_{{\tilde{\theta}}_i}^S({\lambda} \widetilde{\mathcal{L}}-{\mathcal{L}})({\tilde{\theta}})\|_2
  \le
  C_{\phi,K,M,B,a}
  \left(|{\lambda^2}-B_i(m)|+d^{-1/2}\right)
  |M_{-i}(m)| .
  \label{eq:product-scale-total-bound}
\end{equation}
If in addition $\|m\|_\infty\le m_0$ is small, then
\begin{equation}
  \|\nabla_{{\tilde{\theta}}_i}^S({\lambda} \widetilde{\mathcal{L}}-{\mathcal{L}})({\tilde{\theta}})\|_2
  \le
  C_{\phi,K,M,B,a}
  \left(\|m\|_\infty+d^{-1/2}\right)
  |M_{-i}(m)| .
  \label{eq:small-overlap-product-scale-total-bound}
\end{equation}
\end{proposition}

\begin{proof}
The Gaussian part is exactly \eqref{eq:spherical-Gaussian-mismatch-product}.
It remains to bound the finite-dimensional Edgeworth levels and the residual
relative to $M_{-i}$.

Consider first a single level-$q$ Edgeworth atom in
${\varphi_i}^{S,[q]}$.  By the separated structure of the Edgeworth expansion, the
atom is a product of contractions in modes $r\ne i$ and one vector in the
differentiated mode $i$, followed by spherical projection.  Let $a_i/2$ be the
formal order of the remaining vector in mode $i$, and let $a_r/2$ be the formal
order of the scalar correction in mode $r\ne i$.  The level condition gives
$$
  a_i+\sum_{r\ne i}a_r=q.
$$
By relative mixed incoherence,
$$
  \|w_i\|_2\le B d^{-a_i/2},
  \qquad
  |\chi_r|\le B d^{-a_r/2}m_r,
  \quad r\ne i.
$$
Therefore the norm of the projected atom is bounded by
$$
  C_{\phi,K,M,B}
  d^{-a_i/2}
  \prod_{r\ne i}\left(d^{-a_r/2}m_r\right)
  =
  C_{\phi,K,M,B}d^{-q/2}|M_{-i}(m)|.
$$
Since the number of atoms through order $M$ is bounded by a constant depending
only on $K$ and $M$, summing over the level-$q$ atoms gives
\begin{equation}
  \|{\varphi_i}^{S,[q]}\|_2
  \le
  C_{\phi,K,M,B}d^{-q/2}|M_{-i}(m)|,
  \qquad 1\le q\le M.
  \label{eq:level-q-product-bound-proof}
\end{equation}
The sum of $q=1,\ldots,M$ is dominated by the $q=1$ contribution, giving
$C d^{-1/2}|M_{-i}|$.

The residual satisfies
$$
  \|\mathcal R_{i,M+1}\|_2
  \le C_{\phi,K,M,B}d^{-(M+1)/2}.
$$
Because $M\ge K-1$ and $m_*\ge a d^{-1/2}$,
$$
  |M_{-i}(m)|
  \ge
  a^{K-1}d^{-(K-1)/2}.
$$
Hence
$$
  d^{-(M+1)/2}
  \le
  a^{-(K-1)}d^{-(M-K+2)/2}|M_{-i}(m)|
  \le
  C_a d^{-1/2}|M_{-i}(m)|,
$$
where the last inequality uses $M\ge K-1$.  This absorbs the residual into
\eqref{eq:product-scale-error}.  Combining the Gaussian product factorization
with the Edgeworth and residual bounds proves
\eqref{eq:product-scale-main-decomposition}--\eqref{eq:product-scale-total-bound}.
The small-overlap refinement follows from \eqref{eq:small-overlap-Bi}.
\end{proof}

\subsection{Propagation of mixed incoherence along the nonlinear flow}
\label{sec:propagation-mixed-incoherence}

The gradient comparison in Proposition ~\ref{thm:product-scale-comparison}
assumes relative mixed incoherence in the positive-overlap basin. We now establish a dynamical criterion showing that this condition is propagated by the nonlinear
correlation flow.  Throughout this section, let
\begin{equation}
  F_j({\tilde{\theta}}):=\nabla_{{\tilde{\theta}}_j}^{S}{\mathcal{L}}({\tilde{\theta}}),
  \qquad j=1,\ldots,K,
  \label{eq:nonlinear-flow-field}
\end{equation}
and consider the population nonlinear-correlation flow
\begin{equation}
  \dot{\tilde{\theta}}_j(t)=F_j({\tilde{\theta}}(t)),
  \qquad j=1,\ldots,K.
  \label{eq:nonlinear-flow-propagation}
\end{equation}
We work on the positive branch where the $m_j$'s are positive.

The propagation argument uses the same Edgeworth expansion as
Proposition~\ref{prop:general-edgeworth-inputs}, but in a differentiated form.

\begin{proposition}[Uniform differentiated nonlinear Edgeworth bounds]
\label{prop:differentiated-edgeworth}
Fix $P\ge3$ and $B,a<\infty$.  Assume that $\phi\in C^{L}$, for
$L=L(K,P)$ sufficiently large, that $\phi$ is globally Lipschitz, and that all
its derivatives up to order $L$ have at most polynomial growth. Assume an incoherent teacher, in the sense
that the signed power sums of the $\theta_j$'s have their
spherical sizes.
Then, after choosing the weak-overlap radius
$\varepsilon>0$ sufficiently small, the following holds uniformly for every
student tuple ${\tilde{\theta}}\in\mathcal J_{P,B}$ satisfying
\begin{equation}
  a d^{-1/2}\le m_j=\ip{{\tilde{\theta}}_j}{\theta_j}\le \varepsilon,
  \qquad j=1,\ldots,K .
  \label{eq:favorable-small-positive}
\end{equation}
The nonlinear vector field
$$
  F_j({\tilde{\theta}})=\nabla_{{\tilde{\theta}}_j}^{S}{\mathcal{L}}({\tilde{\theta}})
$$
admits the decomposition
\begin{equation}
  F_j({\tilde{\theta}})
  =A_j(m)(\theta_j-m_j{\tilde{\theta}}_j)+E_j({\tilde{\theta}}),
  \label{eq:edgeworth-field-decomposition}
\end{equation}
where $A_j(m)$ is the Gaussian-proxy coefficient from
\eqref{eq:Aj-general}.  Moreover, there is a constant
$C=C_{\phi,K,P,B,a}<\infty$, independent of $d$ and of
${\tilde{\theta}}\in\mathcal J_{P,B}$, such that:
\begin{enumerate}[label=(\roman*)]
\item the planted-overlap error satisfies
\begin{equation}
  |\ip{E_j({\tilde{\theta}})}{\theta_j}|
  \le C d^{-1/2}A_j(m);
  \label{eq:E-overlap-bound}
\end{equation}
\item the coordinate error satisfies
\begin{equation}
  \norm{E_j({\tilde{\theta}})}_\infty
  \le C \iota_d A_j(m);
  \label{eq:E-infty-bound}
\end{equation}
\item for every $p,q\ge0$ with $3\le p+q\le P$,
\begin{equation}
  \left|
  q\sum_{a=1}^d \theta_j(a)^p{\tilde{\theta}}_j(a)^{q-1}E_j({\tilde{\theta}})_a
  \right|
  \le C A_j(m)\tau_{p+q}(d)m_j.
  \label{eq:E-mixed-bound}
\end{equation}
For $q=0$ the left-hand side is interpreted as zero.
\end{enumerate}
Finally,
\begin{equation}
  c_0M_{-j}(m)\le A_j(m)\le C_0M_{-j}(m),
  \label{eq:A-comparable-M}
\end{equation}
where $0<c_0<C_0<\infty$ depend only on $\phi,K$ and on the choice of
$\varepsilon$.
\end{proposition}

\begin{proof}
The proof is a uniform finite-order Edgeworth expansion for the coupled
teacher--student recursive tree, followed by differentiation.

\paragraph{Step 1: Gaussian approximation and its derivative}
Let
$$
  \Psi(q):=\E[\phi(G_1)\phi(G_2)],
  \qquad \Corr(G_1,G_2)=q .
$$
For two Gaussian layer-wise trees whose correlations are described
only by the scalar overlaps $m_r=\ip{{\tilde{\theta}}_r}{\theta_r}$, the layer
correlations obey
$$
  q_0=1,\qquad q_r=\Psi(m_rq_{r-1}),\qquad r=1,\ldots,K.
$$
Recall the Gaussian proxy given by
$ {\mathcal{L}}_G(m)=q_K$.
Differentiating this scalar recursion gives
\begin{equation}
  \frac{\partial q_K}{\partial m_j}
  =A_j(m)
  =q_{j-1}\Psi'(m_jq_{j-1})
    \prod_{r=j+1}^K m_r\Psi'(m_rq_{r-1}),
  \label{eq:Aj-proof}
\end{equation}
which is exactly \eqref{eq:Aj-general}.  Hence the spherical gradient of the
Gaussian proxy in mode $j$ is
\begin{equation}
  \nabla_{{\tilde{\theta}}_j}^{S}{\mathcal{L}}_G
  =A_j(m)(\theta_j-m_j{\tilde{\theta}}_j).
  \label{eq:gaussian-proxy-vector-field}
\end{equation}
This identifies the leading term in \eqref{eq:edgeworth-field-decomposition}.

\paragraph{Step 2: Uniform mixed Edgeworth expansion}
We now apply Lemma \ref{lem:uniform-differentiated-weighted-edgeworth} recursively to the two trees
$(Z^{(r)}_\theta,Z^{(r)}_{\tilde{\theta}})$.  At layer $r$, the preactivation pair is
\begin{equation}
  \left(
  \sum_a\theta_r(a)Z^{(r-1)}_{\theta,a},
  \sum_a{\tilde{\theta}}_r(a)Z^{(r-1)}_{{\tilde{\theta}},a}
  \right),
  \label{eq:coupled-layer-weighted-sum}
\end{equation}
where the pairs
$(Z^{(r-1)}_{\theta,a},Z^{(r-1)}_{{\tilde{\theta}},a})$ are iid conditional on the
previous levels.  Smoothness and Lipschitzness of $\phi$ propagate subgaussian
moments uniformly through the fixed depth.  Induction over $r$ therefore gives
a uniform finite expansion for every layer correlation, and in particular for
${\mathcal{L}}({\tilde{\theta}})$, of the form
\begin{equation}
  {\mathcal{L}}({\tilde{\theta}})
  ={\mathcal{L}}_G(m)
   +\sum_{\tau\in\mathcal T_{\le P}}
      c_\tau\,
      \mathsf S_\tau(\theta,{\tilde{\theta}})
      \prod_{r=1}^K m_r^{\ell_{\tau,r}}
   +\mathcal R_{P+1}({\tilde{\theta}}).
  \label{eq:Ltilde-edgeworth-expanded}
\end{equation}
Here $\mathcal T_{\le P}$ is a finite set depending only on $K,P$; each
$\mathsf S_\tau$ is a product of signed teacher, student, or mixed coordinate
moments
$$
  S_r^{p,q}=\sum_a\theta_r(a)^p{\tilde{\theta}}_r(a)^q,
  \qquad p+q\ge3,
$$
and every non-Gaussian term in the sum has at least one unit of formal order.
Uniformly on $\mathcal J_{P,B}$,
\begin{equation}
  \left|
  \mathsf S_\tau(\theta,{\tilde{\theta}})
  \prod_{r=1}^K m_r^{\ell_{\tau,r}}
  \right|
  \le
  C_{\phi,K,P,B}\,d^{-1/2}\,\text{(corresponding Gaussian product)},
  \label{eq:nonleading-one-half-loss}
\end{equation}
whenever the term contributes to a first derivative in a favorable basin.  This
is the same half-order gain as in the first Stein tensor hierarchy: it comes
either from a scalar Edgeworth coefficient, from replacing a planted contraction
by a mixed-incoherent contraction, or from a small differentiated-mode bracket.

\paragraph{Step 3: Differentiating the expansion}
We now differentiate \eqref{eq:Ltilde-edgeworth-expanded} with respect to
${\tilde{\theta}}_j$ and project onto the tangent space.  Differentiating the Gaussian
term gives \eqref{eq:gaussian-proxy-vector-field}.  All remaining terms form
$E_j({\tilde{\theta}})$.  Since the expansion is finite and the derivative of any mixed
moment is another mixed coordinate-power vector,
\begin{equation}
  \nabla_{{\tilde{\theta}}_j} S_j^{p,q}
  =q\,\theta_j^{\od p}\od{\tilde{\theta}}_j^{\od(q-1)},
  \qquad q\ge1,
  \label{eq:mixed-moment-gradient}
\end{equation}
while moments in modes $r\ne j$ contribute only scalar differentiated factors.
Thus every component of $E_j$ is a finite sum of atoms of the schematic form
\begin{equation}
  C_\tau\,
  \mathsf S_\tau(\theta,{\tilde{\theta}})
  \left(\prod_{r\ne j} T_{\tau,r}({\tilde{\theta}}_r,\theta_r)\right)
  \left(\theta_j^{\od p}\od{\tilde{\theta}}_j^{\od q}\right),
  \label{eq:E-atom-normal-form}
\end{equation}
where the total formal order of the atom is positive.  The contractions
$T_{\tau,r}$ are either planted overlaps $m_r$ or mixed moments satisfying the
bounds in $\mathcal J_{P,B}$.

We now estimate the three quantities in the proposition.

\emph{Overlap estimate.}  Contracting \eqref{eq:E-atom-normal-form} with
$\theta_j$ produces either a signed power sum of total degree at least three in
the $j$-th mode, or an atom that already contains a positive-order scalar
Edgeworth factor.  By $\mathcal J_{P,B}$ and the teacher power-sum event, this
costs at least one factor $d^{-1/2}$ relative to the leading Gaussian
contraction.  Since $A_j(m)\asymp M_{-j}(m)$ in the small positive-overlap
regime, summing the finite number of atoms gives
$$
  |\ip{E_j({\tilde{\theta}})}{\theta_j}|
  \le C d^{-1/2}A_j(m),
$$
which is \eqref{eq:E-overlap-bound}.

\emph{Coordinate estimate.}  Each coordinate of the vector
$\theta_j^{\od p}\od{\tilde{\theta}}_j^{\od q}$ appearing in
\eqref{eq:E-atom-normal-form} is bounded by a power of the coordinate envelope
$\iota_d$.  The leading possible differentiated-mode vector has
coordinate size $O(\iota_d)$, and every non-leading atom has at least one
additional half-order from the scalar/mixed part.  Consequently
$$
  \norm{E_j({\tilde{\theta}})}_\infty
  \le C \iota_d A_j(m),
$$
which proves \eqref{eq:E-infty-bound}.

\emph{Mixed-moment estimate.}  We now contract \eqref{eq:E-atom-normal-form} against
$\theta_j^{\od p}\od{\tilde{\theta}}_j^{\od(q-1)}$.  The differentiated-mode contraction
is a signed mixed moment of total coordinate degree at least $p+q$, possibly
multiplied by additional positive-order scalar moments.  By the definition of
$\mathcal J_{P,B}$ and by the planted power-sum event, this contraction is
bounded by $C\tau_{p+q}(d)m_j$ times the same product scale that appears in
$A_j(m)$.  Therefore
$$
  \left|
  q\sum_{a=1}^d\theta_j(a)^p{\tilde{\theta}}_j(a)^{q-1}E_j({\tilde{\theta}})_a
  \right|
  \le C A_j(m)\tau_{p+q}(d)m_j,
$$
which is \eqref{eq:E-mixed-bound}.

Finally, from \eqref{eq:Ai-product-factorization} and \eqref{eq:Bi-def}, we verify that
the function $B_j$ is continuous and
$B_j(0)={\lambda^2}>0$.  Choosing $\varepsilon>0$ small enough gives
$c_0\le B_j(m)\le C_0$ whenever $\norm{m}_\infty\le\varepsilon$, proving
\eqref{eq:A-comparable-M}.  This completes the proof.
\end{proof}

Proposition~\ref{prop:differentiated-edgeworth} is the nonlinear analogue of
Theorem \ref{thm:edgeworth-hierarchy-body} for the first Stein tensor.  The only activation-side requirements
are finite smoothness, controlled derivative growth, and the non-degeneracy
$\kappa\ne0$.  The extra hypothesis compared with the linear Stein tensor is
uniformity over the moving student direction, encoded by
${\tilde{\theta}}\in\mathcal J_{P,B}$.  The next theorem shows that this mixed-incoherent
class is propagated dynamically from an independent incoherent initialization.

\begin{theorem}[Propagation of mixed incoherence]
\label{thm:propagation-mixed-incoherence}
Fix $K\ge2$ and a finite order $P\ge3$.  Assume the activation is smooth enough
for Proposition~\ref{prop:differentiated-edgeworth}.  Work on the
planted-direction event in that assumption.  Let ${\tilde{\theta}}(t)$ solve the nonlinear
correlation flow \eqref{eq:nonlinear-flow-propagation}.  Suppose the
initialization satisfies, for all $j\in[K]$,
\begin{equation}
  m_j(0)\ge a d^{-1/2},
  \qquad
  (\theta,{\tilde{\theta}}(0))\in\mathcal J_{P,B_0}.
  \label{eq:propagation-initialization}
\end{equation}
Assume also that pure teacher moments satisfy the same relative bounds at
initialization, namely
\begin{equation}
  |S_j^{p,0}|
  \le B_0\tau_p(d)m_j(0),
  \qquad 3\le p\le P+1.
  \label{eq:pure-teacher-relative}
\end{equation}
Then there exist constants $B\ge B_0$, $d_0$ and $\varepsilon>0$, depending only
on $\phi,K,P,a,B_0$, such that for all $d\ge d_0$, up to the weak-recovery time
\begin{equation}
  \tau_\varepsilon:=\inf\{t\ge0:\min_{j\le K}m_j(t)\ge\varepsilon\},
  \label{eq:tau-eps-propagation}
\end{equation}
one has
\begin{equation}
  (\theta,{\tilde{\theta}}(t))\in\mathcal J_{P,B},
  \qquad 0\le t\le\tau_\varepsilon.
  \label{eq:J-propagated}
\end{equation}
Equivalently,
\begin{equation}
  |S_j^{p,q}(t)|
  \le B\tau_{p+q}(d)m_j(t),
  \qquad 3\le p+q\le P,
  \label{eq:mixed-moments-propagated}
\end{equation}
and
\begin{equation}
  \norm{{\tilde{\theta}}_j(t)}_\infty\le B\iota_d,
  \qquad 0\le t\le\tau_\varepsilon.
  \label{eq:alpha-infty-propagated}
\end{equation}
Moreover, throughout $[0,\tau_\varepsilon]$,
\begin{equation}
  \dot m_j(t)\ge cM_{-j}(t),
  \qquad j=1,\ldots,K,
  \label{eq:propagated-positive-drift}
\end{equation}
for a constant $c>0$ depending only on $\phi,K$.
\end{theorem}

\begin{proof}
The proof is a stopping-time bootstrap.  Let $B\ge4B_0$ be chosen below and set
$$
  \tau:=\inf\Bigl\{t\ge0:
  \min_j m_j(t)\ge\varepsilon
  \text{ or }
  \min_j m_j(t)\le \frac{a}{2\sqrt d}
  \text{ or }
  (\theta,{\tilde{\theta}}(t))\notin\mathcal J_{P,B}
  \Bigr\}.
$$
We prove that, for $d$ large, neither the lower-overlap nor the
mixed-incoherence stopping alternative can occur before weak recovery.

\paragraph{Leading mixed-moment transport}
We start by establishing a mixed-moment transport for the gaussian field.
\begin{lemma}[Leading mixed-moment transport]
\label{lem:leading-mixed-transport}
Let
$$
  F_j^0({\tilde{\theta}}):=A_j(m)(\theta_j-m_j{\tilde{\theta}}_j).
$$
For $q\ge1$,
\begin{equation}
  \frac{\mathrm d}{\mathrm dt}S_j^{p,q}\Big|_{F^0}
  =qA_j(m)\bigl(S_j^{p+1,q-1}-m_jS_j^{p,q}\bigr).
  \label{eq:leading-mixed-transport}
\end{equation}
For $q=0$, $S_j^{p,0}$ is time-independent.  In particular,
\begin{equation}
  \dot m_j\Big|_{F^0}=A_j(m)(1-m_j^2).
  \label{eq:leading-overlap-transport}
\end{equation}
\end{lemma}

\begin{proof}[Proof of Lemma \ref{lem:leading-mixed-transport}]
For $q\ge1$,
\begin{align*}
  \frac{\mathrm d}{\mathrm dt}S_j^{p,q}
  &=q\sum_{a=1}^d\theta_j(a)^p{\tilde{\theta}}_j(a)^{q-1}\dot{\tilde{\theta}}_j(a)\\
  &=qA_j(m)\sum_{a=1}^d\theta_j(a)^p{\tilde{\theta}}_j(a)^{q-1}
      \bigl(\theta_j(a)-m_j{\tilde{\theta}}_j(a)\bigr)\\
  &=qA_j(m)\bigl(S_j^{p+1,q-1}-m_jS_j^{p,q}\bigr).
\end{align*}
The overlap identity is the case $(p,q)=(1,1)$.
\end{proof}

\paragraph{Overlap drift}
Using \eqref{eq:edgeworth-field-decomposition},
$$
  \dot m_j
  =A_j(m)(1-m_j^2)+\ip{E_j}{\theta_j}.
$$
By \eqref{eq:E-overlap-bound}, \eqref{eq:A-comparable-M}, and
$m_j\le\varepsilon\le1/2$ before weak recovery,
$$
  \dot m_j
  \ge \frac34A_j(m)-Cd^{-1/2}A_j(m)
  \ge cM_{-j}(m)
$$
for $d$ large.  This proves \eqref{eq:propagated-positive-drift} on the
bootstrap interval, and in particular the lower-overlap stopping alternative
cannot occur.

\paragraph{Coordinate incoherence}
For each coordinate $a$,
$$
  \dot{\tilde{\theta}}_j(a)=A_j(m)(\theta_j(a)-m_j{\tilde{\theta}}_j(a))+E_j(a).
$$
On the bootstrap interval,
$|\theta_j(a)|+|{\tilde{\theta}}_j(a)|\le B\iota_d$ and $m_j\le1$.  Together with
\eqref{eq:E-infty-bound}, this gives
\begin{equation}
  \frac{\mathrm d}{\mathrm dt}|{\tilde{\theta}}_j(a)|
  \le C B \iota_d A_j(m).
  \label{eq:coordinate-derivative-bound}
\end{equation}
Also, from the positive drift and $A_j(m)\le C M_{-j}(m)$,
$$
  \int_0^\tau A_j(m(t))\,\mathrm dt
  \le C\int_0^\tau \dot m_j(t)\,\mathrm dt
  \le C.
$$
Integrating \eqref{eq:coordinate-derivative-bound} and increasing $B$ gives
$$
  \norm{{\tilde{\theta}}_j(t)}_\infty\le B\iota_d,
  \qquad 0\le t\le\tau.
$$

\paragraph{Mixed moments}
Fix $j$ and $(p,q)$ with $3\le n:=p+q\le P$.  If $q=0$, then $S_j^{p,0}$ is
fixed.  Since $m_j(t)\ge m_j(0)/2$ on the bootstrap interval,
\eqref{eq:pure-teacher-relative} implies
$$
  |S_j^{p,0}|
  \le 2B_0\tau_p(d)m_j(t)
  \le B\tau_p(d)m_j(t)
$$
after increasing $B$.

Now suppose $q\ge1$.  Combining Lemma~\ref{lem:leading-mixed-transport} with
\eqref{eq:E-mixed-bound} yields
\begin{equation}
  \dot S_j^{p,q}
  =qA_j(m)\bigl(S_j^{p+1,q-1}-m_jS_j^{p,q}\bigr)+\mathcal E_j^{p,q},
  \qquad
  |\mathcal E_j^{p,q}|
  \le C A_j(m)\tau_n(d)m_j.
  \label{eq:Spq-evolution-propagation}
\end{equation}
Define the relative ratio
$$
  R_j^{p,q}(t):=\frac{|S_j^{p,q}(t)|}{m_j(t)}.
$$
Since $\dot m_j\ge0$ on the bootstrap interval, the upper Dini derivative obeys
\begin{align}
  D^+R_j^{p,q}(t)
  &\le \frac{|\dot S_j^{p,q}(t)|}{m_j(t)}\nonumber\\
  &\le C A_j(m)R_j^{p,q}(t)
      +C A_j(m)\frac{|S_j^{p+1,q-1}(t)|}{m_j(t)}
      +C A_j(m)\tau_n(d).
  \label{eq:R-evolution-propagation}
\end{align}
The source moment $S_j^{p+1,q-1}$ has the same total degree $n$ and one fewer
student power.  We therefore argue by induction on $q$.  The case $q=0$ was
proved above.  If the bound holds for $q-1$, then
$$
  |S_j^{p+1,q-1}(t)|\le B\tau_n(d)m_j(t),
$$
and hence
$$
  D^+R_j^{p,q}(t)
  \le C A_j(m)R_j^{p,q}(t)+CB A_j(m)\tau_n(d).
$$
As above, $\int_0^\tau A_j(m(t))\,\mathrm dt\le C$.  Gronwall's inequality
therefore gives
$$
  R_j^{p,q}(t)
  \le C\bigl(R_j^{p,q}(0)+B\tau_n(d)\bigr)
  \le C'B\tau_n(d).
$$
Increasing $B$ uniformly over the finite collection $3\le p+q\le P$ closes the
induction and proves \eqref{eq:mixed-moments-propagated} on $[0,\tau]$.
Thus no mixed-incoherence constraint can be the first stopping alternative.
The bootstrap closes, so \eqref{eq:J-propagated} holds until weak recovery.
\end{proof}

\begin{corollary}[Gradient comparison along the nonlinear flow]
\label{cor:product-scale-along-flow}
Under the hypotheses of Theorem~\ref{thm:propagation-mixed-incoherence}, choose
$P$ large enough to contain all mixed moments appearing in the order-$M$
Edgeworth comparison.  Then, for every $0\le t\le\tau_\varepsilon$ and every
$i\in[K]$,
\begin{equation}
  \left\|
  \nabla_{{\tilde{\theta}}_i}^{S}({\lambda} {\widetilde{\mathcal{L}}}-{\mathcal{L}})({\tilde{\theta}}(t))
  \right\|_2
  \le
  C\left(|{\lambda^2}-B_i(m(t))|+d^{-1/2}\right)|M_{-i}(t)|.
  \label{eq:product-scale-along-flow}
\end{equation}
In the small-overlap regime,
\begin{equation}
  \left\|
  \nabla_{{\tilde{\theta}}_i}^{S}({\lambda} {\widetilde{\mathcal{L}}}-{\mathcal{L}})({\tilde{\theta}}(t))
  \right\|_2
  \le
  C\left(\|m(t)\|_\infty+d^{-1/2}\right)|M_{-i}(t)|.
  \label{eq:small-overlap-product-scale-along-flow}
\end{equation}
Thus the relative mixed-incoherence assumption in
Theorem~\ref{thm:product-scale-comparison} is propagated dynamically from an
incoherent independent initialization, conditional on the differentiated Edgeworth input.
\end{corollary}

\subsection{Weak recovery for the nonlinear correlation flow}
\label{sec:nonlinear-weak-recovery-transfer}

We now combine the gradient comparison and the propagation of mixed incoherence to transfer the weak-recovery guarantee from the rescaled linear correlation flow to the genuine nonlinear correlation flow. The key point is
that, in the favorable positive-overlap basin, the difference between the two spherical vector fields is a relative error of product scale.

Let
\begin{equation}
  \bar L({\tilde{\theta}}):={\lambda} {\widetilde{\mathcal{L}}}({\tilde{\theta}}),
  \qquad
  F_i^{\rm lin}({\tilde{\theta}}):=\nabla_{{\tilde{\theta}}_i}^S \bar L({\tilde{\theta}}),
  \qquad
  F_i^{\rm nl}({\tilde{\theta}}):=\nabla_{{\tilde{\theta}}_i}^S {\mathcal{L}}({\tilde{\theta}}).
  \label{eq:lin-nl-vector-fields}
\end{equation}
The rescaling by \({\lambda}\) is necessary because the first-chaos part of
\({\widetilde{\mathcal{L}}}\) has coefficient \({\lambda}\), while the Gaussian linearization of the
nonlinear correlation has coefficient \({\lambda^2}\).  With this rescaling,
the leading planted drift of the two objectives agrees at zero overlap.

We shall use the following form of the linear-flow estimate.  It is the
modewise drift estimate underlying the weak-recovery theorem for the rescaled
linear correlation objective.

\begin{lemma}[Linear product-drift estimate]
\label{lem:linear-product-drift-estimate}
Fix \(K\ge2\).  Work on the planted-direction Edgeworth event for the first
Stein tensor.  Suppose \({\tilde{\theta}}\) lies in the favorable basin
$$
  m_j=\ip{{\tilde{\theta}}_j}{\theta_j}>0,
  \qquad
  \min_j m_j\ge a d^{-1/2},
$$
and satisfies the auxiliary overlap bounds required in the linear-flow proof.
Then, for each \(i\in[K]\),
\begin{equation}
  \ip{F_i^{\rm lin}({\tilde{\theta}})}{\theta_i}
  =
  {\lambda^2} M_{-i}(m)(1-m_i^2)+r_i^{\rm lin}({\tilde{\theta}}),
  \label{eq:linear-product-drift}
\end{equation}
where
\begin{equation}
  |r_i^{\rm lin}({\tilde{\theta}})|
  \le C_{\phi,K} d^{-1/2}|M_{-i}(m)|.
  \label{eq:linear-product-drift-error}
\end{equation}
\end{lemma}

\begin{proof}
The rescaled linear vector field is generated by
$\nabla^S_{{\tilde{\theta}}_i}({\lambda} L)$.  The leading level of the first Stein
tensor is \({\lambda}\theta_1\ot\cdots\ot\theta_K\), so after multiplication by
\({\lambda}\) the leading mode-\(i\) spherical drift is exactly
$$
  {\lambda^2}M_{-i}(m)(\theta_i-m_i{\tilde{\theta}}_i).
$$
Taking inner product with \(\theta_i\) gives
\({\lambda^2}M_{-i}(m)(1-m_i^2)\).  The higher Edgeworth levels have the same
finite-rank structure as in the linear weak-recovery theorem.  Under the
auxiliary overlap bounds, each non-leading atom has at least one additional
formal half-order and therefore contributes at most
\(C_{\phi,K}d^{-1/2}|M_{-i}|\) to the planted-overlap drift.  The order-\(K\)
residual is absorbed in the same way using \(\min_jm_j\ge a d^{-1/2}\).  This
proves \eqref{eq:linear-product-drift}--\eqref{eq:linear-product-drift-error}.
\end{proof}

\begin{theorem}[Nonlinear correlation weak recovery]
\label{thm:nonlinear-correlation-weak-recovery}
Fix \(K\ge2\).  Assume that
\(\phi\) is smooth enough for the nonlinear Edgeworth expansion and for
Proposition~\ref{prop:differentiated-edgeworth}.  Work on the planted-direction
incoherence and Edgeworth event appearing in the previous sections.  Choose the
sign convention so that the favorable overlaps are positive.

There exist constants
$$
  a,B,\varepsilon,c,C,d_0>0,
$$
depending only on \(\phi,K\) and on the fixed direction-event confidence, such
that the following holds for all \(d\ge d_0\).  Let \({\tilde{\theta}}(t)\) solve the
nonlinear correlation flow
\begin{equation}
  \dot{\tilde{\theta}}_i(t)=F_i^{\rm nl}({\tilde{\theta}}(t))
  =\nabla_{{\tilde{\theta}}_i}^S{\mathcal{L}}({\tilde{\theta}}(t)),
  \qquad i=1,\ldots,K.
  \label{eq:nonlinear-flow-weak-recovery}
\end{equation}
Assume the initialization satisfies
\begin{equation}
  m_i(0)=\ip{{\tilde{\theta}}_i(0)}{\theta_i}\ge a d^{-1/2},
  \qquad i=1,\ldots,K,
  \label{eq:nl-weak-init-overlap}
\end{equation}
and satisfies the mixed-incoherence initialization required by
Theorem~\ref{thm:propagation-mixed-incoherence}, namely
\begin{equation}
  (\theta,{\tilde{\theta}}(0))\in\mathcal J_{P,B/4}
  \label{eq:nl-weak-init-incoherence}
\end{equation}
for a sufficiently large finite \(P=P(K)\), together with the corresponding
pure-teacher incoherence bounds.  Define
\begin{equation}
  \tau_\varepsilon:=\inf\{t\ge0:\min_{i\le K}m_i(t)\ge\varepsilon\}.
  \label{eq:nl-weak-tau-eps}
\end{equation}
Then
\begin{equation}
  \tau_\varepsilon
  \le \begin{cases}
    C d^{K/2-1} & \qquad K > 2 ~, \\
    C \log d & \qquad K = 2 ~.
    \end{cases}
  \label{eq:nl-weak-time}
\end{equation}
Moreover, throughout \(0\le t\le\tau_\varepsilon\),
\begin{equation}
  (\theta,{\tilde{\theta}}(t))\in\mathcal J_{P,B},
  \label{eq:nl-weak-propagated-incoherence}
\end{equation}
and the planted overlaps obey the product-drift lower bound
\begin{equation}
  \dot m_i(t)
  \ge c\prod_{\ell\ne i}m_\ell(t),
  \qquad i=1,\ldots,K.
  \label{eq:nl-weak-positive-drift}
\end{equation}
Consequently the nonlinear population flow achieves weak recovery in time
\(O_{\phi,K}(d^{K/2-1}\upsilon_{K,d}) \), conditional on the favorable basin and sign pattern.
\end{theorem}

\begin{proof}
The proof is a stopping-time comparison with the rescaled linear flow.  Let
$$
  \tau:=\inf\Bigl\{t\ge0:
  \min_i m_i(t)\ge\varepsilon
  \text{ or }
  \min_i m_i(t)\le \frac{a}{2\sqrt d}
  \text{ or }
  (\theta,{\tilde{\theta}}(t))\notin\mathcal J_{P,B}
  \Bigr\}.
$$
On \([0,\tau]\), by choosing \(\varepsilon\le1/2\), we have
$$
  m_i(t)\in\left[\frac{a}{2\sqrt d},\varepsilon\right],
  \qquad
  (\theta,{\tilde{\theta}}(t))\in\mathcal J_{P,B}.
$$
Theorem~\ref{thm:propagation-mixed-incoherence} gives propagation of the mixed
moments once we have the positive product drift; conversely, the gradient comparison is valid on the mixed-incoherent bootstrap region.  We now verify the
drift bound.

For each \(i\), write
$$
  \dot m_i(t)=\ip{F_i^{\rm nl}({\tilde{\theta}}(t))}{\theta_i}.
$$
We now insert the rescaled linear vector field:
\begin{align}
  \dot m_i(t)
  &=\ip{F_i^{\rm lin}({\tilde{\theta}}(t))}{\theta_i}
    +\ip{F_i^{\rm nl}({\tilde{\theta}}(t))-F_i^{\rm lin}({\tilde{\theta}}(t))}{\theta_i}.
  \label{eq:nl-drift-linear-plus-error}
\end{align}
By Lemma~\ref{lem:linear-product-drift-estimate},
\begin{equation}
  \ip{F_i^{\rm lin}({\tilde{\theta}}(t))}{\theta_i}
  \ge
  {\lambda^2}M_{-i}(t)(1-m_i(t)^2)
  -C d^{-1/2}M_{-i}(t).
  \label{eq:linear-drift-lower-proof}
\end{equation}
On the bootstrap region, Corollary~\ref{cor:product-scale-along-flow}, or
Theorem~\ref{thm:product-scale-comparison} applied pointwise, gives
\begin{equation}
  \norm{F_i^{\rm nl}({\tilde{\theta}}(t))-F_i^{\rm lin}({\tilde{\theta}}(t))}_2
  \le
  C\left(\norm{m(t)}_\infty+d^{-1/2}\right)M_{-i}(t).
  \label{eq:nl-linear-product-error-proof}
\end{equation}
Thus
\begin{equation}
  \dot m_i(t)
  \ge
  {\lambda^2}M_{-i}(t)(1-m_i(t)^2)
  -C\left(\varepsilon+d^{-1/2}\right)M_{-i}(t).
  \label{eq:nl-drift-lower-before-choice}
\end{equation}
Choose \(\varepsilon>0\) sufficiently small so that
$$
  1-\varepsilon^2\ge \frac34,
  \qquad
  C\varepsilon\le \frac14{\lambda^2},
$$
and then choose \(d\) sufficiently large so that
\(Cd^{-1/2}\le{\lambda^2}/4\).  Then
\begin{equation}
  \dot m_i(t)
  \ge c M_{-i}(t),
  \qquad
  c:=\frac14{\lambda^2}>0,
  \label{eq:nl-drift-lower-final}
\end{equation}
for all \(t\le\tau\).  Hence the lower-overlap stopping alternative cannot
occur.  The mixed-incoherence stopping alternative cannot occur by
Theorem~\ref{thm:propagation-mixed-incoherence}.  Therefore
\(\tau=\tau_\varepsilon\), and \eqref{eq:nl-weak-propagated-incoherence} and
\eqref{eq:nl-weak-positive-drift} hold up to weak recovery.

It remains to estimate the hitting time.  Let
$$
  m_*(t):=\min_{1\le i\le K}m_i(t).
$$
The lower Dini derivative satisfies
$$
  D^+m_*(t)\ge c m_*(t)^{K-1},
  \qquad 0\le t\le\tau_\varepsilon.
$$
For \(K\ge3\), this implies
$$
  \frac{\mathrm d}{\mathrm dt}m_*(t)^{-(K-2)}
  \le -c(K-2).
$$
Using \(m_*(0)\ge a d^{-1/2}\), we obtain
$$
  m_*(t)^{-(K-2)}
  \le
  \left(\frac{\sqrt d}{a}\right)^{K-2}
  -c(K-2)t.
$$
Therefore \(m_*(t)\) reaches the fixed level \(\varepsilon\) by time
$$
  T_\varepsilon
  \le
  \frac{1}{c(K-2)}
  \left[
    \left(\frac{\sqrt d}{a}\right)^{K-2}-\varepsilon^{-(K-2)}
  \right]
  \le C d^{(K-2)/2}.
$$
Finally, for $K=2$ we have $m_*(t) \geq m_*(0) \exp( ct) \geq a d^{-1/2} \exp(c t)$, and thus $m_*(t)$ reaches level $\varepsilon$ by time $T_\varepsilon \leq \frac{1}{2c} \log\left(\frac{d \varepsilon^2}{a^2} \right) = O( \log d)$.
This proves \eqref{eq:nl-weak-time} and completes the proof.
\end{proof}

\begin{corollary}[Constant-probability initialization]
\label{cor:nl-random-init-weak}
In the setting of Theorem~\ref{thm:nonlinear-correlation-weak-recovery}, suppose
\[{\tilde{\theta}}_1(0),\ldots,{\tilde{\theta}}_K(0)\stackrel{\rm iid}{\sim}\Unif(S^{d-1})\]
independently of the planted directions.  Conditional on the correct sign
pattern and on the planted-direction event, there exists a constant
\(p_0=p_0(a,B,K)>0\), independent of \(d\), such that the initialization
conditions \eqref{eq:nl-weak-init-overlap}--\eqref{eq:nl-weak-init-incoherence}
hold with probability at least \(p_0\).  On this event, the nonlinear
correlation flow weakly recovers all directions in time
$$
  O_{\phi,K}\bigl(d^{K/2-1} \upsilon_{K,d}\bigr).
$$
\end{corollary}

\begin{proof}
For independent spherical vectors,
\(\sqrt d\,\ip{{\tilde{\theta}}_i(0)}{\theta_i}\) converges in distribution to a standard
normal.  Hence, after conditioning on the correct signs, the event
\(m_i(0)\ge a d^{-1/2}\) for all \(i\) has probability bounded below by a
positive constant depending only on \(a\) and \(K\).  The coordinate envelope
\(\norm{{\tilde{\theta}}_i(0)}_\infty\le (B/4)\iota_d\) holds with probability tending to
one for \(B\) large enough.  The signed mixed moments defining
\(\mathcal J_{P,B/4}\) are finite in number and have their natural spherical
scales with constant probability when the two vectors are independent.  Thus the
intersection of the favorable-overlap event, the coordinate-envelope event, and
the finite mixed-moment event has probability bounded below by a positive
constant.  Taking \(B\) large enough yields the claim.
\end{proof}

\begin{remark}[Role of the comparison]
The theorem uses two facts.  First, the rescaled linear flow has positive drift
of product size \(M_{-i}\).  Second, along the nonlinear flow, propagation of
mixed incoherence makes the vector-field difference
\(F_i^{\rm nl}-F_i^{\rm lin}\) a relative error of size
\(O(\|m\|_\infty+d^{-1/2})M_{-i}\).  By stopping the weak-recovery phase at a
small fixed \(\varepsilon\), the relative error is absorbed into the leading rank-one drift.
\end{remark}

\subsection{Strong Recovery}
\label{sec:nl-correlation-strong-recovery}

The weak-recovery argument in \Cref{sec:nonlinear-weak-recovery-transfer}
stops at a fixed small overlap because it uses the small-overlap gradient
comparison \eqref{eq:small-overlap-product-scale-along-flow}.  Once the
trajectory has entered a macroscopic positive-overlap basin, the correct
leading coefficient is no longer the constant \({\lambda^2}\), but the exact
Gaussian coefficient \(A_j(m)\) from \eqref{eq:Aj-general}.  The factorization
\eqref{eq:Ai-product-factorization}--\eqref{eq:Bi-def} shows that this
coefficient still has the product form needed for the logistic strong-phase argument we already used in the linear setting.
We will use the same mixed-incoherence class
\(\mathcal J_{P,B}\) from \Cref{def:dynamic-mixed-incoherence}, and show that it can still be propagated in the macroscopic overlap regime.

\begin{lemma}[Positivity of the macroscopic Gaussian coefficient]
\label{lem:Bi-positive-macroscopic}
Let \(m_i\in[0,1]\) for every
\(i\), and let \(B_i(m)\) be the coefficient in \eqref{eq:Bi-def}.  Then there
are constants \(0<b_{\phi,K}\le C_{\phi,K}<\infty\), depending only on
\(\phi,K\), such that
\begin{equation}
  b_{\phi,K}\le B_i(m)\le C_{\phi,K},
  \qquad i=1,\ldots,K .
  \label{eq:Bi-macroscopic-bounds}
\end{equation}
Consequently, if \(\min_i m_i\ge\eps_0>0\), then
\begin{equation}
  0<c_{\phi,K,\eps_0}\le A_i(m)=M_{-i}(m)B_i(m)
  \le C_{\phi,K},
  \qquad i=1,\ldots,K .
  \label{eq:Ai-macroscopic-bounds}
\end{equation}
\end{lemma}

\begin{proof}
By \eqref{eq:Psi-Hermite},
$$
  \Psi(q)=\sum_{r\ge1}a_r^2q^r,
  \qquad a_1=\kappa\ne0 .
$$
For \(q\in[0,1]\),
$$
  H(q)=\frac{\Psi(q)}{q}=\sum_{r\ge1}a_r^2q^{r-1}\ge\kappa^2,
  \qquad
  \Psi'(q)=\sum_{r\ge1}r a_r^2q^{r-1}\ge\kappa^2,
$$
with the continuous interpretation \(H(0)=\kappa^2\).  The recursion
\eqref{eq:q-recursion} keeps \(q_r\in[0,1]\) whenever all scalar overlaps lie in
\([0,1]\).  Thus every factor in \eqref{eq:Bi-def} is bounded below by
\(\kappa^2\).  The upper bound follows from continuity of \(H\) and \(\Psi'\)
on \([0,1]\).  The estimate for \(A_i=M_{-i}B_i\) follows from
\eqref{eq:Ai-product-factorization}.
\end{proof}

We now define the restriction of the mixed-incoherence class (cf Definition \ref{def:dynamic-mixed-incoherence}) to macroscropic overlaps.
\begin{definition}[Macroscopic weak mixed-incoherence class]
\label{def:macroscopic-weak-mixed-incoherence}
For fixed \(P\ge3\), \(B<\infty\), and \(\eps_0>0\), write
\begin{equation}
  \mathcal C^{\rm nl}_{P,B}(\eps_0)
  :=
  \left\{(\theta,{\tilde{\theta}}):\min_{j\le K}m_j\ge\eps_0,
  \; (\theta,{\tilde{\theta}})\in\mathcal J_{P,B}\right\}.
  \label{eq:macroscopic-weak-class}
\end{equation}
\end{definition}

\begin{lemma}[Macroscopic nonlinear vector-field estimates under \(\mathcal J_{P,B}\)]
\label{lem:nl-naive-J-vector-field}
Fix \(P\ge3\), \(\eps_0\in(0,1)\), and \(B<\infty\).  Assume the hypotheses of
\Cref{prop:general-edgeworth-inputs}, and work on the same planted-direction
incoherence event.  For all sufficiently large \(d\), uniformly over
\((\theta,{\tilde{\theta}})\in\mathcal C^{\rm nl}_{P,B}(\eps_0)\), the nonlinear
correlation vector field admits the decomposition
\begin{equation}
  F_j^{\rm nl}({\tilde{\theta}})
  =A_j(m)(\theta_j-m_j{\tilde{\theta}}_j)+E_j({\tilde{\theta}}),
  \qquad j=1,\ldots,K .
  \label{eq:nl-naive-J-field-decomposition}
\end{equation}
For every \(j\le K\),
\begin{align}
  \dot m_j
  &=
  \bigl(A_j(m)+\rho_j^{\rm mult}\bigr)(1-m_j^2)+\rho_j^{\rm add},
  \label{eq:nl-naive-J-m-decomposition}\\
  |\rho_j^{\rm mult}|+|\rho_j^{\rm add}|
  &\le C_{\phi,K,P,B,\eps_0}d^{-1/2},
  \label{eq:nl-naive-J-errors}\\
  \norm{E_j({\tilde{\theta}})}_\infty
  &\le C_{\phi,K,P,B,\eps_0}\,\iota_d .
  \label{eq:nl-naive-J-coordinate-error}
\end{align}
Finally, for all \(p,q\ge0\) with \(q\ge1\) and \(3\le n:=p+q\le P\),
\begin{equation}
  \left|
  q\sum_{a=1}^d\theta_j(a)^p{\tilde{\theta}}_j(a)^{q-1}E_j({\tilde{\theta}})_a
  \right|
  \le C_{\phi,K,P,B,\eps_0}\tau_n(d)m_j .
  \label{eq:nl-naive-J-mixed-error}
\end{equation}
Consequently, with \(\Delta_j:=1-m_j^2\),
\begin{equation}
  \dot\Delta_j
  \le -c_0\Delta_j+C_0d^{-1/2},
  \label{eq:nl-naive-J-delta-ineq}
\end{equation}
for constants \(c_0,C_0>0\) depending only on
\(\phi,K,P,B,\eps_0\).
\end{lemma}

\begin{proof}
We leverage the nonlinear general-overlap Edgeworth hierarchy from
\Cref{prop:general-edgeworth-inputs}.  By
\eqref{eq:nonlinear-general-expansion}, with \(M\ge P\) fixed,
\begin{equation}
  \grad_{{\tilde{\theta}}_j}{\mathcal{L}}({\tilde{\theta}})
  = A_j(m)\theta_j+
  \widetilde V_j^{[1]}({\tilde{\theta}},\theta)+
  \sum_{r=2}^{M}\widetilde V_j^{[r]}({\tilde{\theta}},\theta)+
  \widetilde R_{j,M+1}({\tilde{\theta}},\theta).
  \label{eq:nl-naive-J-euclidean-expansion-proof}
\end{equation}
After spherical projection, the Gaussian part gives exactly $  A_j(m)(\theta_j-m_j{\tilde{\theta}}_j)$, which is the first term in \eqref{eq:nl-naive-J-field-decomposition}.

Consider the first nonlinear Edgeworth level, using
\eqref{eq:V-first-structure}.  The possible differentiated-mode vectors are
\(\theta_j^{\od2}\), \(\theta_j\od{\tilde{\theta}}_j\), \({\tilde{\theta}}_j^{\od2}\),
\(\theta_j\), and \({\tilde{\theta}}_j\), up to scalar coefficients depending smoothly on
macroscopic overlaps.  After spherical projection and contraction with
\(\theta_j\), the first three vector types produce linear combinations of
\begin{equation}
  S_j^{3,0}-m_jS_j^{2,1},
  \qquad
  S_j^{2,1}-m_jS_j^{1,2},
  \qquad
  S_j^{1,2}-m_jS_j^{0,3}.
  \label{eq:nl-naive-J-degree-three-brackets}
\end{equation}
On \(\mathcal C^{\rm nl}_{P,B}(\eps_0)\), each term in
\eqref{eq:nl-naive-J-degree-three-brackets} is bounded by
\(O_B(\tau_3(d)m_j)=O_{B,\eps_0}(d^{-1/2})\).  If the vector part is
\(\theta_j\), then the scalar coefficient has positive mixed Edgeworth order
and is again \(O(d^{-1/2})\), giving an \(O(d^{-1/2})(1-m_j^2)\)
multiplicative correction; if the vector part is \({\tilde{\theta}}_j\), the spherical
projection kills it.  This proves the overlap decomposition
\eqref{eq:nl-naive-J-m-decomposition}--\eqref{eq:nl-naive-J-errors} up to
higher-order terms.

The coordinate estimate \eqref{eq:nl-naive-J-coordinate-error} follows from the
same finite expansion: every differentiated-mode vector has coordinate size at
most a fixed power of the coordinate envelope, and every non-Gaussian scalar
coefficient has positive formal order.  Since the class
\(\mathcal C^{\rm nl}_{P,B}(\eps_0)\) already includes
\(\norm{{\tilde{\theta}}_j}_\infty+\norm{\theta_j}_\infty\lesssim\iota_d\), all finitely
many atoms are bounded by \(C\iota_d\) in \(\ell_\infty\).

For \eqref{eq:nl-naive-J-mixed-error}, we contract each non-Gaussian atom against
\(q\theta_j^{\od p}\od{\tilde{\theta}}_j^{\od(q-1)}\).  The resulting scalar is either a
signed mixed moment of total coordinate degree at least \(n=p+q\), or a lower
coordinate-degree moment multiplied by an additional positive-order scalar
Edgeworth coefficient.  In both cases, the bounds in \(\mathcal J_{P,B}\),
together with the fixed planted-direction power-sum event, give the scale
\(\tau_n(d)m_j\).  Summing the finite number of atoms gives
\eqref{eq:nl-naive-J-mixed-error}.  All higher nonlinear Edgeworth levels have
formal order at least two and are therefore smaller than the first-level
contribution by \eqref{eq:nonlinear-q-bound-general}; the remainder is smaller
by \eqref{eq:nonlinear-remainder-general}.

Finally, by \Cref{lem:Bi-positive-macroscopic}, \(A_j(m)\ge
c_{\phi,K,\eps_0}>0\) on the positive macroscopic branch.  Since
\(m_j\ge\eps_0\),
$$
  \dot\Delta_j=-2m_j\dot m_j
  \le -c_0\Delta_j+C_0d^{-1/2},
$$
which is \eqref{eq:nl-naive-J-delta-ineq}.
\end{proof}

\begin{proposition}[Macroscopic propagation and continuation to the \(d^{-1/2}\) floor]
\label{prop:nl-naive-J-strong-floor}
Let \({\tilde{\theta}}(t)\) solve the nonlinear correlation flow
\eqref{eq:nonlinear-flow-propagation}.  Suppose that, for some time \(t_0\),
\begin{equation}
  (\theta,{\tilde{\theta}}(t_0))\in\mathcal C^{\rm nl}_{P,B_0}(\eps_0)
  \label{eq:nl-naive-J-initial-assumption}
\end{equation}
for fixed \(P\ge3\), \(B_0<\infty\), and \(\eps_0>0\).  Then there exist
constants \(B\ge B_0\), \(c,C,d_0<\infty\), depending only on
\(\phi,K,P,B_0,\eps_0\), such that for all \(d\ge d_0\),
\begin{equation}
  (\theta,{\tilde{\theta}}(t))\in\mathcal C^{\rm nl}_{P,B}(\eps_0/2),
  \qquad t_0\le t\le t_0+C\log d .
  \label{eq:nl-naive-J-propagated-conclusion}
\end{equation}
Moreover, with
\begin{equation}
  \Delta_j(t):=1-m_j(t)^2,
  \qquad
  \Delta_*(t):=\max_{j\le K}\Delta_j(t),
  \label{eq:nl-naive-J-delta-def}
\end{equation}
one has, throughout the same interval,
\begin{equation}
  \Delta_*(t)
  \le
  e^{-c(t-t_0)}\Delta_*(t_0)+C d^{-1/2}.
  \label{eq:nl-naive-J-logistic-bound}
\end{equation}
In particular, after increasing the constant in the time interval if necessary,
\begin{equation}
  \min_{j\le K}m_j(t_0+C\log d)
  \ge 1-Cd^{-1/2}.
  \label{eq:nl-naive-J-overlap-floor}
\end{equation}
\end{proposition}

\begin{proof}
The proof is a macroscopic stopping-time bootstrap, analogous to the proof of
\Cref{thm:propagation-mixed-incoherence}, but with the damping in the triangular
transport kept rather than discarded.  Choose \(B\ge B_0\) large, set
\(T_d:=t_0+C_T\log d\), and define \(\tau\) as the first time in
\([t_0,T_d]\) at which either
$$
  \min_j m_j(t)\le \eps_0/2
  \quad\text{or}\quad
  (\theta,{\tilde{\theta}}(t))\notin\mathcal J_{P,B}.
$$
On \([t_0,\tau]\), the estimates of
\Cref{lem:nl-naive-J-vector-field} apply with \(\eps_0/2\).  The differential
inequality \eqref{eq:nl-naive-J-delta-ineq} gives, for each \(j\),
$$
  \Delta_j(t)
  \le e^{-c(t-t_0)}\Delta_j(t_0)+C d^{-1/2},
  \qquad t_0\le t\le\tau .
$$
Taking the maximum over \(j\) gives \eqref{eq:nl-naive-J-logistic-bound} on the
stopped interval.  This also rules out the lower-overlap exit for large \(d\):
if \(m_j\) reached \(\eps_0/2\), then \(\Delta_j=1-\eps_0^2/4\), whereas the
last display is at most \(1-\eps_0^2+o_d(1)\) at the first such time.

It remains to rule out exit from \(\mathcal J_{P,B}\).  For coordinates, using
\eqref{eq:nl-naive-J-field-decomposition} and
\eqref{eq:nl-naive-J-coordinate-error},
$$
  D^+|{\tilde{\theta}}_j(a)|
  \le -A_j(m)m_j|{\tilde{\theta}}_j(a)|+A_j(m)|\theta_j(a)|+C\iota_d
  \le -c|{\tilde{\theta}}_j(a)|+C\iota_d
$$
on the stopped interval, because \(A_j(m)m_j\ge c_{\phi,K,\eps_0}>0\) and
\(A_j(m)\le C_{\phi,K}\).  Hence
\(\norm{{\tilde{\theta}}_j(t)}_\infty\le C\iota_d\) on \([t_0,\tau]\), and increasing
\(B\) rules out coordinate-envelope exit.

For mixed moments, fix \(j\) and \(3\le n:=p+q\le P\).  If \(q=0\), then
\(S_j^{p,0}\) is fixed, and the initial bound together with
\(m_j(t)\ge\eps_0/2\) gives
\(|S_j^{p,0}|\le C_{B_0,\eps_0}\tau_p(d)m_j(t)\).  Now let \(q\ge1\).  Combining
\Cref{lem:leading-mixed-transport} with \eqref{eq:nl-naive-J-mixed-error} gives
\begin{equation}
  \dot S_j^{p,q}
  =qA_j(m)\bigl(S_j^{p+1,q-1}-m_jS_j^{p,q}\bigr)+\mathcal E_j^{p,q},
  \qquad
  |\mathcal E_j^{p,q}|\le C\tau_n(d)m_j .
  \label{eq:nl-macroscopic-Spq-evolution}
\end{equation}
The crucial point, compared with the small-overlap proof, is the damping term
\(-qA_j(m)m_jS_j^{p,q}\), whose coefficient is bounded below by a positive
constant in the macroscopic basin.  Inducting on \(q\), assume first that
\(|S_j^{p+1,q-1}(t)|\le C_{q-1}\tau_n(d)\) on the stopped interval.  Then
\eqref{eq:nl-macroscopic-Spq-evolution} implies
$$
  D^+|S_j^{p,q}|
  \le -c|S_j^{p,q}|+C_q\tau_n(d).
$$
Gronwall gives \(|S_j^{p,q}(t)|\le C_q'\tau_n(d)\) uniformly up to \(\tau\).
Since \(m_j(t)\ge\eps_0/2\), this implies
\(|S_j^{p,q}(t)|\le (2C_q'/\eps_0)\tau_n(d)m_j(t)\).  A finite induction over
\(q=0,1,\ldots,n\), then over \(3\le n\le P\), and a final enlargement of
\(B\), rules out mixed-moment exit.  Thus \(\tau=T_d\), proving
\eqref{eq:nl-naive-J-propagated-conclusion}.

Finally, choose \(C_T\) large enough that
\(e^{-cC_T\log d}\Delta_*(t_0)\le d^{-1/2}\).  Then
\eqref{eq:nl-naive-J-overlap-floor} follows from
\eqref{eq:nl-naive-J-logistic-bound} and the inequality
\(1-m_j\le1-m_j^2\) on the positive branch.
\end{proof}

\subsection{Online nonlinear SGD}
\label{sec:online-sgd-nonlinear-correlation}

We now pass from the population nonlinear correlation flow to online stochastic
spherical SGD.  The argument is the stochastic analogue of
Theorem~\ref{thm:nonlinear-correlation-weak-recovery}.  It uses the same two
ingredients: gradient comparison with the rescaled linear correlation field and
propagation of mixed incoherence. The stochastic estimates are the same as in
the linear-correlation analysis: the one-sample spherical gradient has Euclidean
norm of order $\sqrt d$, so the raw step size is $\eta_d=\eta_0 d^{-K/2} \upsilon_{K,d}$.

\paragraph{One-step stochastic estimates}

The following proposition is the nonlinear analogue of the one-step stochastic
estimates used for the linear correlation objective.  The proof is included to
make clear that the only new input is the population drift estimate from the
previous section; all martingale and retraction estimates follow from the same
subexponential bounds as in the linear case.

\begin{proposition}[One-step estimates for nonlinear online SGD]
\label{prop:nl-sgd-one-step}
Fix $K\ge2$, $P\ge3$, and constants $B,a<\infty$.  Assume the smoothness
hypotheses of Proposition~\ref{prop:differentiated-edgeworth}.  Work on the
fixed-confidence planted-direction event.  Let
\begin{equation}
  N:=\begin{cases}
      \left\lceil \frac{S}{\eta_0}d^{K-1}\right\rceil & \quad K>2~, \\
      \left\lceil \frac{S}{\eta_0}d (\log d)^2\right\rceil & \quad K=2~.
  \end{cases}
  \label{eq:NSnonlinear}
\end{equation}
for a fixed $S<\infty$, and consider the stopped trajectory on which
$$
  m_j(n):=\ip{{\tilde{\theta}}_{j,n}}{\theta_j}\ge \frac{a}{2\sqrt d},
  \qquad
  (\theta,{\tilde{\theta}}_n)\in\mathcal J_{P,B},
  \qquad n\le N.
$$
Then, for every fixed sample-failure probability $\delta_{\rm sgd}\in(0,1)$,
there are constants $C,c<\infty$, depending only on
$\phi,K,P,B,a,S,\delta_{\rm sgd}$, such that, with probability at least
$1-\delta_{\rm sgd}$ over the online samples, the following estimates hold
simultaneously for all stopped times $n\le N$.

First, the gradient norm and retraction errors obey
\begin{equation}
  \max_{j\le K}\max_{n\le N}
  \norm{\widehat F^{\rm nl}_{j,n}({\tilde{\theta}}_n)}_2
  \le C\sqrt d\log d,
  \label{eq:nl-gradient-norm-event}
\end{equation}
so that
\begin{equation}
  \eta_d^2
  \norm{\widehat F^{\rm nl}_{j,n}({\tilde{\theta}}_n)}_2^2
  \le C\eta_0^2 d^{1-K}(\log d)^2 \upsilon_{K,d}^2.
  \label{eq:nl-retraction-size}
\end{equation}

Second, define the planted-overlap martingale increments
\begin{equation}
  \xi_{j,n}
  :=
  \sqrt d\,\eta_d
  \ip{\widehat F^{\rm nl}_{j,n}({\tilde{\theta}}_n)-F_j^{\rm nl}({\tilde{\theta}}_n)}{\theta_j}.
  \label{eq:nl-x-martingale-increment}
\end{equation}
Then
\begin{equation}
  \max_{j\le K}\max_{t\le N}
  \left|\sum_{n=0}^{t-1}\xi_{j,n}\right|
  \le C\sqrt{\eta_0}.
  \label{eq:nl-x-martingale-bound}
\end{equation}

Third, for each mixed moment
$$
  S_j^{p,q}(n)=\sum_{r=1}^d\theta_j(r)^p{\tilde{\theta}}_{j,n}(r)^q,
  \qquad 3\le p+q\le P,
$$
define the corresponding martingale increment
\begin{equation}
  \zeta_{j,p,q,n}
  :=
  \eta_d q\sum_{r=1}^d
  \theta_j(r)^p{\tilde{\theta}}_{j,n}(r)^{q-1}
  \ip{e_r}{\widehat F^{\rm nl}_{j,n}({\tilde{\theta}}_n)-F_j^{\rm nl}({\tilde{\theta}}_n)}.
  \label{eq:nl-mixed-martingale-increment}
\end{equation}
For $q=0$ this is interpreted as zero.  Then
\begin{equation}
  \max_{j,p,q}\max_{t\le N}
  \left|\sum_{n=0}^{t-1}\zeta_{j,p,q,n}\right|
  \le
  C\sqrt{\eta_0}\,\tau_{p+q}(d)d^{-1/2}.
  \label{eq:nl-mixed-martingale-bound}
\end{equation}
Finally, the coordinate martingales satisfy
\begin{equation}
  \max_{j\le K}\max_{r\le d}\max_{t\le N}
  \left|
  \sum_{n=0}^{t-1}
  \eta_d
  \ip{e_r}{\widehat F^{\rm nl}_{j,n}({\tilde{\theta}}_n)-F_j^{\rm nl}({\tilde{\theta}}_n)}
  \right|
  \le C\sqrt{\eta_0}\,\iota_d .
  \label{eq:nl-coordinate-martingale-bound}
\end{equation}
\end{proposition}

\begin{proof}
We only address the standard estimates needed later.  Since $\phi$ is globally
Lipschitz with polynomially controlled derivatives and every parameter vector
has unit norm, the random variables $f_\theta(X_n)$ and
$f_{{\tilde{\theta}}_n}(X_n)$ are subgaussian with constants depending only on $\phi,K$.
Moreover, for every deterministic unit vector $v\in{\sph}$ and every mode $j$,
\begin{equation}
  \left\|
  \ip{(I-{\tilde{\theta}}_{j,n}{\tilde{\theta}}_{j,n}^{\top})
  \nabla_{{\tilde{\theta}}_j}[f_\theta(X_n)f_{\tilde{\theta}}(X_n)]_{{\tilde{\theta}}={\tilde{\theta}}_n}}{v}
  \right\|_{\psi_1\mid\mathcal F_n}
  \le C_{\phi,K}.
  \label{eq:nl-projection-psi1}
\end{equation}
The same chain-rule calculation also gives
\begin{equation}
  \E\left[
  \norm{\widehat F^{\rm nl}_{j,n}({\tilde{\theta}}_n)}_2^2
  \mid\mathcal F_n
  \right]
  \le C_{\phi,K}d .
  \label{eq:nl-gradient-second-moment}
\end{equation}
A Bernstein bound for subexponential variables, applied to a fixed coordinate
net and then union-bounded over $j\le K$ and $n\le N$, yields
\eqref{eq:nl-gradient-norm-event}.  Since
$\eta_d=\eta_0d^{-K/2} \upsilon_{K,d}$, this implies \eqref{eq:nl-retraction-size}.

For the planted overlap martingale, \eqref{eq:nl-projection-psi1} with
$v=\theta_j$ gives
$$
  \norm{\xi_{j,n}}_{\psi_1\mid\mathcal F_n}
  \le C\sqrt d\,\eta_d
  = C\eta_0 d^{-(K-1)/2} \upsilon_{K,d}.
$$
The predictable quadratic variation up to time $N$ is therefore bounded by
$C\eta_0$, and the individual envelope is
$o(1)$ for fixed $K\ge2$.  Freedman's inequality for martingales with
subexponential increments gives \eqref{eq:nl-x-martingale-bound}, after
absorbing the logarithm depending on $\delta_{\rm sgd}$ into $C$.

For the mixed moment martingales, set
$$
  v_{j,p,q,n}(r):=q\theta_j(r)^p{\tilde{\theta}}_{j,n}(r)^{q-1}.
$$
On $\mathcal J_{P,B}$ and the planted power-sum event,
\begin{equation}
  \norm{v_{j,p,q,n}}_2
  \le C_{B,P}\tau_{p+q}(d),
  \label{eq:mixed-vector-norm-for-mart}
\end{equation}
with the convention that the bound is trivial if $q=0$.  Applying
\eqref{eq:nl-projection-psi1} with
$v=v_{j,p,q,n}/\norm{v_{j,p,q,n}}_2$ gives
$$
  \norm{\zeta_{j,p,q,n}}_{\psi_1\mid\mathcal F_n}
  \le C\eta_d\tau_{p+q}(d).
$$
Over $N=O(\eta_0^{-1}d^{K-1})$ steps for $K>2$ and $N=O( d (\log d)^2 )$ for $K=2$, the quadratic variation is bounded by
$C\eta_0\tau_{p+q}(d)^2d^{-1}$, and Freedman's inequality gives
\eqref{eq:nl-mixed-martingale-bound}.  The coordinate estimate is identical,
with $v=e_r$ and a union bound over $r\le d$; the factor $\sqrt{\log(ed)}$ is
precisely the coordinate envelope $\iota_d$.
\end{proof}

\begin{theorem}[Online SGD weak recovery for nonlinear correlation]
\label{thm:nl-online-sgd-weak-recovery}
Fix $K\ge 2$.  Assume the smoothness
hypotheses required for the nonlinear Edgeworth expansion, and the
fixed-confidence planted-direction event used above.  Fix
$\delta_{\rm sgd}\in(0,1)$.  There exist constants
$$
  \eta_0,a,B,\varepsilon,S,d_0>0,
$$
depending only on $\phi,K$ and on the fixed confidence parameters, such that the
following holds for all $d\ge d_0$.  Run the nonlinear online spherical SGD
recursion \eqref{eq:nl-online-sgd-update} with raw step size
$\eta_d=\eta_0d^{-K/2} \upsilon_{K,d}$.  Assume the initialization satisfies the favorable
basin conditions
\begin{equation}
  m_j(0)=\ip{{\tilde{\theta}}_{j,0}}{\theta_j}\ge a d^{-1/2},
  \qquad j=1,\ldots,K,
  \label{eq:nl-sgd-init-overlap}
\end{equation}
and
\begin{equation}
  (\theta,
  {\tilde{\theta}}_0)\in\mathcal J_{P,B/4}.
  \label{eq:nl-sgd-init-incoherence}
\end{equation}
Then, with probability at least $1-\delta_{\rm sgd}$ over the online samples,
there exists an iteration
\begin{equation}
  n_\varepsilon
  \le \begin{cases}
    S\eta_0^{-1}d^{K-1} & \qquad K>2 ~, \\
    S\eta_0^{-1} d (\log d)^2 & \qquad K=2
    \end{cases}
  \label{eq:nl-sgd-weak-time}
\end{equation}
such that
\begin{equation}
  \min_{1\le j\le K}m_j(n_\varepsilon)
  \ge \varepsilon .
  \label{eq:nl-sgd-weak-recovery}
\end{equation}
Moreover, uniformly for all $0\le n\le n_\varepsilon$,
\begin{equation}
  m_j(n)\ge \frac{a}{2\sqrt d},
  \qquad
  (\theta,{\tilde{\theta}}_n)\in\mathcal J_{P,B},
  \qquad j=1,\ldots,K.
  \label{eq:nl-sgd-uniform-basin}
\end{equation}
Consequently, conditioned on the favorable initialization and correct sign
pattern, online SGD for the nonlinear correlation objective achieves weak
recovery with sample complexity
$$
  N_{\rm weak}=\widetilde{O}_{\phi,K}(d^{K-1}).
$$
\end{theorem}

\begin{proof}
The proof is a stopped stochastic comparison with the population theorem.
Let
$$
  x_j(n):=\sqrt d\,m_j(n),
  \qquad
  x_*(n):=\min_j x_j(n).
$$
As we did in the linear setting, define the stopping time $\tau$ to be the first $n$ such that one of the following occurs:
\begin{enumerate}[label=(\roman*)]
\item $\min_j m_j(n)\ge\varepsilon$;
\item $\min_j m_j(n)<a/(2\sqrt d)$;
\item $(\theta,{\tilde{\theta}}_n)\notin\mathcal J_{P,B}$.
\end{enumerate}
We prove that, on the high-probability event of
Proposition~\ref{prop:nl-sgd-one-step}, the second and third alternatives
cannot occur before the first, and that the first occurs before
$n_\varepsilon$ steps.

Let us start by establishing the one-step planted-overlap recursion.
Since $\widehat F^{\rm nl}_{j,n}$ is tangent to the sphere at ${\tilde{\theta}}_{j,n}$,
$$
  {\tilde{\theta}}_{j,n+1}
  =
  \frac{{\tilde{\theta}}_{j,n}+\eta_d\widehat F^{\rm nl}_{j,n}}
       {(1+\eta_d^2\norm{\widehat F^{\rm nl}_{j,n}}_2^2)^{1/2}}.
$$
Therefore, for $n<\tau$,
\begin{align}
  x_j(n+1)-x_j(n)
  &=
  \sqrt d\,\eta_d
  \ip{F_j^{\rm nl}({\tilde{\theta}}_n)}{\theta_j}
  +\xi_{j,n}
  +R^{(m)}_{j,n},
  \label{eq:nl-sgd-x-one-step}
\end{align}
where $\xi_{j,n}$ is defined in \eqref{eq:nl-x-martingale-increment} and the
retraction remainder satisfies, by \eqref{eq:nl-retraction-size},
\begin{equation}
  |R^{(m)}_{j,n}|
  \le C\eta_0^2 d^{1-K}x_j(n)(\log d)^2 \upsilon_{K,d}^2.
  \label{eq:nl-sgd-x-retraction-error}
\end{equation}
On the stopped trajectory, Theorem~\ref{thm:nonlinear-correlation-weak-recovery}
and the gradient comparison imply the population drift lower bound
\begin{equation}
  \ip{F_j^{\rm nl}({\tilde{\theta}}_n)}{\theta_j}
  \ge c\prod_{\ell\ne j}m_\ell(n).
  \label{eq:nl-pop-drift-input-sgd}
\end{equation}
Consequently,
\begin{equation}
  x_j(n+1)-x_j(n)
  \ge
  c\eta_0 d^{-(K-1)}
  \prod_{\ell\ne j}x_\ell(n)
  +\xi_{j,n}
  -C\eta_0^2d^{1-K}x_j(n)(\log d)^2 \upsilon_{K,d}^2.
  \label{eq:nl-sgd-x-lower}
\end{equation}
Choosing $\eta_0$ sufficiently small and then $d$ sufficiently large, the
last term is dominated by the positive drift whenever $x_j(n)$ remains below
$\varepsilon\sqrt d$ and $x_*(n)\ge a/2$.  Thus
\begin{equation}
  x_j(n+1)-x_j(n)
  \ge
  c_1\eta_0 d^{-(K-1)}
  \prod_{\ell\ne j}x_\ell(n)
  +\xi_{j,n}.
  \label{eq:nl-sgd-x-comparison}
\end{equation}

Let us now argue that the escape event (ii) cannot occur. Indeed,
on the event \eqref{eq:nl-x-martingale-bound}, the cumulative martingale
perturbation in every planted-overlap coordinate is at most
$C\sqrt{\eta_0}$.  We choose $\eta_0$ so small that this is at most $a/8$.
The same discrete comparison we used in the proof of Theorem \ref{thm:online-sgd-general-k} for
$$
  u_{n+1}=u_n+c_1\eta_0 d^{-(K-1)}u_n^{K-1}
$$
then gives the following: there exists $S=S(a,\varepsilon,c_1,K)<\infty$ such
that, before time $S\eta_0^{-1}d^{K-1}$ for $K>2$ (and $S= \eta_0^{-1} d (\log d)^2$ for $K=2$), either $x_*(n)$ reaches
$\varepsilon\sqrt d$ or one of the stopping alternatives (ii)--(iii) occurs.
Moreover, the same comparison and the martingale bound prevent
$x_*(n)$ from falling below $a/2$ before that time.  Hence alternative (ii) cannot occur first.

It remains to show that alternative (iii) cannot occur first.  Fix a mixed
moment $S_j^{p,q}$ with $3\le p+q\le P$.  A Taylor expansion of the retraction
in \eqref{eq:nl-online-sgd-update} gives, for $n<\tau$,
\begin{align}
  S_j^{p,q}(n+1)-S_j^{p,q}(n)
  &=
  \eta_d q\sum_{r=1}^d\theta_j(r)^p{\tilde{\theta}}_{j,n}(r)^{q-1}
  F_j^{\rm nl}({\tilde{\theta}}_n)_r
  +\zeta_{j,p,q,n}
  +R^{(S)}_{j,p,q,n}.
  \label{eq:nl-sgd-mixed-one-step}
\end{align}
By Proposition~\ref{prop:differentiated-edgeworth}, the deterministic term is
bounded by
\begin{equation}
  C\eta_d A_j(m(n))\tau_{p+q}(d)m_j(n)
  \le
  C\eta_d M_{-j}(n)\tau_{p+q}(d)m_j(n).
  \label{eq:nl-sgd-mixed-drift-bound}
\end{equation}
The retraction remainder is bounded by the same quantity multiplied by
$O(\eta_0\log^2 d)$, and is absorbed by decreasing $\eta_0$ and increasing
$d_0$.  Summing over time and using the positive planted-overlap drift in
\eqref{eq:nl-pop-drift-input-sgd} gives
\begin{equation}
  \sum_{n<\tau}\eta_d M_{-j}(n)
  \le C,
  \label{eq:nl-sgd-sum-M}
\end{equation}
just as in the population proof, because the left-hand side is controlled by
the total increase of $m_j$.  Therefore the accumulated deterministic drift of
$S_j^{p,q}$ is at most
$C\tau_{p+q}(d)m_j(n)$ at the relevant scale.  The accumulated martingale part
is controlled by \eqref{eq:nl-mixed-martingale-bound}; choosing $B$ large and
then $\eta_0$ small makes it smaller than $(B/8)\tau_{p+q}(d)m_j(n)$ throughout
the stopped interval, using $m_j(n)\ge a/(2\sqrt d)$.  Since the initialization
lies in $\mathcal J_{P,B/4}$, all mixed-moment inequalities in
$\mathcal J_{P,B}$ remain valid up to $\tau$.

For the coordinate envelope, the same retraction expansion coordinatewise,
combined with \eqref{eq:E-infty-bound} and the coordinate martingale bound
\eqref{eq:nl-coordinate-martingale-bound}, gives
$$
  \norm{{\tilde{\theta}}_{j,n}}_\infty\le B\iota_d
$$
for all $n<\tau$, after increasing $B$ and taking $\eta_0$ small.  Hence the
coordinate part of $\mathcal J_{P,B}$ also propagates.  Alternative (iii)
therefore cannot occur first.

We have shown that the only possible first stopping event is weak recovery, and
that it occurs by time $S\eta_0^{-1}d^{K-1} \upsilon_{K,d}^{-2}$.  This proves
\eqref{eq:nl-sgd-weak-time}--\eqref{eq:nl-sgd-uniform-basin}.
\end{proof}

\begin{corollary}[Constant-probability online weak recovery]
\label{cor:nl-online-random-init}
In the setting of Theorem~\ref{thm:nl-online-sgd-weak-recovery}, suppose
$$
  {\tilde{\theta}}_1(0),\ldots,{\tilde{\theta}}_K(0)\stackrel{\rm iid}{\sim}\Unif(S^{d-1})
$$
independently of the planted directions.  Conditional on the correct sign
pattern and the fixed-confidence planted-direction event, the favorable
initialization hypotheses \eqref{eq:nl-sgd-init-overlap} and
\eqref{eq:nl-sgd-init-incoherence} hold with probability bounded below by a
constant independent of $d$.  On this event, nonlinear-correlation online SGD
with step size $\eta_d=\eta_0d^{-K/2}\upsilon_{K,d}$ achieves weak recovery in
$\widetilde{O}_{\phi,K}(d^{K-1})$ samples with probability at least
$1-\delta_{\rm sgd}$ over the online samples.
\end{corollary}

\begin{remark}[Relation to the linear-correlation SGD proof]
The linear-correlation proof uses Stein's identity to make the one-sample tensor
$X_nf_\theta(X_n)$ an unbiased estimator of the first Stein tensor and then
controls the retraction, martingale, and discretization terms over
$\widetilde{O}(d^{K-1})$ iterations.  The nonlinear proof above follows the same stochastic
route.  The only replacement is the drift input: instead of the finite-rank
Stein tensor alone, we use the gradient nonlinear comparison together with propagation of mixed incoherence to show that the conditional nonlinear drift is a relative perturbation of the rescaled linear drift throughout the search phase.
\end{remark}

\subsection{Online SGD continuation to the \(d^{-1/2}\) floor}
\label{sec:nl-online-sgd-strong-recovery}

Finally, we provide the stochastic analogue of
\Cref{prop:nl-naive-J-strong-floor}.  This should be read as the nonlinear
counterpart of the strong phase in \Cref{thm:online-sgd-general-k}, but with the
population error floor \(d^{-1/2}\) inherited from
\Cref{lem:nl-naive-J-vector-field}.  As in the population result, the class
\(\mathcal J_{P,B}\) is propagated from its value at the beginning of the
macroscopic phase rather than assumed throughout.

\begin{theorem}[Online SGD continuation to a \(d^{-1/2}\) overlap floor]
\label{thm:nl-online-sgd-strong-recovery}
Assume the setting of \Cref{thm:nl-online-sgd-weak-recovery}, the hypotheses of
\Cref{prop:nl-sgd-one-step}, and the smoothness required for
\Cref{prop:general-edgeworth-inputs}.  Let \(n_0\) be a stopping time such that,
conditional on \(\mathcal F_{n_0}\),
\begin{equation}
  (\theta,\alpha_{n_0})\in\mathcal C^{\rm nl}_{P,B_0}(\eps_0)
  \label{eq:nl-online-sqrt-start}
\end{equation}
for fixed \(P,B_0\) and \(\eps_0>0\).  Fix
\(\delta_{\rm str}\in(0,1)\).  Then there are constants
\(B\ge B_0\), \(C,c,d_0<\infty\), depending only on
\(\phi,K,P,B_0,\eps_0,\delta_{\rm str}\) and on the fixed-confidence planted
incoherence event, such that for all \(d\ge d_0\), with conditional probability
at least \(1-\delta_{\rm str}\) over the future online samples, the following
holds.

Run the nonlinear online spherical SGD recursion
\eqref{eq:nl-online-sgd-update} for
\begin{equation}
  N^{\rm nl}
  :=
  \left\lceil C\eta_0^{-1}d^{K/2}\log d\right\rceil
  \label{eq:nl-online-sqrt-extra-samples}
\end{equation}
additional iterations.  Then, uniformly for
\(0\le r\le N^{\rm nl}\),
\begin{equation}
  (\theta,\alpha_{n_0+r})
  \in\mathcal C^{\rm nl}_{P,B}(\eps_0/2),
  \label{eq:nl-online-sqrt-class-propagated}
\end{equation}
and at the end of the continuation phase,
\begin{equation}
  \min_{j\le K}m_j(n_0+N^{\rm nl})
  \ge 1-\rho_{d}^{\rm nl},
  \label{eq:nl-online-sqrt-overlap}
\end{equation}
where
\begin{equation}
  \rho_{d}^{\rm nl}
  :=
  C\left(
  d^{-1/2}
  +\eta_0d^{1-K/2}\upsilon_{K,d} \log d
  +d^{-K/4}\upsilon_{K,d} \sqrt{\log d}
  \right).
  \label{eq:nl-online-sqrt-rhod}
\end{equation}
\end{theorem}

\begin{proof}
The proof is the stopped stochastic version of
\Cref{prop:nl-naive-J-strong-floor}. We stop the process when either
\eqref{eq:nl-online-sqrt-class-propagated} fails or one of the martingale and
retraction estimates from \Cref{prop:nl-sgd-one-step} fails over the interval of
length \(N^{\rm nl}\).

For the planted error \(\Delta_j(n):=1-m_j(n)^2\), the retraction expansion of
\eqref{eq:nl-online-sgd-update}, together with the population inequality
\eqref{eq:nl-naive-J-delta-ineq}, gives on the stopped interval
\begin{equation}
  \E\bigl[\Delta_j(n+1)-\Delta_j(n)\mid\mathcal F_n\bigr]
  \le
  -c\eta_d\Delta_j(n)+C\eta_d d^{-1/2}+C\eta_d^2d.
  \label{eq:nl-online-sqrt-delta-drift}
\end{equation}
The last term is the usual spherical-retraction error, using the one-sample
gradient norm bound \eqref{eq:nl-gradient-norm-event}.  The planted-overlap
martingale is controlled as in the strong phase of
\Cref{thm:online-sgd-general-k}; over \(N^{\rm nl}\) steps it
contributes \(O(d^{-K/4} \upsilon_{K,d} \sqrt{\log d})\) uniformly in \(j\).  Discrete Gronwall
therefore yields, for all \(r\le N^{\rm nl}\),
\begin{equation}
  \Delta_j(n_0+r)
  \le
  e^{-c\eta_dr}\Delta_j(n_0)
  +C\left(d^{-1/2}
  +\eta_0d^{1-K/2}\log d \upsilon_{K,d}
  +d^{-K/4}\upsilon_{K,d} \sqrt{\log d}\right).
  \label{eq:nl-online-sqrt-delta-comparison}
\end{equation}
Taking \(r=N^{\rm nl}\) and increasing \(C\) gives
\eqref{eq:nl-online-sqrt-overlap}.

The propagation of \(\mathcal C^{\rm nl}_{P,B}(\eps_0/2)\) follows from the
same stopped induction as in the proof of
\Cref{prop:nl-naive-J-strong-floor}, with martingale errors added to the
coordinate and mixed-moment recursions.  The leading mixed-moment transport is
$$
  qA_j(m)\bigl(S_j^{p+1,q-1}-m_jS_j^{p,q}\bigr),
$$
where \(A_j(m)\asymp1\) and \(m_j\ge\eps_0/2\) on the stopped interval by
\Cref{lem:Bi-positive-macroscopic}.  Thus the term
\(-qA_j(m)m_jS_j^{p,q}\) damps each moment, while
\(S_j^{p+1,q-1}\) is handled by induction on the number of student powers.  The
non-Gaussian corrections are estimated using
\Cref{lem:nl-naive-J-vector-field}; the martingale increments are the same ones
controlled in
\eqref{eq:nl-mixed-martingale-increment}--\eqref{eq:nl-mixed-martingale-bound},
and the coordinate envelope follows from \eqref{eq:nl-coordinate-martingale-bound}.
After increasing \(B\), no stopped exit occurs before
\(N^{\rm nl}\) with probability at least \(1-\delta_{\rm str}\).
\end{proof}

\begin{corollary}[Online weak-to-strong nonlinear recovery]
\label{cor:nl-online-weak-to-sqrt-recovery}
In the setting of \Cref{thm:nl-online-sgd-weak-recovery}, suppose that at the
weak-recovery time \(n_\eps\) from \eqref{eq:nl-sgd-weak-time} the iterates
satisfy
\begin{equation}
  (\theta,\alpha_{n_\eps})\in\mathcal C^{\rm nl}_{P,B_0}(\eps).
  \label{eq:nl-online-sqrt-at-weak-time}
\end{equation}
Then, conditioned on the favorable initialization and correct sign pattern,
nonlinear-correlation online SGD reaches
\begin{equation}
  \min_{j\le K}m_j
  \ge 1-\rho_{d}^{\rm nl}
  \label{eq:nl-online-weak-to-sqrt-overlap}
\end{equation}
after an additional \(O_{\phi,K}(d^{K/2}\log d)\) samples, with
\(\rho_{d}^{\rm nl}\) defined in \eqref{eq:nl-online-sqrt-rhod}.
\end{corollary}
\begin{remark}[Correlation vs MSE strong recovery]
    Corollary \ref{cor:nl-online-weak-to-sqrt-recovery} shows that the basic SGD algorithm over the correlation loss used for weak recovery can automatically reach overlaps of order $1 - \rho_d^{\rm nl} = 1 - o_d(1)$. A natural question is how to reach truly strong recovery, in the sense of $\min_j m_j(n) \ge 1-\varepsilon$ for any $\varepsilon >0$ for any fixed $d$. In the companion paper \cite{dai2026MSIM} we show that SGD on the MSE loss achieves such strong recovery when $K=2$; we believe the same argument should extend to $K>2$ and yield strong recovery guarantees for $K>2$, provided this SGD method is warmstarted after the weak recovery phase, but we leave this extension as an interesting question for future work. That said, the resulting algorithm is arguably less `canonical' than our SGD method, requiring the user to switch from correlation to MSE loss after weak recovery.
\end{remark}

\section{Numerical Experiments}
\label{sec:experiments}

 While correlation SGD enjoys strong recovery guarantees in the high-dimensional limit for any $K$ with constant probability, as we just showed, the MSE SGD dynamics are affected by the additional energy gradient $\nabla_{\tilde{\theta}} \| f_{\tilde{\theta}} \|^2$, which creates fluctuations of order $d^{-1/2}$ even in the mediocrity zone. As a result, MSE recovery is presumably efficient only when $K \leq 2$, where the signal strength $d^{-(K-1)/2}$ can outweigh these non-informative fluctuations.

To gain further intuition on this phenomenon, we illustrate empirically the performance of online SGD, specifically the comparison between $\mathcal{L}$ and $\mathcal{L}_{\text{MSE}}$.
We consider a MSIM model using the normalized GeLU
activation
$$
  \phi(x)=\frac{\operatorname{GeLU}(x)-\mu_{\rm GeLU}}{\sigma_{\rm GeLU}},
$$
with $\mu_{\rm GeLU}$ and $\sigma_{\rm GeLU}$ adjusted so that
$\E[\phi]=0$ and $\| \phi\|_2 = 1$.
Teacher directions are sampled independently and uniformly from
\(\mathcal{S}^{d-1}\).  Student directions are sampled independently and uniformly from the same sphere and then flipped layerwise, if needed, so that we are in the positive overlap event at initialization
$$
  \langle \tilde\theta_j,\theta_j\rangle \ge 0,\qquad j=1,\ldots,K,
$$
in accordance to our theory.
We consider online SGD with minibatches, to better exploit the GPU resources, so the per-update learning rate is
$  \eta_{\rm update}=B\,\eta_0 d^{-K/2},$ where \(B\) is the effective batch size.  In our experiments, we use $B=64$ and $\eta_0 \in \{0.05, 0.1\}$. The sample horizon is parameterized as
$N=C\,d^{K-1},$ where \(C\) is a reported horizon constant that we adjust around $C \sim 10^2$. We used Slurm job arrays using NVIDIA A100 GPUs.

Figure \ref{fig:overlaps} shows the recovery performance of online SGD.
We verify that, as expected, the correlation loss achieves recovery at roughly $O( d^{K-1})$ samples, whereas the MSE loss fails to recover the planted parameters as $K>2$ and $d$ increases, relative to the correlation setting. Because the theorem is conditional on a favorable initialization event of constant probability, the main curves show the top half of seeds, corresponding to successful/favorable trajectories. We also report the raw success fraction across all seeds in Table \ref{tab:overlap-success}. As expected, the success probability decays with $K$.

\begin{figure}
    \centering
    \includegraphics[width=0.65\linewidth]{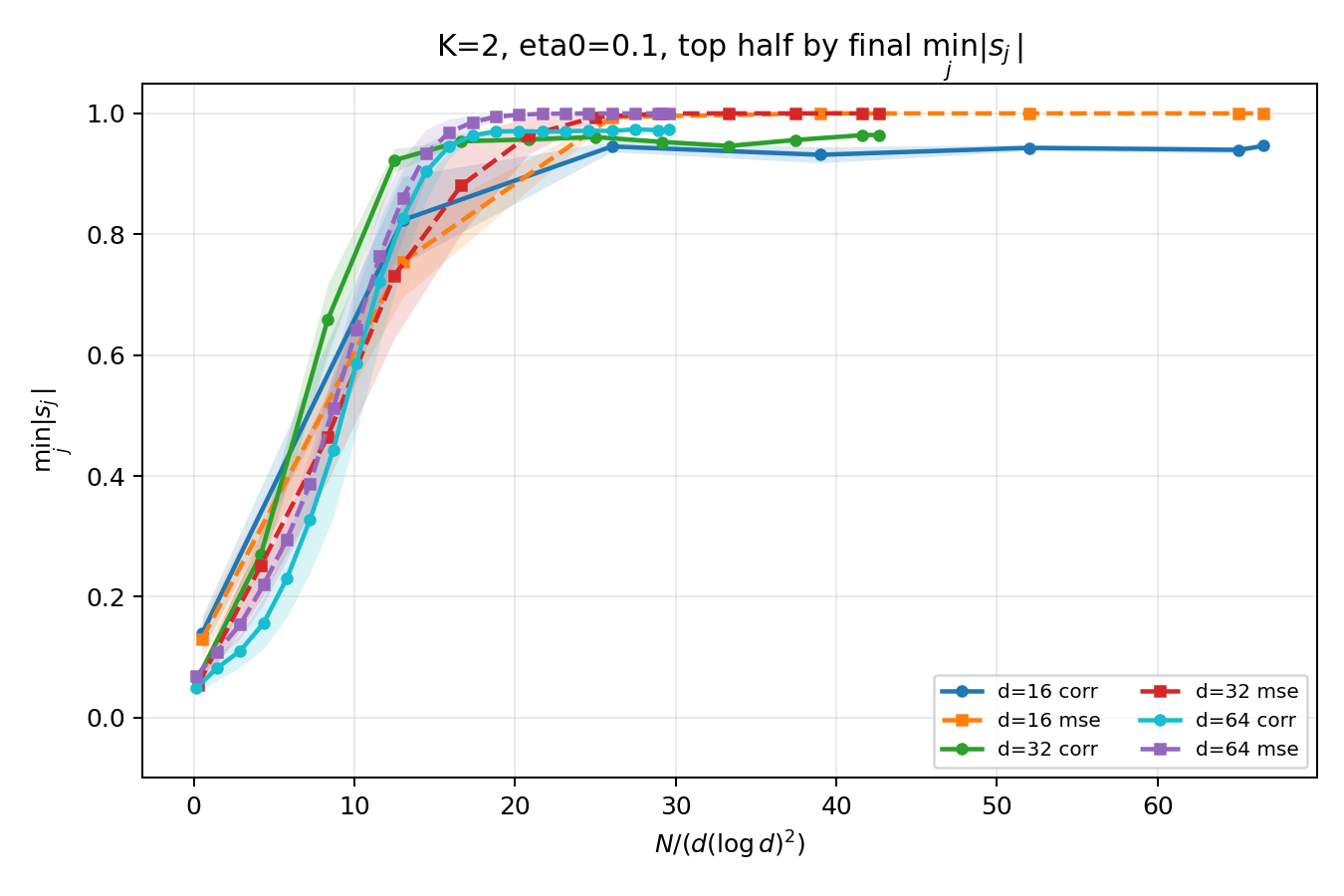}
    \includegraphics[width=0.65\linewidth]{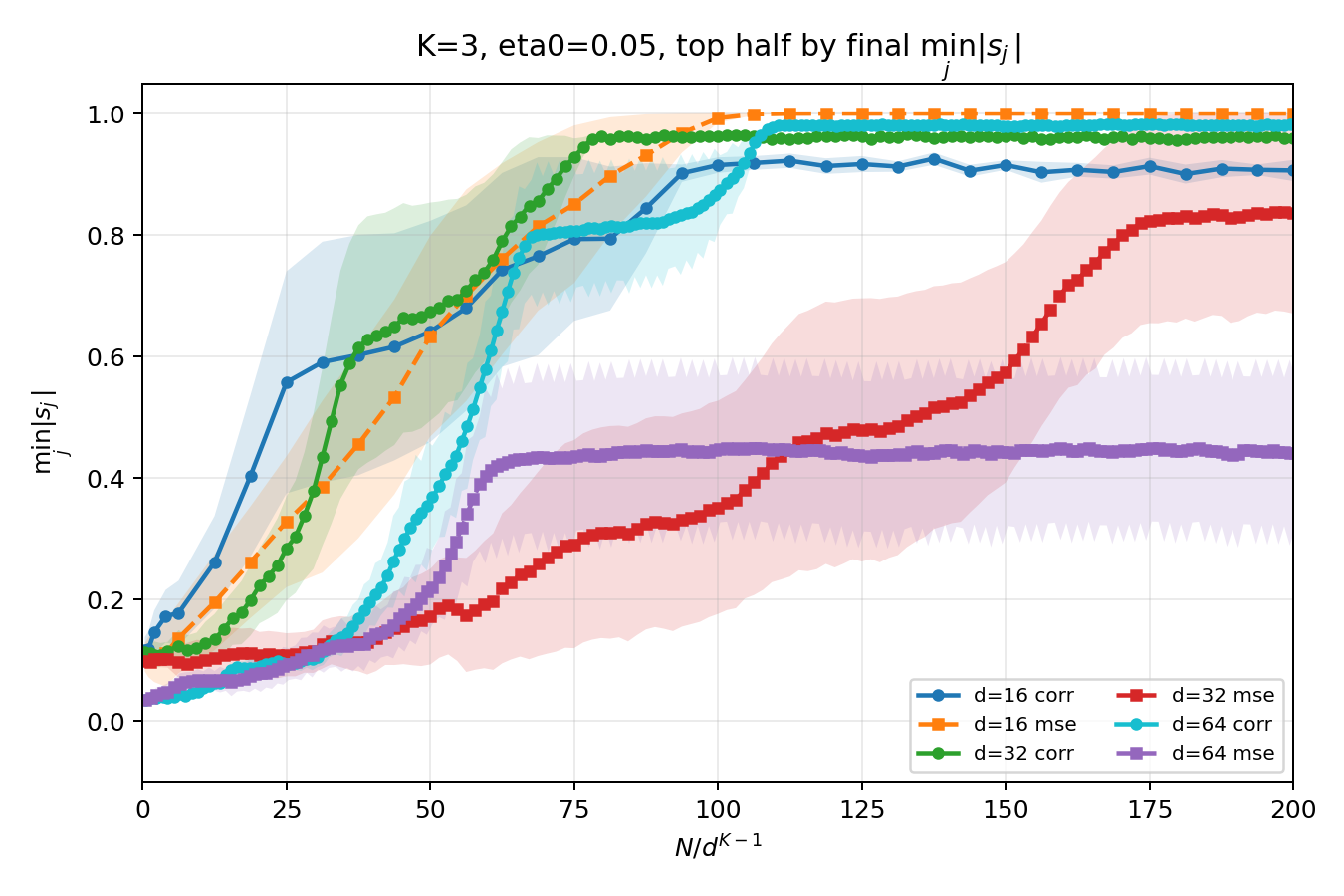}
    \includegraphics[width=0.65\linewidth]{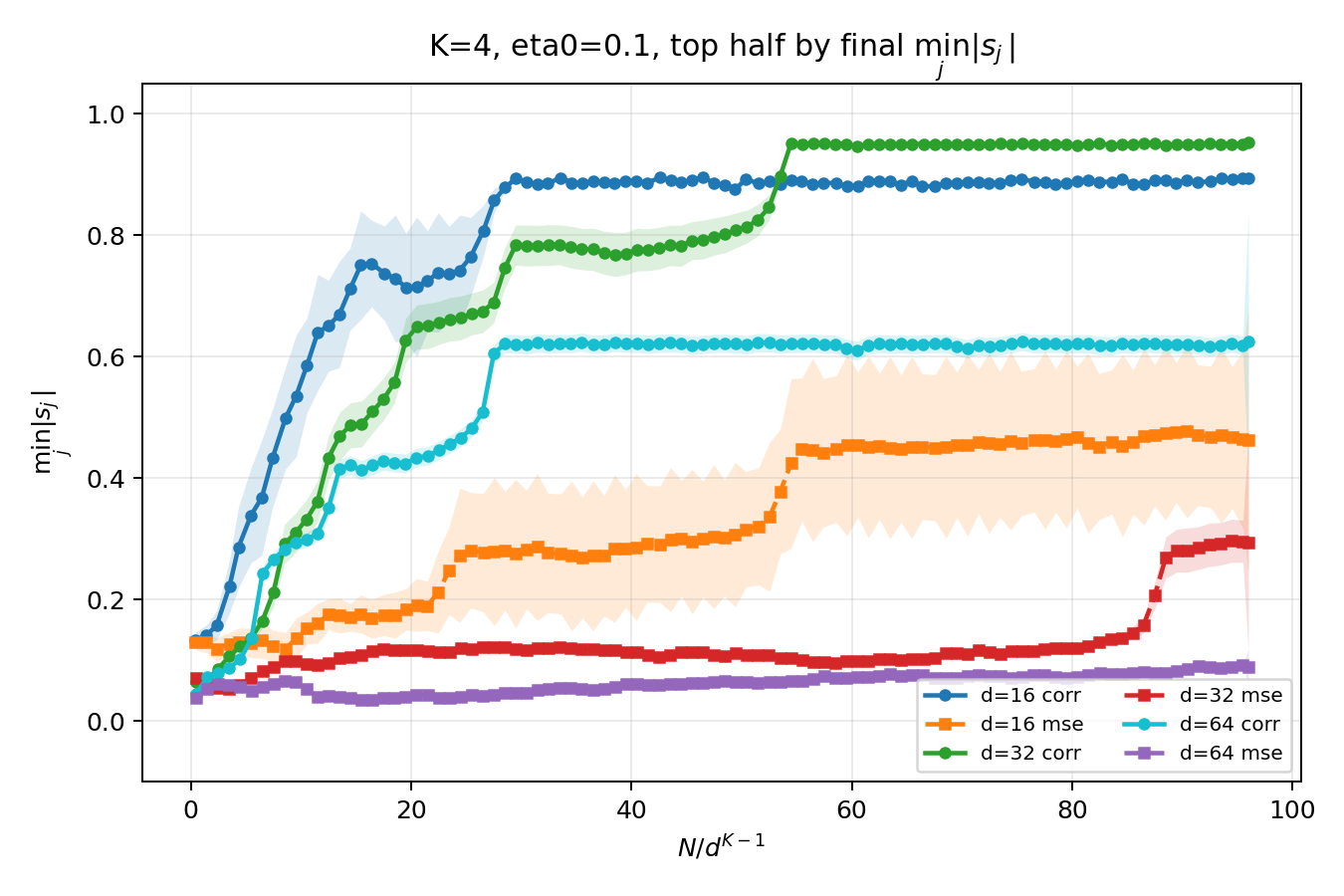}
    \caption{Overlap performance $\min_j |m_j(n)|$ as a function of the number of samples $n$, for $K \in \{2,3,4\}$ and $d \in \{16,32,64\}$. Solid lines correspond to correlation SGD while dashed lines to MSE. We run 10 different seeds and provide error bars for the top half performing ones. }
    \label{fig:overlaps}
\end{figure}

\begin{table}
    \centering
    \begin{tabular}{ccccc}
        \toprule
        \(K\) & \(d\) & Correlation SGD & MSE SGD & Experiment setting \\
        \midrule
        2 & 16 & \(10/10\) & \(9/10\) & \(\eta_0=0.1,\ C=512\) \\
        2 & 32 & \(10/10\) & \(10/10\) & \(\eta_0=0.1,\ C=512\) \\
        2 & 64 & \(10/10\) & \(8/10\) & \(\eta_0=0.1,\ C=512\) \\
        \midrule
        3 & 16 & \(10/10\) & \(5/10\) & \(\eta_0=0.05,\ C=640\) \\
        3 & 32 & \(10/10\) & \(6/10\) & \(\eta_0=0.05,\ C=640\) \\
        3 & 64 & \(10/10\) & \(3/10\) & \(\eta_0=0.05,\ C=640\) \\
        \midrule
        4 & 16 & \(5/10\) & \(2/10\) & \(\eta_0=0.1,\ C=96\) \\
        4 & 32 & \(6/10\) & \(1/10\) & \(\eta_0=0.1,\ C=96\) \\
        4 & 64 & \(3/10\) & \(0/10\) & \(\eta_0=0.1,\ C=96\) \\
        \bottomrule
    \end{tabular}
    \caption{Raw success fractions across all 10 seeds for the experiments in Figure~\ref{fig:overlaps}. A run is counted as successful if its final overlap satisfies \(\min_j |m_j|>0.75\). The horizon is \(N=C d^{K-1}\), except for \(K=2\), where the plotted theory-normalized scale uses \(N/(d(\log d)^2)\).}
    \label{tab:overlap-success}
\end{table}

\section{Conclusions and Future Directions}
\label{sec:conclusions}
In this work, we have analyzed the Multiscale Single-Index Model, originally introduced in \cite{oymak2021learning}, in which a $d^K$-dimensional input with multiscale structure is iteratively coarse-grained via local, translation-invariant single-index models, and where the number of scales $K$ corresponds directly to the depth of the associated neural network.

The scale separation present in the MSIM model, combined with  incoherence and non-degeneracy assumptions, creates informative low-rank structure in the associated Wiener chaos expansion. In particular, the first chaos provides a natural nonlinear analogue of the Tensor PCA model where the informative dominant rank-one spike is `polluted' by a hierarchy of subspaces of increasing dimension along a staircase spectrum at orders $d^{-q/2}$.  While spectral recovery methods such as tensor unfolding are robust to these spurious subspaces, gradient-based methods require a bookkeeping of high-order terms, that we implement via an Edgeworth expansion. As a result of this expansion, our main results establish (i) a quantitative separation between MSIM models and shallow NNs, showing that depth is necessary to achieve good approximation, and (ii) efficient SGD recovery of the MSIM model using online SGD on the correlation loss, confirming that same-timescale online backpropagation on the nonlinear correlation objective suffices for recovery, without any need for layerwise training.

Our work leaves several open questions ahead of us. First, an arguably more natural   high-dimensional regime is where $d$ (the so-called `receptive field') is fixed, but depth $K$ is diverging; this regime will a priori require substantial technical innovation, since the mean-field approximation underlying our Edgeworth arguments no longer applies.
Next, an important limitation of our work is that it imposes a form of strong scale separation: each layer `lives' on its own dedicated scale, and, in the language of CNNs, the architecture imposes stride equal to the filter patch size. While this is a valid assumption for certain physical data (e.g images), and is also present in the Random Hierarchy Model from \cite{cagnetta2024deep}, many popular deep architectures do not have scale separation; for instance in ResNets \cite{he2015deepresiduallearningimage} and their associated NeuralODEs \cite{chen2019neuralordinarydifferentialequations}, as well as in Transformers \cite{vaswani2017attention} and their mathematical blueprints \cite{geshkovski2025mathematical}. More generally, in architectures where depth is akin to a temporal evolution, as in diffusion and flow-based generative models \cite{sohl2015deep, albergo2023stochastic, lipman2022flow}, there is no scale separation by construction. Building toy models that explain feature learning in these settings remains an outstanding challenge, initiated e.g. in \cite{chizat2025hidden, chaintron2026resnets}. 

\bibliographystyle{alpha}
\bibliography{arxiv_refs}

\newpage
\appendix

\section{Deferred Proofs of Section \ref{sec:spectral_estimation}}

\begin{lemma}[Matricized empirical Stein concentration]\label{lem:matrix-conc}
Assume \eqref{ass:regularity}. Conditional on any fixed directions \(\theta_1,\ldots,\theta_K\in S^{d-1}\), there is a constant \(C_{\phi,K}<\infty\) such that
$$
  \norm{f_\theta(X)}_{\psi_2}\le C_{\phi,K}.
$$
Consequently, for every \(\delta\in(0,1)\), with probability at least \(1-\delta\) over the samples,
\begin{equation}
  \norm{\Mat_a(T_n-\E_XT_n)}_\op
  \le
  C_{\phi,K}\left[
  \sqrt{\frac{D_*\log((D_L+D_R)/\delta)}{n}}
  +
  \frac{\sqrt{D_*}\,\log^{3/2}(n/\delta)\log((D_L+D_R)/\delta)}{n}
  \right].
  \label{eq:matrix-conc-full}
\end{equation}
In particular, if
$$
  n\ge C_{\phi,K}D_*\log^4((D_L+D_R)n/\delta),
$$
then the first term dominates and
\begin{equation}
  \norm{\Mat_a(T_n-\E_XT_n)}_\op
  \le
  C_{\phi,K}
  \sqrt{\frac{D_*\log((D_L+D_R)/\delta)}{n}}.
  \label{eq:matrix-conc-simple}
\end{equation}
\end{lemma}

\begin{proof}[Proof of Lemma \ref{lem:matrix-conc}]
Let \(A:=\Mat_a(X)\in\R^{D_L\times D_R}\) and \(Y:=f_\theta(X)A\). Since \(\phi\) is globally Lipschitz and all contractions have norm one, each scalar coordinate generated by the network is subgaussian with a \(\psi_2\)-norm depending only on \(\phi\) and \(K\). In particular,
$$
  \norm{f_\theta(X)}_{\psi_2}\le C_{\phi,K}.
$$
Let \(Z:=Y-\E Y\). We first bound the matrix variance. For any unit vector \(p\in\R^{D_L}\),
\begin{align*}
  p^\top\E[YY^\top]p
  &=\E\bigl[f_\theta(X)^2\norm{A^\top p}_2^2\bigr]\\
  &=\sum_{j=1}^{D_R}\E\bigl[f_\theta(X)^2\ip{A_{:j}}{p}^2\bigr].
\end{align*}
For each column \(j\), \(\ip{A_{:j}}p\) is a standard Gaussian random variable. By Cauchy-Schwarz and the subgaussian moment bound for \(f_\theta(X)\),
$$
  \E\bigl[f_\theta(X)^2\ip{A_{:j}}p^2\bigr]
  \le
  (\E f_\theta(X)^4)^{1/2}(\E\ip{A_{:j}}p^4)^{1/2}
  \le C_{\phi,K}.
$$
Hence \(\norm{\E YY^\top}_\op\le C_{\phi,K}D_R\). Similarly,
\(\norm{\E Y^\top Y}_\op\le C_{\phi,K}D_L\). Centering can only decrease the second moment in the Loewner order after subtracting \((\E Y)(\E Y)^\top\) and its transpose. Therefore
\begin{equation}
  \sigma^2:=
  \max\{\norm{\E ZZ^\top}_\op,\norm{\E Z^\top Z}_\op\}
  \le C_{\phi,K}D_*.
  \label{eq:variance-proxy}
\end{equation}

The sample matrices are not bounded, so we truncate. Standard Gaussian matrix concentration gives
$$
  \norm A_\op\le C(\sqrt{D_L}+\sqrt{D_R}+t)
$$
with probability at least \(1-e^{-t^2}\), and subgaussianity gives
\(\abs{f_\theta(X)}\le C_{\phi,K}t\) with probability at least \(1-e^{-ct^2}\). Taking \(t\asymp\sqrt{\log(n/\delta)}\) and union bounding over \(n\) samples yields, except on an event of probability at most \(\delta/2\),
\begin{equation}
  \norm{Y_\ell}_\op
  \le
  C_{\phi,K}\sqrt{D_*}\sqrt{\log(n/\delta)},
  \qquad 1\le\ell\le n.
  \label{eq:sample-envelope}
\end{equation}
The expectation of the truncated tail is exponentially small at this level and is absorbed into the second term in \eqref{eq:matrix-conc-full}. Applying rectangular matrix Bernstein to the truncated independent centered matrices, with variance proxy \(n\sigma^2\) and envelope \eqref{eq:sample-envelope}, gives \eqref{eq:matrix-conc-full}. The simplified bound \eqref{eq:matrix-conc-simple} follows when the envelope term is no larger than the square-root term.
\end{proof}

\begin{lemma}[Rank-one Wedin bound]\label{lem:wedin}
Let
$$
  \widehat M=\lambda uv^\top+E,
  \qquad
  \norm u_2=\norm v_2=1,
  \qquad
  \lambda\ne0.
$$
If \(\norm E_\op\le\abs\lambda/4\), then any top left and right singular vectors \(\widehat u,\widehat v\) of \(\widehat M\) satisfy
\begin{equation}
  \distpm(\widehat u,u)
  \le C\frac{\norm E_\op}{\abs\lambda},
  \qquad
  \distpm(\widehat v,v)
  \le C\frac{\norm E_\op}{\abs\lambda}.
  \label{eq:wedin}
\end{equation}
\end{lemma}

\begin{proof}[Proof of Lemma \ref{lem:wedin}]
This is Wedin's sin-theta theorem for singular subspaces. The rank-one signal has first singular value \(\abs\lambda\) and second singular value zero. Under \(\norm E_\op\le\abs\lambda/4\), the empirical spectral gap is at least \(\abs\lambda/2\), so the sines of the principal angles between \(\widehat u,u\) and \(\widehat v,v\) are bounded by \(C\norm E_\op/\abs\lambda\). For unit vectors, Euclidean distance up to sign is bounded by a universal constant times the sine of the angle.
\end{proof}

\begin{lemma}[Recursive Kronecker factor extraction]\label{lem:recursive}
Fix \(m\ge1\) and unit vectors \(q_1,\ldots,q_m\in S^{d-1}\). Let
$$
  q=q_1\ot\cdots\ot q_m\in\R^{d^m}.
$$
Suppose \(\widehat q\in\R^{d^m}\) is unit norm and
$$
  \distpm(\widehat q,q)\le\eps\le c_m
$$
for a sufficiently small constant \(c_m>0\). Apply the recursive SVD factorization from Section 3.1 to \(\widehat q\). Then the returned vectors \(\widehat q_1,\ldots,\widehat q_m\) satisfy
\begin{equation}
  \max_{1\le r\le m}\distpm(\widehat q_r,q_r)
  \le C_m\eps,
  \label{eq:recursive-factor-bound}
\end{equation}
where \(C_m\) depends only on \(m\).
\end{lemma}

\begin{proof}[Proof of Lemma \ref{lem:recursive}]
The proof is by induction on \(m\). The case \(m=1\) is immediate. For \(m\ge2\), write \(m_1=\lfloor m/2\rfloor\), \(m_2=m-m_1\), and define
$$
  q_L:=q_1\ot\cdots\ot q_{m_1},
  \qquad
  q_R:=q_{m_1+1}\ot\cdots\ot q_m.
$$
The balanced reshaping of \(q\) is exactly the rank-one matrix \(q_Lq_R^\top\), with singular value one. After choosing the sign of \(\widehat q\) so that \(\norm{\widehat q-q}_2\le\eps\), the reshaped perturbation has Frobenius norm \(\eps\), hence operator norm at most \(\eps\). Wedin's theorem applied to
$$
  \Mat(\widehat q)=q_Lq_R^\top+E,
  \qquad
  \norm E_\op\le\eps,
$$
gives
$$
  \distpm(\widehat q_L,q_L)+\distpm(\widehat q_R,q_R)
  \le C\eps.
$$
Applying the induction hypothesis to \(\widehat q_L\) and \(\widehat q_R\) proves \eqref{eq:recursive-factor-bound}, after increasing \(C_m\).
\end{proof}

\end{document}